\documentclass{article}

\usepackage{microtype}
\usepackage{graphicx}
\usepackage{booktabs} 

\usepackage{hyperref}
\usepackage{url}            
\usepackage{booktabs}       
\usepackage{amsfonts}       
\usepackage{nicefrac}       
\usepackage{microtype}      
\usepackage{mkolar_definitions}
\usepackage{xspace}
\usepackage{amsmath}
\usepackage{algorithm}
\usepackage{algorithmic}
\usepackage{color}
\usepackage{enumitem}
\usepackage{comment}
\usepackage{bm}
\usepackage{subcaption}
\usepackage{csquotes}

\newtheorem*{theorem*}{Theorem}
\theoremstyle{definition}
\newtheorem*{runex*}{Running Example}
\newtheorem*{ex*}{Example}
\newtheorem*{remark*}{Remark}

\newcount\Comments  
\definecolor{darkgreen}{rgb}{0,0.5,0}
\definecolor{darkred}{rgb}{0.7,0,0}
\definecolor{teal}{rgb}{0.3,0.8,0.8}
\definecolor{orange}{rgb}{1.0,0.5,0.0}
\definecolor{purple}{rgb}{0.8,0.0,0.8}
\newcommand{\kibitz}[2]{\ifnum\Comments=1{\textcolor{#1}{\textsf{\footnotesize #2}}}\fi}

\newcommand{\defeq}{\triangleq}

\usepackage[accepted]{icml2019}

\icmltitlerunning{Imitation Learning from Observation}

\begin{document}

\twocolumn[
\icmltitle{Provably Efficient Imitation Learning from Observation Alone}



\icmlsetsymbol{equal}{*}

\begin{icmlauthorlist}
\icmlauthor{Wen Sun}{goo}
\icmlauthor{Anirudh Vemula}{goo}
\icmlauthor{Byron Boots}{to}
\icmlauthor{J. Andrew Bagnell}{ed}
\end{icmlauthorlist}

\icmlaffiliation{to}{College of Computing, Georgia Institute of Technology, USA}
\icmlaffiliation{goo}{Robotics Institute, Carnegie Mellon University, USA}
\icmlaffiliation{ed}{Aurora Innovation, USA}

\icmlcorrespondingauthor{Wen Sun}{wensun@cs.cmu.edu}

\icmlkeywords{Imitation Learning, Learning from Demonstrations}

\vskip 0.3in
]



\printAffiliationsAndNotice{}  

\begin{abstract}

We study Imitation Learning (IL) from Observations alone (\ilo)  in large-scale MDPs. While most IL algorithms rely on an expert to directly provide actions to the learner, in this setting the expert only supplies sequences of observations. We design a new model-free algorithm for \ilo, \emph{Forward Adversarial Imitation Learning} (\fail), which learns a sequence of time-dependent policies by minimizing an Integral Probability Metric between the observation distributions of the expert policy and the learner.  \fail{} is the \emph{first} provably efficient algorithm in \ilo{} setting, which learns a near-optimal policy with a number of samples that is polynomial in all relevant parameters but independent of the number of unique observations. The resulting theory extends the domain of provably sample efficient learning algorithms beyond existing results, which typically only consider tabular reinforcement learning settings or settings that require access to a near-optimal reset distribution. We also investigate the extension of \fail{} in a model-based setting.  Finally we demonstrate the efficacy of \fail{} on multiple OpenAI Gym control tasks. 

\end{abstract}


\section{Introduction}
\label{sec:introduction}

Imitation Learning (IL) is a sample efficient approach to policy optimization \citep{Ross2011_AISTATS,ross2014reinforcement,sun2017deeply} that has been extensively used in real applications, including Natural Language Processing \citep{daume2009search,chang2015learning,chang2015learning_dependency}, game playing \citep{silver2016mastering,hester2017deep},  system identification \citep{venkatraman2015improving,sun2016learning}, and robotics control tasks \citep{pan2018agile}. Most  previous IL work considers settings where an expert can directly provide action signals. In these settings, a general strategy is to directly learn a policy that maps from state to action, via supervised learning approaches (e.g., DAgger \citep{Ross2011_AISTATS}, AggreVaTe~\citep{ross2014reinforcement}, THOR~\citep{sun2018truncated}, Behaviour Cloning \citep{syed2010reduction}). Another popular strategy is to learn a policy by minimizing some divergence between the policy's state-action distribution and the expert's state-action distribution. Popular divergences include Forward KL (i.e., Behaviour Cloning), Jensen Shannon Divergence (e.g., GAIL \citep{ho2016generative}).

Here, we consider a more challenging IL setting, where experts' demonstrations consist only of observations, no action or reward signals are available to the learner, and no reset is allowed (e.g., a robot learns a task by just watching an expert performing the task). We call this setting \emph{Imitation Learning from Observations alone} (\ilo). 
Under this setting, without access to expert actions, approaches like DAgger, AggreVaTe, GAIL, and Behaviour Cloning by definition cannot work. Although recently several model-based approaches, which learn an inverse model that predicts the actions taken by an  expert~\citep{torabi2018behavioral,edwards2018imitating} based on successive observations, have been proposed, these approaches can suffer from covariate shift \cite{Ross2011_AISTATS}. 
While we wish to train a predictor that can infer an expert's actions accurately \emph{under the expert's observation distribution}, we do not have access to actions generated by the expert conditioned on the expert's observation (See~\pref{sec:more_related_work} for a more detailed discussion). 
An alternative strategy is to handcraft cost functions that use some distance metric to penalize deviation from the experts' trajectories (e.g., \citet{liu2018imitation,peng2018deepmimic}), which is then optimized by Reinforcement Learning (RL). These methods typically involve hand-designed cost functions that sometimes require prior task-specific knowledge \cite{peng2018deepmimic}.  The quality of the learned policy is therefore completely dependent on the hand-designed costs which could be widely different from the true cost. 
Ideally, we would like to learn a policy that minimizes the unknown true cost function of the underlying MDP.


In this work, we explicitly consider learning near-optimal policies in a sample and computationally efficient manner. Specifically, we focus on large-scale MDPs where the number of unique observations is extremely large (e.g., high-dimensional observations such as raw-pixel images). Such large-scale MDP settings immediately exclude most existing sample efficient RL algorithms, which are often designed for small tabular MDPs, whose sample complexities have a polynomial dependency on the number of observations and hence cannot scale well. To solve large-scale MDPs, we need to design algorithms that leverage function approximation for generalization. 
Specifically, we are interested in algorithms with the following three properties: (1) \emph{near-optimal performance guarantees}, i.e., we want to search for a policy whose performance is  close to the expert's in terms of the expected total cost of the underlying MDP (and not a hand-designed cost function); (2) \emph{sample efficiency}, we require sample complexity that scales polynomially with respect to all relevant parameters (e.g., horizon, number of actions, statistical complexity of function approximators) except the cardinality of the observation space---hence excluding PAC RL algorithms designed for small tabular MDPs; (3) \emph{computational efficiency}: we rely on the notion of oracle-efficiency \cite{agarwal2014taming} and require the number of efficient oracle calls to scale polynomially---thereby excluding recently proposed algorithms for Contextual Decision Processes which are not computationally efficient \cite{jiang2016contextual,sun2018model}. 
To the best of our knowledge, the desiderata above requires designing new algorithms.  

With access to experts' trajectories of observations 
we introduce a model-free algorithm, called Forward Adversarial Imitation Learning (\fail), that decomposes \ilo{} into $H$ independent two-player min-max games, where $H$ is the horizon length. We aim to learn a sequence of time-dependent policies from $h=1$ to $H$, where at any time step $h$, the policy $\pi_h$ is learned such that the generated observation distribution at time step $h+1$, conditioned on $\pi_1,\dots, \pi_{h-1}$ being fixed, matches the expert's observation distribution at time step $h+1$, in terms of an Integral Probability Metric (IPM) \citep{Muller1997}. IPM is a family of divergences that can be understood as using a set of discriminators to distinguish two distributions (e.g., Wasserstein distance is one such special instance).  
We analyze the sample complexity of \fail{}  and show that \fail{} can learn a near-optimal policy in sample complexity that does not explicitly depend on the cardinality of observation space, but rather only depends on the complexity measure of the policy class and the discriminator class.  Hence \fail{} satisfies the above mentioned three properties. The resulting theory extends the domain of provably sample efficient learning algorithms beyond existing results, which typically only consider tabular reinforcement learning settings (e.g., \citet{dann2015sample}) or settings that require access to a near-optimal reset distribution (e.g., \citet{kakade2002approximately,bagnell2004policy,munos2008finite}). 
We also demonstrate that learning under \ilo{} can be exponentially more sample efficient than pure RL.
We also study \fail{} under three specific settings: (1) Lipschitz continuous MDPs, (2) Interactive \ilo{} where one can query the expert at any time during training, and (3) state abstraction. 
Finally, we demonstrate the efficacy of \fail{} on multiple continuous control tasks.

\section{Preliminaries}
\label{sec:preliminaries}

We consider an episodic finite horizon Decision Process that consists of $\{\Xcal_h\}_{h=1}^H, \Acal, c, H, \rho, P$, where $\Xcal_h$ for $h\in[H]$ is a time-dependent observation space,\footnote{we use the term observation throughout the paper instead of the term state as one would normally use in defining MDPs, for the purpose of  sharply distinguishing our setting from tabular MDPs where $\Xcal$ has very small number of states. } $\Acal$ is a discrete action space such that $\abr{\Acal} = K\in \NN^+$, $H$ is the horizon. We assume the cost function $c:\Xcal_H\to\mathbb{R}$ is only defined at the last time step $H$ (e.g., sparse cost), and $\rho \in \Delta(\Xcal_1)$ is the initial observation distribution, and $P$ is the transition model  i.e., $P: \Xcal_h\times\Acal \to \Delta(\Xcal_{h+1})$ for $h\in [H-1]$. 
Note that here we assume the cost function only depends on observations. 
We assume that $\Xcal_h$ for all $h\in[H]$ is discrete, but $\abr{\Xcal_h}$ is extremely large and hence any sample complexity that has polynomial dependency on $\abr{\Xcal_h}$ should be considered as sample inefficient,  i.e., one cannot afford to visit every unique observation.  We assume that the  cost is bounded, i.e., for any sequence of observations, $ c_H \leq 1$ (e.g., zero-one loss at the end of each episode).
For a time-dependent policy $\bm{\pi} = \{\pi_1, \dots, \pi_H\}$ with $\pi_h:\Xcal_h\to\Delta(\Acal)$, the value function $V^{\bm{\pi}}_h:\Xcal_h\to [0,1]$ is defined as:
\begin{align*}
    V^{\bm{\pi}}_h(x_h) = \EE\left[  c(x_H) \vert a_i\sim \pi_i(\cdot|x_i),x_{i+1}\sim P_{x_i, a_i} \right],
\end{align*} and state-action function $Q^{\bm{\pi}}_h(x_h, a_h)$ is defined as $Q^{\bm{\pi}}_h(x_h,a_h) =  \EE_{x_{h+1}\sim P_{x_h,a_h}}\left[V^{\bm{\pi}}_{h+1}(x_{h+1})\right]$ with $V_H^{\bm\pi}(x) = c(x)$. We denote $\mu^{\bm{\pi}}_h$ as the distribution over $\Xcal_h$ at time step $h$ following $\bm{\pi}$. Given $H$ policy classes $\{\Pi_1,\dots, \Pi_H\}$, the goal is to learn a $\bm{\pi} = \{\pi_1,\dots, \pi_H\}$ with $\pi_h\in\Pi_h$, which minimizes the expected cost:
\begin{align*}
    J(\bm\pi) =  \mathbb{E}\left[  c(x_H) | a_h\sim \pi_h(\cdot|x_h), x_{h+1}\sim P(\cdot| x_{h},a_h)\right].
\end{align*}





Denote $\Fcal_h \subseteq \{f:\Xcal_{h}\to \RR\}$ for $h\in[H]$. We define a Bellman Operator $\Gamma_h$ associated with the expert $\pi^\star_h$ at time step $h$ as $\Gamma_h: \Fcal_{h+1} \to \{f:\Xcal_h\to\mathbb{R}\}$
where for any $x_h\in\Xcal_h, f\in\Fcal_{h+1}$,
\begin{align*}
\left(\Gamma_h f\right)(x_h) \defeq  \EE_{a_h\sim \pi^\star_h(\cdot|x_h),x_{h+1}\sim P_{x_h,a_h}}\left[f(x_{h+1}) \right].
\end{align*}

\paragraph{Notation} For a function $f:\mathcal{X}\to\mathbb{R}$, we denote $\|f\|_{\infty} = \sup_{x\in\Xcal} \abr{f(x)}$, $\|f\|_L$ as the Lipschitz constant: $\|f\|_L = \sup_{x_1,x_2, x_1\neq x_2} (f(x_1) - f(x_2)) / d(x_1, x_2)$, with $d: \Xcal\times\Xcal\to \RR^+$ being the metric in space $\Xcal$.\footnote{A metric $d$ (e.g., Euclidean distance) satisfies the following conditions: $d(x,y)\geq 0$, $d(x,y) = 0$ iff $x=y$, $d(x,y) = d(y,x)$ and $d$ satisfies triangle inequality. } We consider a Reproducing Kernel Hilbert Space (RKHS) $\Hcal$ defined with a positive definite kernel $k:\mathcal{X}\times\mathcal{X}\to [0,1]$ such that $\Hcal$ is the span of $\{k(x,\cdot): x\in \Xcal\}$, and we have $f(x) = \langle f, k(x,\cdot)\rangle$, with $\langle k(x_1,\cdot), k(x_2,\cdot)\rangle \triangleq k(x_1,x_2)$. For any $f \in \Hcal$, we define $\|f\|_{\Hcal}^2 = \langle  f, f\rangle$. We denote $U(\Acal)$ as a uniform distribution over action set $\Acal$. For $N\in\NN^+$, we denote $[N]\defeq \{1,2,\dots,N\}$.




\paragraph{Integral Probability Metrics (IPM)} \citep{muller1997integral} is a family of distance measures on distributions: given two distributions $P_1$ and $P_2$ over $\Xcal$, and a function class $\Fcal$ containing real-value functions $f:\Xcal\to\mathbb{R}$ and symmetric (e.g., $\forall f\in\Fcal$, we have $-f\in\Fcal$), IPM is defined as:
\begin{align}
    \label{eq:ipm}
    \sup_{f\in\Fcal}  \left( \mathbb{E}_{x\sim P_1}[f(x)] - \mathbb{E}_{x\sim P_2}[f(x)]\right).
\end{align}
By choosing different class of functions $\Fcal$, various popular distances can be obtained. For instance, IPM with $\Fcal = \{f:\|f\|_{\infty}\leq 1\}$ recovers Total Variation distance, IPM with $\Fcal = \{f: \|f\|_L\leq 1\}$ recovers Wasserstein distance, and IPM with $\Fcal = \{f: \|f\|_{\Hcal} \leq 1\}$ with RKHS $\Hcal$ reveals maximum mean discrepancy (MMD).

\subsection{Assumptions}
We first assume access to a Cost-Sensitive oracle and to a Linear Programming oracle. 

\paragraph{Cost-Sensitive Oracle} The Cost-Sensitive (CS) oracle takes a class of classifiers $\Pi \defeq \{\pi: \Xcal\to\Delta(\Acal)\} $, a dataset consisting of pairs of feature $x$ and cost vector $c\in \mathbb{R}^{K}$, i.e., $\Dcal = \{x_i, c_i\}_{i=1}^N$, as inputs, and outputs a classifier that minimizes the average expected classification cost: $\sum_{i=1}^N \pi(\cdot|x_i)^{\top} c_i /N$.  

Efficient cost sensitive classifiers exist (e.g., \citet{beygelzimer2005error}) and are widely used in sequential decision making tasks (e.g., \citet{agarwal2014taming,chang2015learning}).

\paragraph{Linear Programming Oracle}
A Linear Programming (LP) oracle takes a class of functions $\Gcal$ as inputs, optimizes a linear functional with respect to $g\in\Gcal$: $\min_{g\in\Gcal} \sum_{i=1}^N \alpha_i g(x_i)$. 

When $\Gcal$ is in RKHS with bounded norm,  the linear functional becomes $\max_{g: \|g\|\leq c} \langle g, \sum_{i=1}^n \alpha_i\phi(x_i) \rangle$, from which one can obtain the closed-form solution. Another example is when $\Gcal$ consists of all functions with bounded Lipschitiz constant, i.e., $\Gcal = \{g: \|g\|_L\leq c\}$ for $c\in\RR^+$, \citet{sriperumbudur2012empirical} showed that $\max_{g\in\Gcal} \sum_{i=1}^n \alpha_i g(x_i)$ can be solved by Linear Programming with $n$ many constraints, one for each pair $(\alpha_i, x_i)$. In~\pref{app:proof_of_claim}, we provide a new reduction to Linear Programming for $\Gcal = \{g: \|g\|_L\leq c_1, \|g\|_{\infty}\leq c_2\}$ for $c_1,c_2\in\RR^+$.

The second assumption is related to the richness of the function class. 
We simply consider a time-dependent policy class $\Pi_h$ and $\Fcal_h$ for $h\in [H]$, and we assume realizability:
\begin{assum}[Realizability and Capacity of Function Class]
We assume $\Pi_h$ and $\Fcal_h$ contains $\pi^\star_h$ and $ V^{\star}_h$, i.e., $\pi_h^{\star}\in \Pi_h$ and $V_h^{\star}\in \Fcal_h$, $\forall h\in[H]$.  Further assume that for all $h$, $\Fcal_h$ is symmetric, and $\Pi_h$ and $\Fcal_h$ is finite in size. 
\label{ass:realizable}
\end{assum} 
Note that we assume $\Pi_h$ and $\Fcal_h$ to be discrete (but could be extremely large) for analysis simplicity. As we will show later, our bound scales \emph{only logarithmically} with respect to the size of function class. 





\section{Algorithm}

Our algorithm, Forward Adversarial Imitation Learning (\fail), aims to learn a sequence of policies $\bm{\pi} = \{\pi_1, \dots, \pi_H\}$ such that its value $J(\bm{\pi})$ is close to $J(\bm{\pi}^\star)$. 
Note that $J(\bm{\pi}) \approx J(\bm{\pi}^\star)$ does not necessarily mean that the state distribution of $\bm{\pi}$ is close to $\bm{\pi}^\star$. \fail{} learns a sequence of policies with this property in a forward training manner. The algorithm learns $\pi_i$ starting from $i = 1$. When learning $\pi_{i}$, the algorithm fixes $\{\pi_1, \dots, \pi_{i-1}\}$, and solves a min-max game to compute $\pi_i$; it then proceeds to time step $i+1$. At a high level, \fail{} decomposes a sequential learning problem into $H$-many independent  two-player min-max games, where each game can be solved efficiently via no-regret online learning.   Below, we first consider how to learn $\pi_{h}$ conditioned on $\{\pi_1,\dots, \pi_{h-1}\}$ being fixed.  We then present \fail{} by chaining  $H$  min-max games together.

\subsection{Learning One Step Policy via a Min-Max Game}
Throughout this section, we assume $\{\pi_1,\dots, \pi_{h-1}\}$ are learned already and fixed. The sequence of policies $\{\pi_1, \dots, \pi_{h-1}\}$ for $h\geq 2$ determines a distribution $\nu_{h}\in\Delta(\Xcal_{h})$ over observation space $\Xcal_{h}$ at time step $h$.  Also expert policy $\bm{\pi}^\star$ naturally induces a sequence of observation distributions $\mu_{h}^\star \in \Delta(\Xcal_h)$ for $h\in [H]$.  The problem we consider in this section is to learn a policy $\pi_{h}\in \Pi_h$, such that the resulting observation distribution from $\{\pi_1, \dots, \pi_{h-1}, \pi_h\}$ at time step $h+1$ is close to the expert's observation distribution  $\mu^\star_{h+1}$ at time step $h+1$.

We consider the following IPM minimization problem. Given the distribution $\nu_h\in\Delta(\Xcal_h)$, and a policy $\pi\in\Pi_h$, via the Markov property, we have the observation distribution at time step $h+1$ conditioned on $\nu_h$ and $\pi$ as $\nu_{h+1}(x) \defeq \sum_{x_h, a_h} \nu_h(x_h)\pi(a_h|x_h)P(x | x_h, a_h)$ for any $x\in\Xcal_{h+1}$. Recall that the expert observation distribution at time step $h+1$ is denoted as $\mu^{\star}_{h+1}$. The IPM with $\Fcal_{h+1}$ between $\nu_{h+1}$ and $\mu^{\star}_{h+1}$ is defined as:
\begin{align*}
    & d_{\Fcal_{h+1}}(\pi \vert \nu_h, \mu^\star_{h+1})\\
    & \triangleq  \max_{f\in \Fcal_{h+1}} \left( \mathbb{E}_{x\sim\nu_{h+1}} \left[f(x)\right] - \mathbb{E}_{ x\sim \mu^{{\star}}_{h+1}}\left[f(x)\right]   \right) .
\end{align*}
Note that $d_{\Fcal_{h+1}}(\pi \vert \nu_h, \mu^\star_{h+1})$ is parameterized by $\pi$, and our goal is to minimize $d_{\Fcal_{h+1}}(\pi \vert \nu_h, \mu^\star_{h+1})$ with respect to $\pi$ over $\Pi_{h}$. However, $d_{\Fcal_{h+1}}(\pi \vert \nu_h, \mu^\star_{h+1})$ is not measurable directly as we do not have access to $\mu^\star_{h+1}$ but only samples from $\mu^\star_{h+1}$. To estimate $d_{\Fcal_{h+1}}$, we draw a dataset $\Dcal = \left\{(x^{i}_h, a^{i}_h, x^{i}_{h+1})\right\}_{i=1}^N$ such that $x^{i}_h \sim \nu_h$, $a_h^{i}\sim U(\Acal)$, $x_{h+1}^{i}\sim P(\cdot|x^i_{h}, a_h^i)$, together with  observation set resulting from expert $\Dcal^\star = \{\tilde{x}_{h+1}^i\}_{i=1}^{N'} \iidsim \mu^\star_{h+1}$, we form the following empirical estimation of $d_{\Fcal_{h+1}}$ for any $\pi$, via importance weighting (recall $a_h^i\sim U(\Acal)$):
\begin{align}
\label{eq:emprical_IPM}
    &{\wipm}_{\Fcal_{h+1}}(\pi \vert \nu_h, \mu^{\star}_{h+1})  \defeq 
    \\& \max_{f\in \Fcal_{h+1}} \left( \frac{1}{N}\sum_{i=1}^N \frac{\pi(a^i_h|x^i_h)}{1/K}f(x^i_{h+1}) - \frac{1}{N'}\sum_{i=1}^{N'} f(\tilde{x}_{h+1}^i)   \right), \nonumber
\end{align} where recall that $K = \abr{\Acal}$ and the importance weight $K\pi(a_h^i | x_h^i)$ is used to account for the fact that we draw actions uniformly from $\Acal$ but want to evaluate $\pi$. 
Though due to the max operator, ${\wipm}_{\Fcal_{h+1}}(\pi \vert \nu_h, \mu^\star_{h+1})$ is not an unbiased estimate of $d_{\Fcal_{h+1}}(\pi \vert \nu_h, \mu^\star_{h+1})$, in~\pref{app:proof_of_L2M_result}, we show that $\wipm_{\Fcal_{h+1}}$ indeed is a good approximation of $d_{\Fcal_{h+1}}$ via an application of the standard Bernstein's inequality and a union bound over $\Fcal_{h+1}$. Hence we can approximately minimize $\wipm_{\Fcal_{h+1}}$ with respect to $\pi$: $\min_{\pi\in\Pi} {\wipm}_{\Fcal_{h+1}}(\pi \vert \nu_h, \mu^{\star}_{h+1}) $,
resulting in a two-player min-max game. Intuitively, we can think of  $\pi$ as a generator, such that, conditioned on $\nu_h$, it generates next-step samples $x_{h+1}$ that are similar to the expert samples from $\mu^\star_{h+1}$, via fooling discriminators $\Fcal_{h+1}$. 

Note that the above formulation is a two-player game, with the utility function for  $\pi$ and $f$ defined as:
\begin{equation}
\label{eq:utility}
\resizebox{1.\hsize}{!}{$u(\pi, f) \defeq  \sum_{i=1}^N {K\pi(a^i_h|x^i_h)}f(x_{h+1}^i)/N - \sum_{i=1}^{N'} f(\tilde{x}_{h+1}^i)/{N'}$}.
\end{equation}


\begin{algorithm}[t]
\begin{algorithmic}[1]
\STATE Initialize $\pi^0\in \Pi$
\FOR{$n = 1$ to $T$}
    \STATE 
    \label{line:lp_oracle}
    $f^n = \arg\max_{f\in\Fcal} u(\pi^n, f)$ (LP Oracle) 
    \STATE 
    $u^n = u(\pi^n, f^n)$
    \STATE 
    $\pi^{n+1} = \arg\min_{\pi\in\Pi} \sum_{t=1}^n u(\pi, f^t) + \phi(\pi)$
    (Regularized CS Oracle)  \label{line:cs_oracle}
\ENDFOR
\STATE \textbf{Output}:  $\pi^{n^\star}$ with $n^\star = \arg\min_{n\in [T]} u^n$  \label{line:output_l2m}
\end{algorithmic}
\caption{Min-Max Game ($\Dcal^{\star}, \Dcal, \Pi, \Fcal, T$)}
\label{alg:l2m}
\end{algorithm}

\pref{alg:l2m} solves the minmax game $\min_{\pi}\max_{f} u(\pi,f)$ using no-regret online update on both $f$ and $\pi$. At iteration $n$, player $f$ plays the best-response via  $f_n = \arg\max_{f} u(\pi^n, f)$ (Line 3) and player $\pi$ plays the Follow-the-Regularized Leader (FTRL) \citep{shalev2012online} as $\pi^{n+1} = \sum_{t=1}^n u(\pi, f^t) + \phi(\pi)$ with $\phi$ being convex regularization (Line 5). Note that other no-regret online learning algorithms (e.g., replacing FTRL by incremental online learning algorithms like OGD \citep{Zinkevich2003_ICML} can also be used to approximately optimize the above min-max formulation.   After the end of~\pref{alg:l2m}, we output a policy $\pi$ among all computed policies $\{\pi^i\}_{i=1}^T$ such that $\wipm_{\Fcal_{h+1}}(\pi \vert \nu_n, \mu_{n+1}^\star)$ is minimized (Line 7).

Regarding the computation efficiency of~\pref{alg:l2m},  the best response computation on $f$ in Line 3 can be computed by a call to the LP Oracle, while FTRL on $\pi$ can be implemented by a call to the regularized CS Oracle. Regarding the statistical performance, we have the following theorem:
\begin{theorem}
\label{thm:L2M_result} Given $\epsilon\in (0,1], \delta\in(0,1]$, set $T = \Theta\left(\frac{4K^2}{\epsilon^2}\right)$, $N=N' =\Theta\left( \frac{K\log(|\Pi_h||\Fcal_{h+1}|/\delta)}{\epsilon^2} \right) $,~\pref{alg:l2m} outputs ${\pi}$ such that with probability at least $1-\delta$, 
\begin{align*}
    \abr{ d_{\Fcal_{h+1}}({\pi} \vert \nu_h,\mu^\star_{h+1}) - \min_{\pi'\in\Pi_{h}} d_{\Fcal_{h+1}}(\pi' \vert \nu_h, \mu^\star_{h+1}) } \leq O(\epsilon).
\end{align*}
\end{theorem}The proof of the above theorem is included in~\pref{app:proof_of_L2M_result}, which combines standard min-max theorem and uniform convergence analysis. The above theorem essentially shows that~\pref{alg:l2m} successfully finds a policy ${\pi}$ whose resulting IPM is close to the smallest possible IPM one could achieve if one had access to $d_{\Fcal_{h+1}}(\pi\vert \nu_h, \mu^\star_{h+1}) $ directly, up to $\epsilon$ error.  Intuitively, from~\pref{thm:L2M_result}, we can see that if $\nu_h$---the observation distribution resulting from fixed policies $\{\pi_1,\dots, \pi_{h-1}\}$, is similar to $\mu^\star_h$, then we guarantee to learn a policy ${\pi}$, such that the new sequence of policies $\{\pi_1, \dots, \pi_{h-1}, {\pi}\}$ will generate a new distribution $\nu_{h+1}$ that is close to $\mu_{h+1}^\star$, in terms of IPM with $\Fcal_{h+1}$.  The algorithm introduced below is based on this intuition.

\subsection{Forward Adversarial Imitation Learning}

\pref{thm:L2M_result} indicates that conditioned on $\{\pi_1,\dots, \pi_{h-1}\}$ being fixed, \pref{alg:l2m} finds a policy $\pi\in\Pi_h$ such that it approximately minimizes the divergence---measured under IPM with $\Fcal_{h+1}$, between the observation distribution $\nu_{h+1}$ resulting from $\{\pi_1,\dots,\pi_{h-1}, \pi\}$,  and the corresponding  distribution $\mu^\star_{h+1}$ from expert.

\begin{algorithm}
\begin{algorithmic}[1]
\STATE Set $\bm\pi = \emptyset$
\FOR{$h = 1$ to $H-1$}
    \STATE Extract expert's data at $h+1$: $\tilde{\Dcal} = \{\tilde{x}_{h+1}^i\}_{i=1}^{n'}$ 
    \STATE $\Dcal = \emptyset$
    \FOR{$i = 1$ to $n$}
        \STATE Reset $x^{(i)}_1\sim \rho$
        \STATE Execute $\bm\pi=\{\pi_1,\dots,\pi_{h-1}\}$ to generate state $x_{h}^{i}$
        \STATE Execute $a_h^i\sim U(\Acal)$ to generate $x_{h+1}^{i}$ and add $(x_h^i, a_h^i,x_{h+1}^i)$ to $\Dcal$
    \ENDFOR
    \STATE Set $\pi_h$ to be the return of~\pref{alg:l2m} with inputs $\left(\tilde{\Dcal}, \Dcal, \Pi_h, \Fcal_{h+1}, T\right)$
    \STATE Append $\pi_h$ to $\bm\pi$
\ENDFOR
\end{algorithmic}
\caption{\fail ($\{\Pi_h\}_h$, $\{\Fcal_h\}_h$, $\epsilon, n,n', T$)}
\label{alg:main_alg}
\end{algorithm}

With~\pref{alg:l2m} as the building block, we now introduce our model-free algorithm---Forward Adversarial Imitation Learning (\fail) in ~\pref{alg:main_alg}. \pref{alg:main_alg} integrates~\pref{alg:l2m} into the Forward Training framework \citep{Ross2010}, by decomposing the sequential learning problem into $H$ many independent distribution matching problems where each one is solved using~\pref{alg:l2m} independently. Every time step $h$, \fail{} assumes that $\pi_1,\dots, \pi_{h-1}$ have been correctly learned in the sense that the resulting observation distribution $\nu_h$ from $\{\pi_1,\dots,\pi_{h-1}\}$ is close to $\mu_h^\star$ from expert. Therefore, \fail{} is only required to focus on learning $\pi_{h}$ correctly conditioned on $\{\pi_1,\dots,\pi_{h-1}\}$ being fixed, such that $v_{h+1}$ is close to $\mu_{h+1}^\star$, in terms of the IPM with $\Fcal_{h+1}$.  Intuitively, when one has a strong class of discriminators, and the two-player game in each time step is solved near optimally, then by induction from $h=1$ to $H$, \fail{} should be able to learn a sequence of policies such that $\nu_h$ is close to $\mu_{h}^\star$ for all $h\in [H]$ (for the base case, we simply have $\nu_1 = \mu_1^\star = \rho$). 

\subsection{Analysis of~\pref{alg:main_alg}} 

The performance of \fail{} crucially depends on the capacity of the discriminators. Intuitively, discriminators that are too strong cause overfitting (unless one has extremely large number of samples). Conversely, discriminators that are too weak will not be able to distinguish $\nu_h$ from $\mu_{h}^\star$. This dilemma was studied in the Generative Adversarial Network (GAN) literature already by \citet{arora2017generalization}. Below we study this tradeoff explicitly in IL. 

To quantify the power of discriminator class $\Fcal_h$ for all $h$, we use \emph{inherent Bellman Error} (\textsc{iBE}) with respect to $\bm\pi^\star$:
\begin{align}
    \epsilon_{\mathrm{be}} =\max_h\left( \max_{g\in\Fcal_{h+1}}\min_{f\in\Fcal_{h}} \| f - \Gamma_h g\|_{\infty}\right).
    \label{eq:be_classic}
\end{align} The Inherent Bellman Error is commonly used in approximate value iteration literature \citep{munos2005error,munos2008finite,lazaric2016analysis} and policy evaluation literature \citep{Sutton1998}. It measures the worst possible projection error when projecting $\Gamma_h g$  to function space $\Fcal_h$.  Intuitively increasing the capacity of $\Fcal_h$ reduces $\ibe $.


 Using a restricted function class $\Fcal$ potentially introduces $\ibe$, hence one may tend to set $\Fcal_h$ to be infinitely powerful discriminator class such as function class consisting of all bounded functions $\{f: \|f\|_{\infty}\leq c\}$ (recall IPM becomes total variation in this case). However, using  $\Fcal_h \defeq \{f: \|f\|_{\infty}\leq c \} $ makes efficient learning impossible. The following proposition excludes the possibility of sample efficiency with discriminator class  being $\{f:\|f\|_{\infty}\leq c\}$.
 
 \begin{theorem}
[Infinite Capacity $\Fcal$ does not generalize]
There exists a MDP with $H=2$, a policy set $\Pi = \{\bm\pi, \bm\pi'\}$, an expert policy $\bm\pi^\star$ with $\bm\pi = \bm\pi^\star$ (i.e., $\Pi$ is realizable), such that for datasets $\Dcal^{\star} = \{\tilde{x}^{i}_2\}_{i=1}^M$ with $\tilde{x}^{i}_2\sim \mu^{\star}_2$, $\Dcal = \{x^{i}_2\}_{i=1}^M$ with $x^{i}_2\sim \mu^{\bm\pi}_2$, and $\Dcal' = \{x'^{(i)}_2\}_{i=1}^M$ with $x'^{(i)}_2\sim \mu^{\bm\pi'}_2$,  as long as $M = O(\log (|\Xcal|))$, we must have:
\begin{align*}
    \lim_{|\Xcal|\to\infty} \mathrm{P}(\Dcal^\star \cap \Dcal = \emptyset) = 1, \; \lim_{|\Xcal|\to\infty} \mathrm{P}(\Dcal^\star \cap \Dcal' = \emptyset) = 1. 
\end{align*} Namely, denote $\hat{\Dcal}$ as the empirical distribution of a dataset $\Dcal$ by assigning probability $1/|\Dcal|$ to any sample, we have:
\begin{align*}
    &\lim_{|\Xcal|\to\infty} \|\hat{\Dcal}^\star - \hat{\Dcal}\|_1 = 2 , 
     \lim_{|\Xcal|\to\infty}\|\hat{\Dcal}^\star - \hat{\Dcal}'\|_1 = 2.
\end{align*}
\label{thm:non_generalize}
\vspace{-10pt}
\end{theorem}
The above theorem shows by just looking at the samples generated from $\bm\pi$ and $\bm\pi'$, and comparing them to the samples generated from the expert policy $\bm\pi^\star$ using $\{f:\|f\|_{\infty} \leq c\}$ (IPM becomes Total variation here), we cannot distinguish $\bm\pi$ from $\bm\pi'$, as they both look similar to $\bm\pi^\star$,  
i.e., none of the three datasets overlap with each other, resulting the TV distances between the empirical distributions become constants, \emph{unless} the sample size scales $\Omega(\mathrm{poly}(|\Xcal|))$. 

 
\pref{thm:non_generalize} suggests that one should explicitly regularize discriminator class so that it has finite capacity (e.g., bounded VC or Rademacher Complexity). The restricted discriminator class $\Fcal$ has been widely used in practice as well such as learning generative models (i.e., Wasserstain GANs \cite{arjovsky2017wasserstein}). 
Denote $|\Pi| = \max_h |\Pi_h|$ and $|\Fcal| = \max_{h} |\Fcal_h|$. The following theorem shows that the learned time-dependent policies $\bm\pi$'s performance is close to the expert's performance:
\begin{theorem}
[Sample Complexity of \fail]
Under~\pref{ass:realizable}, for any $\epsilon,\delta\in (0,1]$, set
$T = \Theta(\frac{K}{\epsilon^2})$, $n=n' = \Theta(\frac{K \log(|\Pi||\Fcal|H/\delta)}{\epsilon^2})$, with probability at least $1-\delta$, \fail{} (\pref{alg:main_alg}) outputs $\bm\pi$, such that, 
\begin{align*}
J(\bm\pi) - J(\bm\pi^\star) \leq H^2\epsilon'_{\mathrm{be}} + H^2\epsilon,
\end{align*}  by using 
    $\tilde{O}\left( \frac{H K}{\epsilon^2} \log\left( \frac{|\Pi||\Fcal|}{\delta}\right)   \right)$
\footnote{In $\tilde{O}$, we drop log terms that does not dependent on $|\Pi|$ or $|\Fcal|$. In $\Theta$ we drop constants that do not depend on $H,K, |\Xcal|, |\Pi|, |\Fcal|, 1/\epsilon, 1/\delta$.  Details can be found in Appendix.}
many trajectories with an average inherent Bellman Error $\epsilon'_{\mathrm{be}}$:
\begin{align*}
\epsilon'_{\mathrm{be}} \defeq \max_h \max_{g\in\Fcal_{h+1}}\min_{f\in\Fcal_h}\EE_{x\sim (\mu_{h}^{\bm\pi} + \mu_h^\star)/2 }[|f(x) - (\Gamma_h g)(x)|].
\end{align*} 
\label{thm:ft_main_theorem}
\vspace{-10pt}
\end{theorem}
Note that the average inherent Bellman error  $\epsilon'_{\mathrm{be}}$ defined above is averaged over the state distribution of the learned policy $\bm{\pi}$ and the state distribution of the expert, which is guaranteed to be smaller than the classic inherent Bellman error used in RL literature (i.e.,~\pref{eq:be_classic}) which uses infinity norm over $\Xcal$. The proof of~\pref{thm:ft_main_theorem} is included in~\pref{app:proof_main_theorem}.
Regarding computational complexity of~\pref{alg:main_alg}, we can see that it requires polynomial number of calls (with respect to parameters $H, K, 1/\epsilon$) to the efficient oracles (Regularized CS oracle and LP oracle).  
Since our analysis only uses uniform convergence analysis and standard concentration inequalities, extension to continuous $\Pi$ and $\Fcal$ with complexity measure such as VC-dimension, Rademacher complexity, and covering number is standard.  We give an example in~\pref{sec:metric_MDP}.

\section{The Gap Between \ilo{} and RL}


To quantify the gap between RL and \ilo, below we  present an exponential separation between \ilo{} and RL in terms of sample complexity to learn a near-optimal policy. We assume that the expert policy is optimal.


\begin{proposition} 
[Exponential Separation Between RL and \ilo]
Fix $H \in\mathbb{N}^{+}$ and $\epsilon\in (0,\sqrt{1/8})$. There exists a family of MDP with deterministic dynamics, with horizon $H$, $2^{H}-1$ many states, two different actions, such that for any RL algorithm, the probability of outputting a policy $\bm{\hat{\pi}}$ with $J(\hat{\bm\pi})\leq J(\bm\pi^\star) + \epsilon$ after collecting T trajectories is at most $2/3$ for all $T\leq O(2^H/\epsilon^2)$. On the other hand, for the same MDP, given one trajectory of observations $\{\tilde{x}_h\}_{h=1}^H$ from the expert policy $\bm\pi^\star$, there exists an algorithm that deterministically outputs $\bm\pi^\star$ after collecting $O(H)$ trajectories. 
\label{prop:seperate_IL_and_RL}
\vspace{-14pt}
\end{proposition}
\pref{prop:seperate_IL_and_RL} shows having access to expert's trajectories of observations allows us to efficiently solve some MDPs that are otherwise provably intractable for any RL algorithm (i.e., requiring exponentially many trajectories to find a near optimal policy). 
This kind of exponential gap previously was studied in the interactive imitation learning setting where the expert also provides action signals \cite{sun2017deeply} and one can query the expert's action at any time step during training. 
To the best of our knowledge, this is the \emph{first exponential gap} in terms of sample efficiency between \ilo{} and RL. Note that our theorem applies to a specific family of purposefully designed MDPs, which is standard for proving information-theoretical lower bounds. 

\section{Case Study}
\label{sec:case_study}

In this section, we study three settings where inherent Bellman Error will disappear even under restricted discriminator class : (1) Lipschitz Continuous MDPs (e.g., \cite{kakade2003exploration}),  (2) Interactive Imitation Learning from Observation where expert is available to query during training, and (3) state abstraction. 

\subsection{Lipschitz Continuous MDPs}
\label{sec:metric_MDP}

We consider a setting where cost functions, dynamics and $\pi^\star_h$ are Lipschitz continuous in metric space $(\Xcal, d)$:
\begin{align*}
    & \|P(\cdot|x,a) - P(\cdot|x',a) \|_1 \leq L_P d(x, x'), \\
    &  \| \pi^\star_h(\cdot|x) - \pi^\star_h(\cdot|x')\|_1 \leq L_{\pi} d(x,x'),
\end{align*} for the known metric $d$ and Lipschitz constants $L_P, L_\pi$. 
Under this condition, the Bellman operator with respect to $\bm\pi^\star$ is Lipschitz continuous in the metric space $(\Xcal,d)$:
$\left\lvert\Gamma_h f(x_1) - \Gamma_h f(x_2) \right\rvert 
\leq (\|f\|_{\infty} (L_P+L_{\pi})) d(x_1, x_2)$, where we applied Holder's inequality. Hence, we can design the function class $\Fcal_h$ for all $h\in [H]$ as follows:
\begin{align}
\Fcal_{h} = \{f: \|f\|_L \leq (L_P+L_\pi), \|f\|_{\infty}\leq 1 \}, 
\label{eq:special_F}
\end{align} which will give us $\epsilon_{\mathrm{be}} = 0$ and $V_h^\star \in \Fcal_h$ due to the assumption on the cost function. Namely $\Fcal_h$ is the class of functions with bounded Lipschitz constant and bounded value. This class of functions is widely used in practice for learning generative models (e.g., Wasserstain GAN). Note that this setting was also studied in \cite{munos2008finite} for the Fitted Value Iteration algorithm.

 Denote $L \defeq L_P + L_{\pi}$. For $\Fcal = \{f: \|f\|_L \leq L, \|f\|_{\infty}\leq 1\}$ we show that we can evaluate the empirical IPM $\sup_{f\in\Fcal} \left(\sum_{i=1}^N f(x_i) /N- \sum_{i=1}^{N'} f(x_i')/{N'}\right)$ by reducing it to Linear Programming, of which the details are deferred to~\pref{app:proof_of_claim}.
Regarding the generalization ability, note that our function class $\Fcal$ is a subset of all functions with bounded Lipschitz constant, i.e., $\Fcal \subset \{f: \|f\|_{L}\leq L\}$. The Rademacher complexity for bounded Lipschitz function class grows in the order of $O(N^{-1/\mathrm{cov}(\Xcal)})$ (e.g., see \citep{luxburg2004distance,sriperumbudur2012empirical}), with $\mathrm{cov}(\Xcal)$ being the covering dimension of the metric space $(\Xcal, d)$.\footnote{Covering dimension is defined as $\mathrm{cov}(\Xcal) \defeq \inf_{d > 0} \{N_{\epsilon}(\Xcal) \leq \epsilon^{-d}, \forall \epsilon > 0\}$, where $N_{\epsilon}(\Xcal)$ is the size of the minimum $\epsilon$-net of metric space $(\Xcal,\Dcal)$.} 
Extending~\pref{thm:ft_main_theorem} to Lipschitz continuous MDPs, we have the following corollary. 
\begin{corollary}[Sample Complexity of \fail{} for Lipschitz Continuous MDPs]
\label{corr:metric_MDP}
With the above set up on Lipschitz continuous MDP and $\Fcal_h$ for $h\in [H]$ \pref{eq:special_F}, given $\epsilon,\delta\in (0,1]$, set $T = \Theta(\frac{K}{\epsilon^2})$, $n=n' =\tilde{\Theta}(\frac{K (LK)^{\mathrm{cov}(\Xcal)} \log(|\Pi|/\delta)}{\epsilon^{2+\mathrm{cov}(\Xcal)}})$, then with probability at least $1-\delta$,  \fail{} (\pref{alg:main_alg}) outputs a policy with $J(\bm{\pi}) - J(\bm{\pi}^\star) \leq O(H^2\epsilon)$ using at most $\tilde{O}\left( \frac{HK(KL)^{\mathrm{cov}(\Xcal)}}{\epsilon^{(2+\mathrm{cov}(\Xcal))}} \log\left(\frac{|\Pi|}{\delta\epsilon}\right) \right)$ many trajectories.
\end{corollary}
The proof of the above corollary is deferred to~\pref{app:proof_of_metric_MDP} which uses a standard covering number argument over $\Fcal_h$ with norm $\|\cdot\|_{\infty}$. Note that we get rid of $\nameibe$ here and hence as the number of sample grows, \fail{} approaches to the global optimality.  Though the bound has an exponential dependency on the covering dimension, note that the covering dimension $\mathrm{cov}(\Xcal)$ is completely dependent  on the underlying metric space $(\Xcal,d)$ and could be much smaller than the real dimension of $\Xcal$.  Note that the above theorem also serves an example regarding how we can extend~\pref{thm:ft_main_theorem} to settings where $\Fcal$ contains infinitely many functions but with bounded statistical complexity (similar techniques can be used for $\Pi$ as well).

\subsection{Interactive Imitation Learning from Observations}

We can avoid \nameibe{} in an interactive learning setting, where we assume that we can query expert during training. But different from previous interactive imitation learning setting such as AggreVaTe, LOLS \citep{ross2014reinforcement,chang2015learning}, and DAgger \citep{Ross2011_AISTATS}, here we do not assume that expert provides actions nor cost signals. Given any observation $x$, we simply ask expert to take over for just one step, and observe the observation at the next step, i.e., $x' \sim P(\cdot| x, a)$ with  $a\sim \pi^\star(\cdot|x)$. Note that compared to the non-interactive setting, interactive setting assumes a much stronger access to expert. 
In this setting, we can use arbitrary class of discriminators with bounded complexity. Due to space limit, we defer the detailed description of the interactive version of \fail{} (\pref{alg:main_alg_2}) to~\pref{app:ifail}.  The following theorem states that we can avoid $\nameibe{}$:

\begin{theorem}
Under~\pref{ass:realizable} and the existence of an interactive expert, for any $\epsilon\in(0,1]$ and $\delta \in (0,1]$, set $T = \Theta(\frac{K}{\epsilon^2})$, $n = \Theta(\frac{K\log(|\Pi||\Fcal|H)/\delta}{\epsilon^2})$, with probability at least $1-\delta$, \pref{alg:main_alg_2}) outputs a policy $\bm\pi$ such that:
 \begin{align*}
     J(\bm\pi) - J(\bm\pi^\star)\leq H\epsilon, 
 \end{align*} by using at most 
$     \tilde{O}\left( \frac{HK}{\epsilon^2}\log\left(\frac{|\Pi||\Fcal|}{\delta}\right)\right)$many trajectories.
\label{thm:sample_complexity_interactive}
\end{theorem}
Compare to the non-interactive setting, we get rid of $\nameibe$, at the cost of a much stronger expert. 


\subsection{State Abstraction}
Denote $\phi: \Xcal \to \Scal$ as the abstraction that maps $\Xcal$ to a discrete set $\Scal$. We assume that the abstraction satisfies the Bisimulation property \citep{givan2003equivalence}: for any $x, x'\in\Xcal$, if $\phi(x) = \phi(x')$, we have that $ \forall s\in\Scal, a\in\Acal, h\in [H]$:
\begin{align}
    &c(x) = c(x'), \;\; 
     \pi^\star_h(a|x) = \pi^\star_h(a|x')\notag\\
    & \sum_{x'' \in \phi^{-1}(s)}P(x''| x,a) = \sum_{x''\in \phi^{-1}(s)}P(x''|x',a).
    \label{eq:bisimulation}
\end{align} In this case, one can show that the $V^{\star}$ is piece-wise constant under the partition induced from $\phi$, i.e., $V^{\star}(x) = V^\star(x')$ if $\phi(x) = \phi(x')$. Leveraging the abstraction $\phi$, we can then design $\Fcal_h = \{f: \|f\|_{\infty}\leq 1, f(x) = f(x'), \forall x,x', \text{s.t., } \phi(x) = \phi(x')\}$. Namely, $\Fcal_h$ contains piece-wise constant functions with bounded values. Under this setting, we can show that the inherent Bellman Error is zero as well (see Proposition 9 from \citet{chen2019information}). Also $\sup_{f\in\Fcal_h}u(\pi, f)$ can be again computed via LP and $\Fcal_h$ has Rademacher complexity scales $O(\sqrt{|\Scal|/N})$ with $N$ being the number of samples. Details are in~\pref{app:abstraction}.

\vspace{-5pt}
\section{Discussion on Related Work}
\label{sec:more_related_work}
Some previous works use the idea of learning an inverse model to predict actions (or latent causes) \citep{nair2017combining,torabi2018behavioral} from two successive observations and then use the learned inverse model to generate actions using expert observation demonstrations. With the inferred actions, it reduces the problem to normal imitation learning.  We note here that learning an inverse model is ill-defined. Specifically, simply by the Bayes rule, the inverse model $P(a | x_h, x_{h+1})$---the probability of action $a$ was executed at $x_h$ such that the system generated $x_{h+1}$, is equivalent to $P(a | x_h, x_{h+1}) \propto P(x_{h+1} | x_h, a) P(a | x_h)$, 
 i.e., an inverse model $P(a|x_h, x_{h+1})$ is explicitly dependent on the action generation policy $P(a|x_{h})$. Unlike $P(x_{h+1}|x_h, a)$, without  the policy $P(a|x_h)$, the inverse model is ill-defined by itself alone.  This means that if one wants to learn an inverse model that predicts expert actions along the trajectory of observations generated by the expert, one would need to learn an inverse model, denoted as $P^\star(a|x_h,x_{h+1})$, such that $P^\star(a|x_h, x_{h+1}) \propto P(x_{h+1} | x_h, a) \pi_h^\star(a | x_h)$, which indicates that one needs to collect actions from $\pi_h^\star$.  An inverse model  makes sense when the dynamics is deterministic and bijective. 
 Hence replying on learning such an inverse model $P^\star(a|x,x')$ will not provide any performance guarantees in general, unless we have actions collected from $\pi^\star$.  

\vspace{-5pt}
\section{Simulation}
We test \fail{} on three simulations from openAI Gym \citep{brockman2016openai}: Swimmer, Reacher, and the Fetch Robot Reach task (FetchReach). For Swimmer we set H to be 100 while for Reacher and FetchReach, $H$ is 50 in default. The Swimmer task has dense reward (i.e.,., reward at every time step). For reacher, we try both dense reward and sparse reward (i.e., success if it reaches to the goal within a threshold). FetchReach is a sparse reward task.  As our algorithm is presented for discrete action space, for all three tasks, we discrete the action space via discretizing each dimension into 5 numbers and applying categorical distribution independently for each dimension.\footnote{i.e., $\pi(a | x) = \prod_{i=1}^d \pi_i(a[i]| x)$, with $a[i]$ stands for the $i$-th dimension. Note that  common implementation for continuous control often assumes such factorization across action dimensions as the covariance matrix of the Gaussian distribution is often diagonal. 
See comments in~\pref{app:additional} for extending \fail{} to continuous control setting in practice.  
} \footnote{Implementation and scripts for reproducing results can be found at \url{https://github.com/wensun/Imitation-Learning-from-Observation}}

For baseline, we modify GAIL \citep{ho2016generative}, a model-free IL algorithm, based on the implementation from OpenAI Baselines, to make it work for \ilo. We delete the input of actions to discriminators in GAIL to make it work for \ilo.  Hence the modified version can be understood as using RL (the implementation from OpenAI uses TRPO \cite{schulman2015trust}) methods to minimize the divergence between the learner's average state distribution and the expert's average state distribution. 

For \fail{} implementation, we use MMD as a special IPM, where we use RBF kernel and set the width using the common median trick (e.g.,\cite{fukumizu2009kernel}) without any future tuning. All policies are parameterized by one-layer neural networks. Instead of using FTRL, we use ADAM as an incremental no-regret learner, with all default parameters (e.g., learning rate) \citep{kingma2014adam}.  The total number of iteration $T$ in~\pref{alg:l2m} is set to 1000 without any future tuning.  During experiments, we stick to default hyper-parameters  for the purpose of  best reflecting the algorithmic contribution of \fail.\footnote{We also refer readers to~\pref{app:additional} for additional experimental results on a variant of \fail{} (\pref{alg:main_alg_2}).}  All the results below are averaged over ten random trials with seeds randomly generated between $[1, 1e6]$. The experts are trained via a reinforcement learning algorithm (TRPO \cite{schulman2015trust}) with multiple millions of samples till convergence. 

\begin{figure}[t]
\begin{subfigure}{.235\textwidth}
  \includegraphics[width=1.0\linewidth]{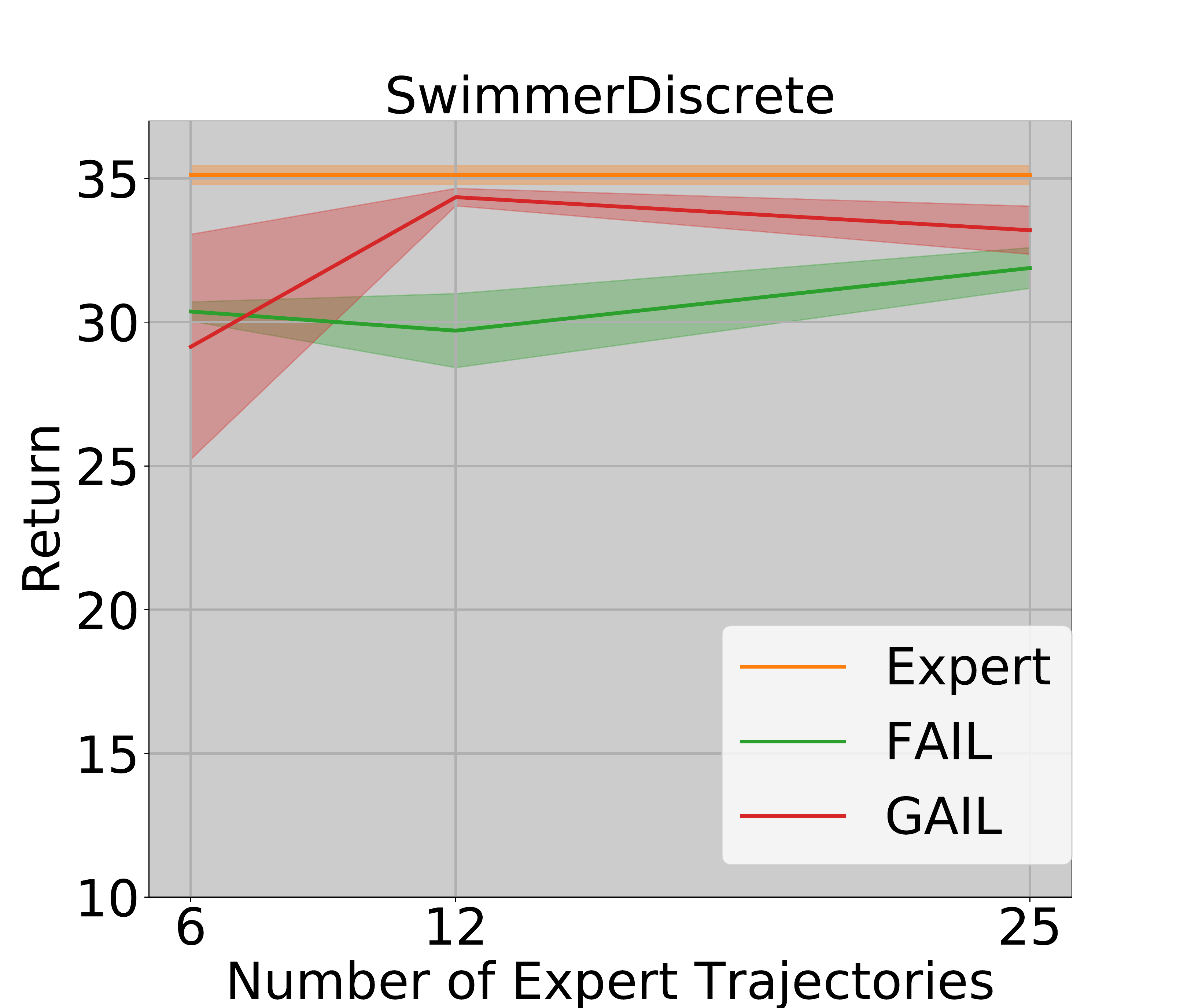}
  \caption{Swimmer (1m)}
  \label{fig:hopper_onem}
\end{subfigure}
\begin{subfigure}{.235\textwidth}
  \includegraphics[width=1.0\linewidth]{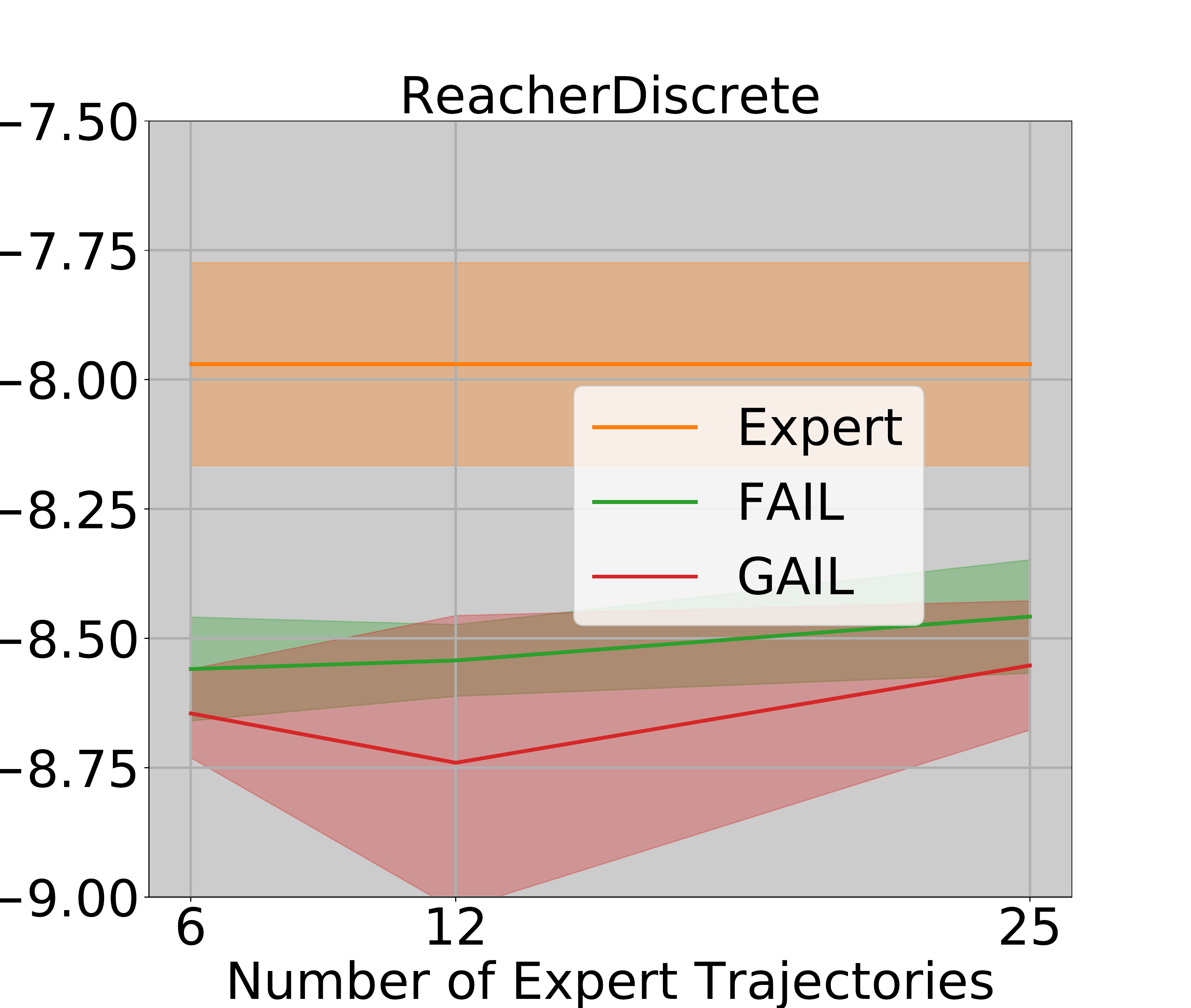}
  \caption{Reacher (1m)}
  \label{fig:reacher_onem}
\end{subfigure}
\begin{subfigure}{.235\textwidth}
  \centering
  \includegraphics[width=1.0\linewidth]{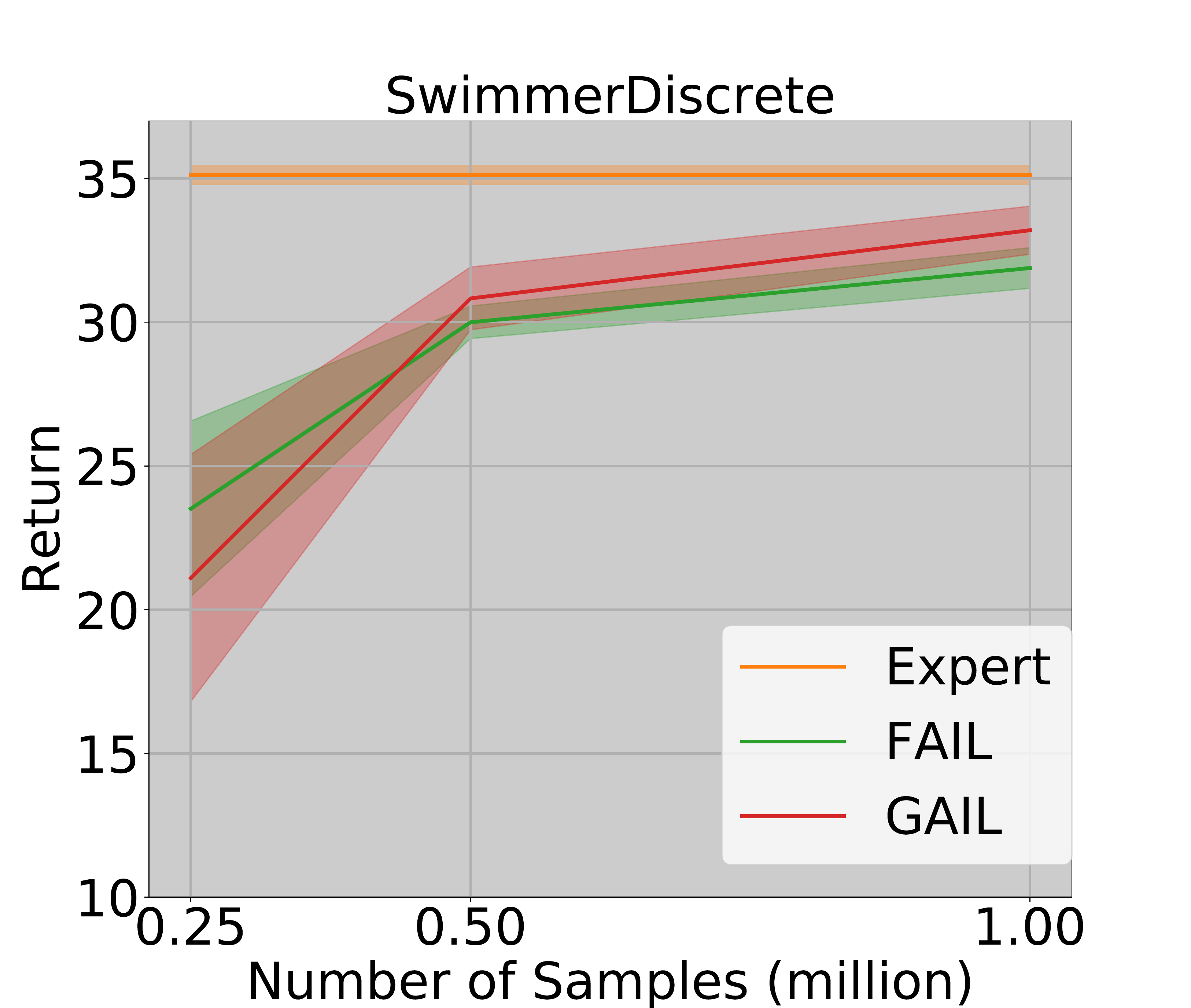}
  \caption{swimmer (25)}
  \label{fig:swimmer_fixed_expert}
\end{subfigure}
\begin{subfigure}{.235\textwidth}
  \centering
  \includegraphics[width=1.0\linewidth]{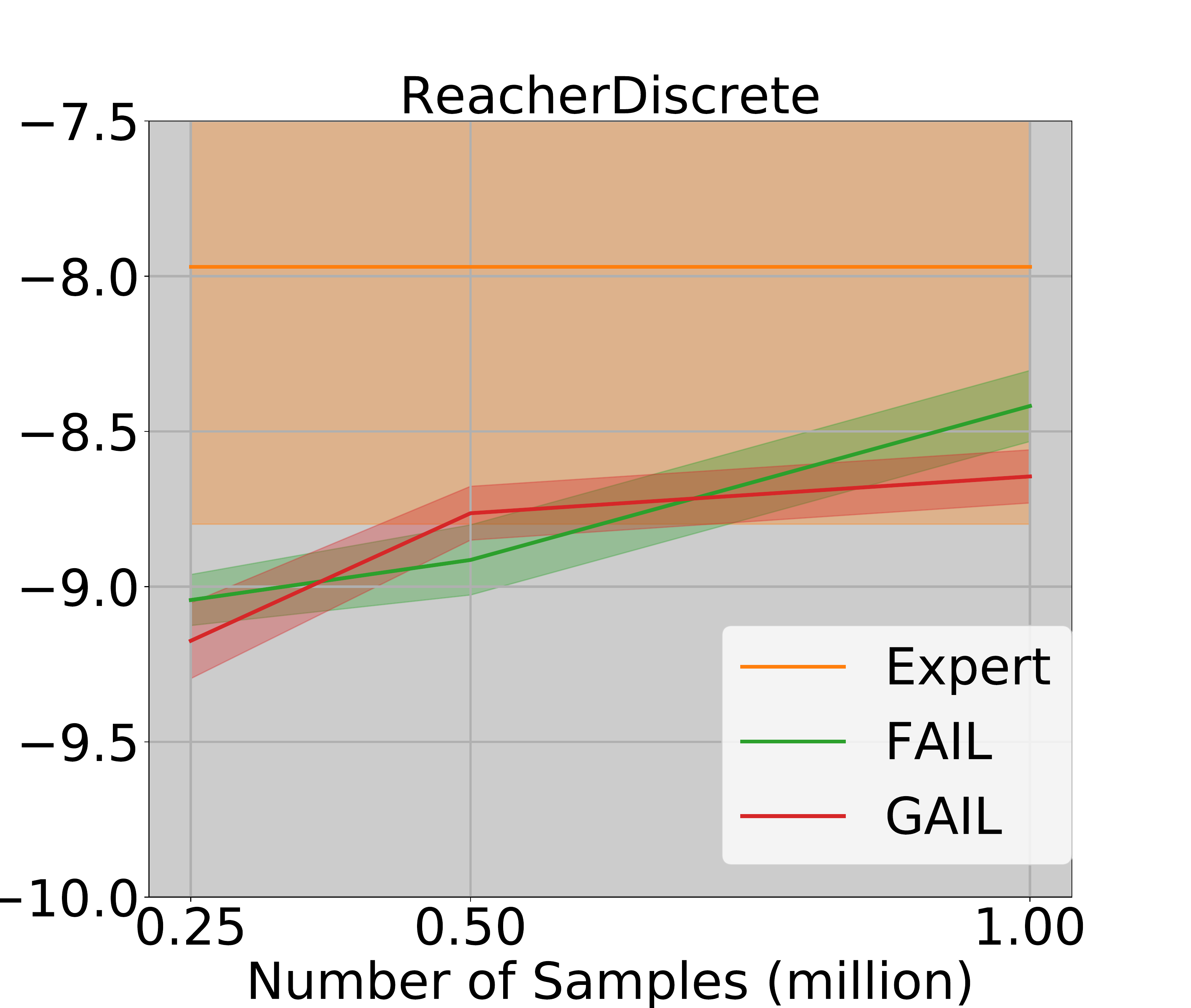}
  \caption{Reacher (25)}
  \label{fig:reacher_fixed_expert}
\end{subfigure}

\caption{Performance of expert, \fail, and GAIL (without actions) on dense reward tasks (Reacher and Hopper). For (a) and (b), we fix the number of training samples while varying the number of expert demonstrations (6, 12, 25).  For (c) and (d), we fix the number of expert trajectories, while varying the training samples.}
\label{fig:comparison_fixed_training_samples}
\vspace{-10pt}
\end{figure}

\begin{figure}[t]
\centering
\begin{subfigure}{.235\textwidth}
  \centering
  \includegraphics[width=1.0\linewidth]{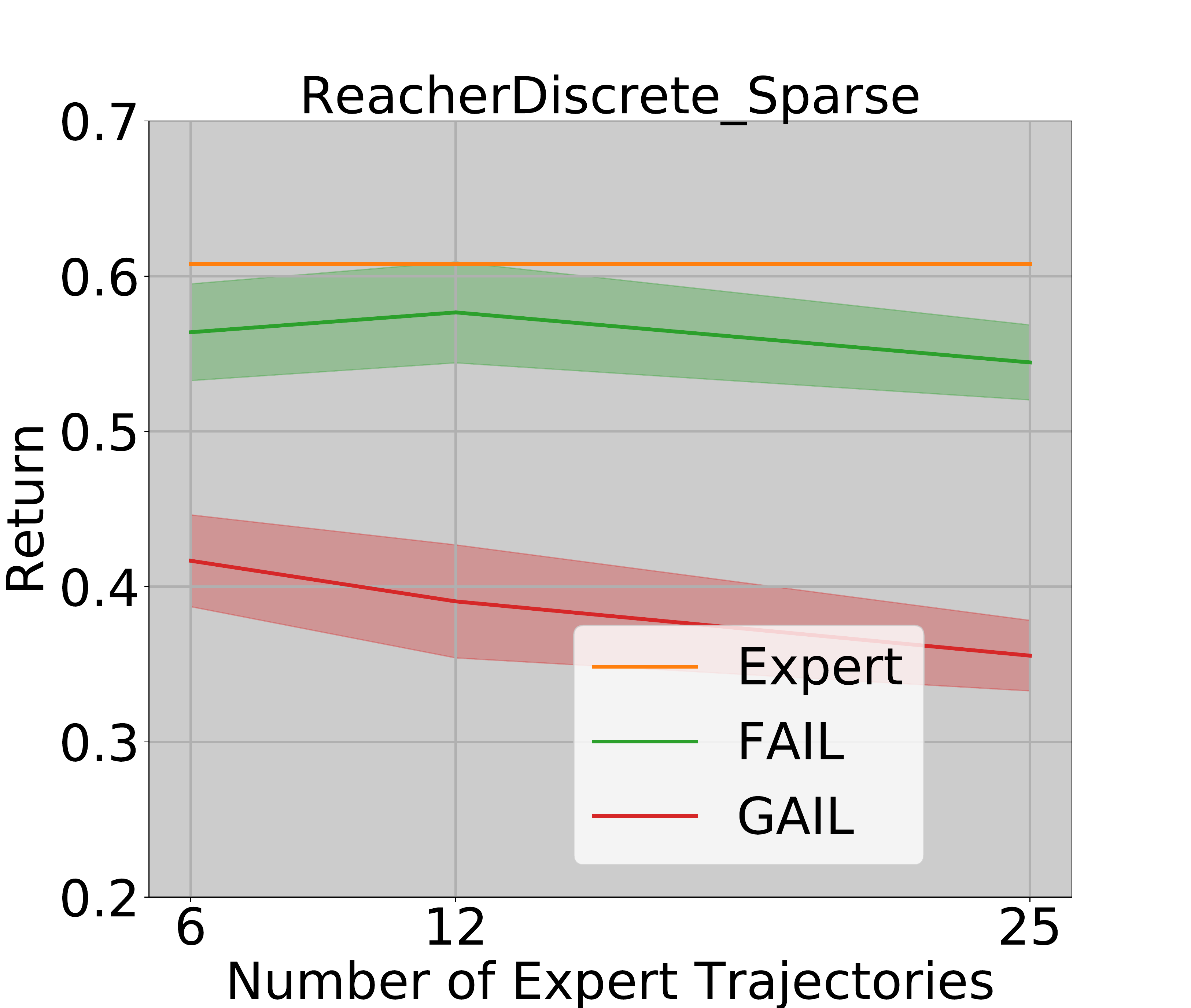}
  \caption{Reacher Sparse (1m)}
  \label{fig:reacher_sparse}
\end{subfigure}%
\begin{subfigure}{.235\textwidth}
  \centering
  \includegraphics[width=1.0\linewidth]{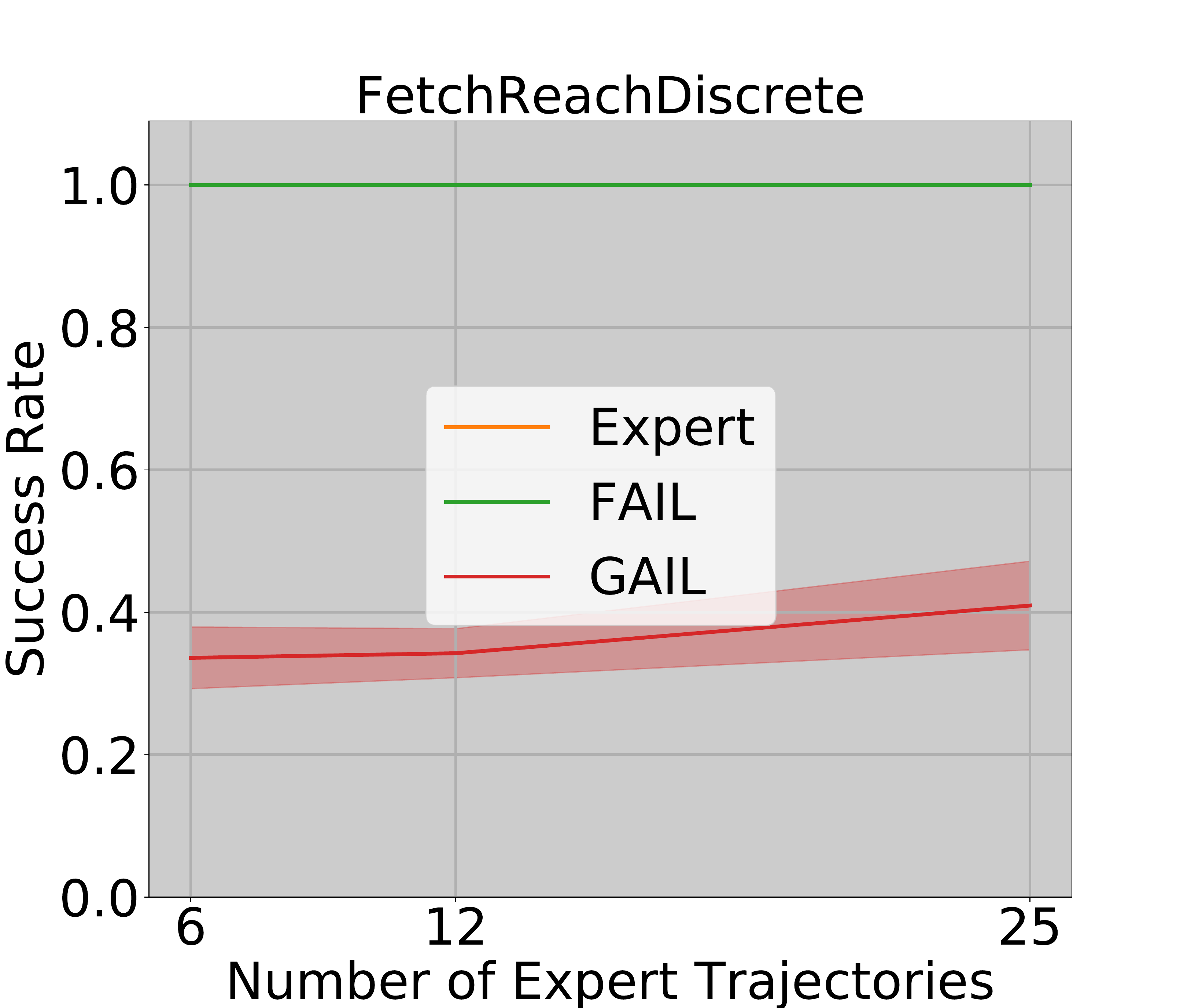}
  \caption{FetchReach (1m)}
  \label{fig:fetch}
\end{subfigure}
\caption{Performance of expert, \fail, and GAIL (without actions) on two sparse control tasks (Reacher Sparse and Fetch-Reach).  We fix the number of training samples while varying the number of expert demonstrations (6, 12, 25). }
\label{fig:sparse}
\vspace{-10pt}
\end{figure}

\pref{fig:comparison_fixed_training_samples} shows the comparison of expert, \fail, and GAIL on two dense reward tasks with different number of expert demonstrations, under fixed total number of training samples (one million).  We report mean and standard error in~\pref{fig:comparison_fixed_training_samples}. We observe GAIL outperforms \fail{}  in Swimmer on some configurations, while \fail{} outperforms GAIL on Reacher (Dense reward) for all configuration. 

\pref{fig:sparse} shows the comparison of expert, \fail, and GAIL on two sparse reward settings. We observe that \fail{} significantly outperforms GAIL on these two sparse reward tasks. For sparse reward, note that what really matters is to reach the target at the end, \fail{} achieves that by matching expert's state distribution and learner's state distribution one by one at every time step till the end, while GAIL (without actions) loses the sense of ordering by focusing on the average state distributions.

The above experiments also indicates that \fail{} in general can work well for shorter horizon (e.g., $H=50$ for Reacher and Fetch), while shows much less improvement over GAIL on longer horizon task. We think this is because of the nature of \fail{} which has to learn a sequence of time-dependent policies along the entire horizon $H$. Long horizon requires larger number of samples. While method like GAIL learns a single stationary policy with all training data, and hence is less sensitive to horizon increase. We leave extending \fail{} to learning a single stationary policy as a future work.



\section{Model-based FAIL}
Due to the possible non-zero inherent Bellman error, \fail{} in general cannot guarantee to find the globally optimal policy. So the remaining question is that:

\begin{displayquote} \emph{In \ilo{}, under a realizability assumption, does there exist an algorithm that can learn an $\epsilon$ near-optimal policy with sample complexity scales 
polynomially with respect to horizon, number of actions, $1/\epsilon$, and statistical complexities of function classes, with high probability?}
\end{displayquote}

While we are not able to answer this question in the model-free setting considered in this work,  we study a model-based algorithm which we show can achieve the above sample complexity, though it is computationally not efficient. 

Rather than starting with a realizable policy class $\Pi$, our model-based algorithm starts with a class of models $\Pcal$ with $P \in \Pcal$, i.e., the model class contains the true model $P$ (realizability). Similarly, we assume we have a set of discriminators $\{\Fcal_h\}_{h=1}^H$ such that $V_h^*\in \Fcal_h$. The above is the realizability assumption in the model-based setting. 

Note that intuitively we use discriminators to approximate expert's value functions. For any $\hat{P}\in \Pcal, f\in \Fcal_{h+1}$, define $Q_h^{\hat{P}, f}(x,a) \defeq \EE_{x'\sim \hat{P}_{x,a}}f(x')$ for any $(x,a)$. Denote $\Qcal_h$ as:
\begin{align}
\label{eq:Q_functions}
    \Qcal_h \defeq \left\{ Q_h^{\hat{P}, f} \right\}_{\hat{P}\in\Pcal, f\in\Fcal_{h+1}}, \forall h\in [H-1].
\end{align} Note that due to the assumption that $V_{h+1}^*\in \Fcal_{h+1}$ and $P\in\Pcal$, we must have $Q_{h}^\star = Q_h^{P, V_{h+1}^\star}\in \Qcal_h$, i.e., $\Qcal_h$ is realizable.  Each $Q_h$ induces a policy $\pi_h^{Q}(x) = \arg\min_{a} Q_h(x,a)$. Denote a policy class $\widetilde{\Pi}_h$ as follows:
\begin{align}
    \label{eq:policy_class}
    \widetilde{\Pi}_h \defeq \left\{ \pi_h^{Q}\right\}_{Q\in\Qcal_h}, \forall h\in [H-1].
\end{align}
Note that we have $\abr{\widetilde{\Pi}_h} = \abr{\Pcal}\abr{\Fcal_{h+1}}$. Note that since $Q_h^\star \in \Qcal_h$, we must have $\pi_h^\star \in \widetilde{\Pi}_h$, i.e., the policy class is realizable. Now we expand the discriminator classes as follows. We set $\widetilde{\Fcal}_H = \Fcal_H$. We know that $\widetilde{\Fcal}_H$ is realizable, i.e., $V_{H}^\star\in \widetilde{\Fcal}_H$.  Given a realizable $\widetilde{\Fcal}_{h+1}$, we design $\widetilde{\Fcal}_h$ as follows:
\begin{align}
\label{eq:discriminators}
    \widetilde{\Fcal}_h\defeq & \Fcal_h \cup \notag\\
    &~\{ f: f(x) \defeq \EE_{a\sim \pi_x, x'\sim \hat{P}_{x,a}}f'(x')   \}_{\hat{P}\in\Pcal, \pi\in\widetilde{\Pi}_h, f'\in\widetilde{\Fcal}_{h+1}}
\end{align} Namely we explicitly expand $\Fcal_h$ by applying a potential Bellman operator (defined using a pair $(\hat{P},\pi)$) to a discriminator at $\widetilde{\Fcal}_{h+1}$. Via the induction assumption, since $V_{h+1}^\star\in\widetilde{\Fcal}_{h+1}$ and $P\in\Pcal, \pi_h^\star\in\widetilde{\Pi}_h$, we must have $\Gamma_h V_{h+1}^\star \in \widetilde{\Fcal}_h$ by construction. Hence $\widetilde{\Fcal}_h$ is realizable. We can also recursively compute the size of $\widetilde{\Fcal}_h$. Define $F   \defeq \max_h \abr{\Fcal_h}$. Starting with $\abr{\widetilde{\Fcal}_H} = F$, we can show that $\log\left(\abr{\widetilde{\Fcal}_h}\right) = O\left(\text{poly}(H)(\abr{\Pcal}F)\right)$ for $h \in [H]$. 

\emph{Model-based} \fail{} takes $\Pcal$ and $\{\Fcal_h\}_h$ as inputs, and performs the two procedures shown in~\pref{fig:model_based_fail}.
\begin{figure}[h]
\begin{enumerate} 
\item Construct $\{\widetilde{\Pi}_h\}_h$ and $\{\widetilde{\Fcal}_h\}_h$ as shown in~\pref{eq:policy_class} and~\pref{eq:discriminators};
\item Run \fail{} (\pref{alg:main_alg}) with $\{\widetilde{\Pi}_h\}$ and $\{\widetilde{\Fcal}_h\}$.
\end{enumerate}
\caption{Model-based \fail}
\label{fig:model_based_fail}
\vspace{-10pt}
\end{figure}

In~\pref{fig:model_based_fail}, note that model-based \fail{} does not require \emph{any real-world samples} in the first step. Real-world samples are only required in the second step where we call \fail. As we show that $\{\widetilde{\Pi}_h\}_h$ and $\{\widetilde{\Fcal}_h\}_h$ are realizable, and \nameibe{} is zero due to the construction of $\widetilde{\Fcal}_h$, we can simply invoke~\pref{thm:ft_main_theorem} here and reach the following corollary:

\begin{corollary}[Model-based FAIL sample Complexity] Given a pair $(\epsilon,\delta)$, $\Pcal$ and $\{\Fcal_h\}_h$ with $P\in\Pcal$, and $V_h^\star\in\Fcal_h$ for all $h\in [H]$, the algorithm shown in~\pref{fig:model_based_fail} can learn an $\epsilon$ near-optimal policy with number of samples and number of expert's samples both scale in the order of $\text{poly}\left(H, K, (1/\epsilon), \log(\abr{\Pcal}), \log(F), \log(1/\delta)\right)$, with probability at least $1-\delta$.
\end{corollary}

The above corollary shows statistically, from a model-based perspective, realizability alone is sufficient to achieve the polynomial sample complexity. However, the above corollary does not show if realizability is sufficient for model-free methods. Also, though it is statistically efficient, the computational complexity of model-based \fail{} is still questionable (the naive implementation in~\pref{fig:model_based_fail} has computational complexity $\Theta(F^H)$). Investigating sufficient and necessary conditions for achieving polynomial sample and polynomial computation complexity under \ilo{} setting is an interesting future work.  

\vspace{-5pt}
\section{Conclusion and Future Work}
We study Imitation Learning from Observation (\ilo) setting and propose an algorithm, Forward Adversarial Imitation Learning (\fail), that achieves sample efficiency and computational efficiency. \fail{} decomposes the sequential learning tasks into  independent two-player min-max games of which is solved via general no-regret online learning. 
In addition to the algorithmic contribution, we present the first exponential gap in terms of sample complexity between \ilo{} and RL, demonstrating the potential benefit from expert's observations.  A key observation is that one should explicitly regularize the class of discriminators to achieve sample efficiency and design discriminators  to decrease the inherent Bellman Error. Experimentally, while GAIL can be used to solve the \ilo{} problem by  removing action inputs to the discriminators, \fail{} works just as well in problems with dense reward. Our analysis of \fail{} provides the first strong theoretical guarantee for solving \ilo, and \fail{} significantly outperforms GAIL on sparse reward MDPs, which are common in practice.


\section*{Acknowledgement}
WS is supported in part by Office of Naval Research contract N000141512365. WS thanks Nan Jiang and Akshay Krishnamurthy for valuable discussions. We thank the first anonymous reviewer for carefully reviewing the proofs.

\nocite{langley00}

\bibliography{refs,reference}

\begin{thebibliography}{47}
\providecommand{\natexlab}[1]{#1}
\providecommand{\url}[1]{\texttt{#1}}
\expandafter\ifx\csname urlstyle\endcsname\relax
  \providecommand{\doi}[1]{doi: #1}\else
  \providecommand{\doi}{doi: \begingroup \urlstyle{rm}\Url}\fi

\bibitem[Agarwal et~al.(2014)Agarwal, Hsu, Kale, Langford, Li, and
  Schapire]{agarwal2014taming}
Agarwal, A., Hsu, D., Kale, S., Langford, J., Li, L., and Schapire, R.
\newblock Taming the monster: A fast and simple algorithm for contextual
  bandits.
\newblock In \emph{International Conference on Machine Learning}, pp.\
  1638--1646, 2014.

\bibitem[Arjovsky et~al.(2017)Arjovsky, Chintala, and
  Bottou]{arjovsky2017wasserstein}
Arjovsky, M., Chintala, S., and Bottou, L.
\newblock Wasserstein gan.
\newblock \emph{arXiv preprint arXiv:1701.07875}, 2017.

\bibitem[Arora et~al.(2017)Arora, Ge, Liang, Ma, and
  Zhang]{arora2017generalization}
Arora, S., Ge, R., Liang, Y., Ma, T., and Zhang, Y.
\newblock Generalization and equilibrium in generative adversarial nets (gans).
\newblock \emph{arXiv preprint arXiv:1703.00573}, 2017.

\bibitem[Bagnell et~al.(2004)Bagnell, Kakade, Schneider, and
  Ng]{bagnell2004policy}
Bagnell, J.~A., Kakade, S.~M., Schneider, J.~G., and Ng, A.~Y.
\newblock Policy search by dynamic programming.
\newblock In \emph{Advances in neural information processing systems}, pp.\
  831--838, 2004.

\bibitem[Beygelzimer et~al.(2005)Beygelzimer, Dani, Hayes, Langford, and
  Zadrozny]{beygelzimer2005error}
Beygelzimer, A., Dani, V., Hayes, T., Langford, J., and Zadrozny, B.
\newblock Error limiting reductions between classification tasks.
\newblock In \emph{Proceedings of the 22nd international conference on Machine
  learning}, pp.\  49--56. ACM, 2005.

\bibitem[Brockman et~al.(2016)Brockman, Cheung, Pettersson, Schneider,
  Schulman, Tang, and Zaremba]{brockman2016openai}
Brockman, G., Cheung, V., Pettersson, L., Schneider, J., Schulman, J., Tang,
  J., and Zaremba, W.
\newblock Openai gym.
\newblock \emph{arXiv preprint arXiv:1606.01540}, 2016.

\bibitem[Chang et~al.(2015{\natexlab{a}})Chang, He, Daum{\'e}~III, and
  Langford]{chang2015learning_dependency}
Chang, K.-W., He, H., Daum{\'e}~III, H., and Langford, J.
\newblock Learning to search for dependencies.
\newblock \emph{arXiv preprint arXiv:1503.05615}, 2015{\natexlab{a}}.

\bibitem[Chang et~al.(2015{\natexlab{b}})Chang, Krishnamurthy, Agarwal, Daume,
  and Langford]{chang2015learning}
Chang, K.-w., Krishnamurthy, A., Agarwal, A., Daume, H., and Langford, J.
\newblock Learning to search better than your teacher.
\newblock In \emph{ICML}, 2015{\natexlab{b}}.

\bibitem[Chen \& Jiang(2019)Chen and Jiang]{chen2019information}
Chen, J. and Jiang, N.
\newblock Information-theoretic considerations in batch reinforcement learning.
\newblock \emph{arXiv preprint arXiv:1905.00360}, 2019.

\bibitem[Dann \& Brunskill(2015)Dann and Brunskill]{dann2015sample}
Dann, C. and Brunskill, E.
\newblock Sample complexity of episodic fixed-horizon reinforcement learning.
\newblock In \emph{Advances in Neural Information Processing Systems}, pp.\
  2818--2826, 2015.

\bibitem[Daum{\'e}~III et~al.(2009)Daum{\'e}~III, Langford, and
  Marcu]{daume2009search}
Daum{\'e}~III, H., Langford, J., and Marcu, D.
\newblock Search-based structured prediction.
\newblock \emph{Machine learning}, 2009.

\bibitem[Edwards et~al.(2018)Edwards, Sahni, Schroeker, and
  Isbell]{edwards2018imitating}
Edwards, A.~D., Sahni, H., Schroeker, Y., and Isbell, C.~L.
\newblock Imitating latent policies from observation.
\newblock \emph{arXiv preprint arXiv:1805.07914}, 2018.

\bibitem[Fukumizu et~al.(2009)Fukumizu, Gretton, Lanckriet, Sch{\"o}lkopf, and
  Sriperumbudur]{fukumizu2009kernel}
Fukumizu, K., Gretton, A., Lanckriet, G.~R., Sch{\"o}lkopf, B., and
  Sriperumbudur, B.~K.
\newblock Kernel choice and classifiability for rkhs embeddings of probability
  distributions.
\newblock In \emph{Advances in neural information processing systems}, pp.\
  1750--1758, 2009.

\bibitem[Givan et~al.(2003)Givan, Dean, and Greig]{givan2003equivalence}
Givan, R., Dean, T., and Greig, M.
\newblock Equivalence notions and model minimization in markov decision
  processes.
\newblock \emph{Artificial Intelligence}, 147\penalty0 (1-2):\penalty0
  163--223, 2003.

\bibitem[Hester et~al.(2017)Hester, Vecerik, Pietquin, Lanctot, Schaul, Piot,
  Horgan, Quan, Sendonaris, Dulac-Arnold, et~al.]{hester2017deep}
Hester, T., Vecerik, M., Pietquin, O., Lanctot, M., Schaul, T., Piot, B.,
  Horgan, D., Quan, J., Sendonaris, A., Dulac-Arnold, G., et~al.
\newblock Deep q-learning from demonstrations.
\newblock \emph{arXiv preprint arXiv:1704.03732}, 2017.

\bibitem[Ho \& Ermon(2016)Ho and Ermon]{ho2016generative}
Ho, J. and Ermon, S.
\newblock Generative adversarial imitation learning.
\newblock In \emph{NIPS}, 2016.

\bibitem[Jiang et~al.(2016)Jiang, Krishnamurthy, Agarwal, Langford, and
  Schapire]{jiang2016contextual}
Jiang, N., Krishnamurthy, A., Agarwal, A., Langford, J., and Schapire, R.~E.
\newblock Contextual decision processes with low bellman rank are
  pac-learnable.
\newblock \emph{arXiv preprint arXiv:1610.09512}, 2016.

\bibitem[Kakade \& Langford(2002)Kakade and Langford]{kakade2002approximately}
Kakade, S. and Langford, J.
\newblock Approximately optimal approximate reinforcement learning.
\newblock In \emph{ICML}, 2002.

\bibitem[Kakade et~al.(2003)Kakade, Kearns, and
  Langford]{kakade2003exploration}
Kakade, S., Kearns, M.~J., and Langford, J.
\newblock Exploration in metric state spaces.
\newblock In \emph{Proceedings of the 20th International Conference on Machine
  Learning (ICML-03)}, pp.\  306--312, 2003.

\bibitem[Kingma \& Ba(2014)Kingma and Ba]{kingma2014adam}
Kingma, D.~P. and Ba, J.
\newblock Adam: A method for stochastic optimization.
\newblock \emph{arXiv preprint arXiv:1412.6980}, 2014.

\bibitem[Krishnamurthy et~al.(2016)Krishnamurthy, Agarwal, and
  Langford]{krishnamurthy2016pac}
Krishnamurthy, A., Agarwal, A., and Langford, J.
\newblock Pac reinforcement learning with rich observations.
\newblock In \emph{Advances in Neural Information Processing Systems}, pp.\
  1840--1848, 2016.

\bibitem[Lazaric et~al.(2016)Lazaric, Ghavamzadeh, and
  Munos]{lazaric2016analysis}
Lazaric, A., Ghavamzadeh, M., and Munos, R.
\newblock Analysis of classification-based policy iteration algorithms.
\newblock \emph{The Journal of Machine Learning Research}, 17\penalty0
  (1):\penalty0 583--612, 2016.

\bibitem[Liu et~al.(2018)Liu, Gupta, Abbeel, and Levine]{liu2018imitation}
Liu, Y., Gupta, A., Abbeel, P., and Levine, S.
\newblock Imitation from observation: Learning to imitate behaviors from raw
  video via context translation.
\newblock In \emph{2018 IEEE International Conference on Robotics and
  Automation (ICRA)}, pp.\  1118--1125. IEEE, 2018.

\bibitem[Luxburg \& Bousquet(2004)Luxburg and Bousquet]{luxburg2004distance}
Luxburg, U.~v. and Bousquet, O.
\newblock Distance-based classification with lipschitz functions.
\newblock \emph{Journal of Machine Learning Research}, 5\penalty0
  (Jun):\penalty0 669--695, 2004.

\bibitem[M{\"u}ller(1997)]{muller1997integral}
M{\"u}ller, A.
\newblock Integral probability metrics and their generating classes of
  functions.
\newblock \emph{Advances in Applied Probability}, 29\penalty0 (2):\penalty0
  429--443, 1997.

\bibitem[M\"{u}ller et~al.(1997)M\"{u}ller, Smola, and R\"{a}tsch]{Muller1997}
M\"{u}ller, K., Smola, A., and R\"{a}tsch, G.
\newblock {Predicting time series with support vector machines}.
\newblock \emph{Artificial Neural Networks ‚Äî ICANN'9}, 1327:\penalty0
  999--1004, 1997.
\newblock \doi{10.1007/BFb0020283}.
\newblock URL \url{http://link.springer.com/chapter/10.1007/BFb0020283}.

\bibitem[Munos(2005)]{munos2005error}
Munos, R.
\newblock Error bounds for approximate value iteration.
\newblock In \emph{Proceedings of the National Conference on Artificial
  Intelligence}, volume~20, pp.\  1006. Menlo Park, CA; Cambridge, MA; London;
  AAAI Press; MIT Press; 1999, 2005.

\bibitem[Munos \& Szepesv{\'a}ri(2008)Munos and
  Szepesv{\'a}ri]{munos2008finite}
Munos, R. and Szepesv{\'a}ri, C.
\newblock Finite-time bounds for fitted value iteration.
\newblock \emph{Journal of Machine Learning Research}, 9\penalty0
  (May):\penalty0 815--857, 2008.

\bibitem[Nair et~al.(2017)Nair, Chen, Agrawal, Isola, Abbeel, Malik, and
  Levine]{nair2017combining}
Nair, A., Chen, D., Agrawal, P., Isola, P., Abbeel, P., Malik, J., and Levine,
  S.
\newblock Combining self-supervised learning and imitation for vision-based
  rope manipulation.
\newblock In \emph{Robotics and Automation (ICRA), 2017 IEEE International
  Conference on}, pp.\  2146--2153. IEEE, 2017.

\bibitem[Pan et~al.(2018)Pan, Cheng, Saigol, Lee, Yan, Theodorou, and
  Boots]{pan2018agile}
Pan, Y., Cheng, C.-A., Saigol, K., Lee, K., Yan, X., Theodorou, E., and Boots,
  B.
\newblock Agile autonomous driving using end-to-end deep imitation learning.
\newblock \emph{Proceedings of Robotics: Science and Systems. Pittsburgh,
  Pennsylvania}, 2018.

\bibitem[Peng et~al.(2018)Peng, Abbeel, Levine, and van~de
  Panne]{peng2018deepmimic}
Peng, X.~B., Abbeel, P., Levine, S., and van~de Panne, M.
\newblock Deepmimic: Example-guided deep reinforcement learning of
  physics-based character skills.
\newblock \emph{arXiv preprint arXiv:1804.02717}, 2018.

\bibitem[Ross \& Bagnell(2010)Ross and Bagnell]{Ross2010}
Ross, S. and Bagnell, J.~A.
\newblock Efficient reductions for imitation learning.
\newblock In \emph{AISTATS}, pp.\  661--668, 2010.

\bibitem[Ross \& Bagnell(2014)Ross and Bagnell]{ross2014reinforcement}
Ross, S. and Bagnell, J.~A.
\newblock Reinforcement and imitation learning via interactive no-regret
  learning.
\newblock \emph{arXiv preprint arXiv:1406.5979}, 2014.

\bibitem[Ross et~al.(2011)Ross, Gordon, and Bagnell]{Ross2011_AISTATS}
Ross, S., Gordon, G.~J., and Bagnell, J.
\newblock A reduction of imitation learning and structured prediction to
  no-regret online learning.
\newblock In \emph{AISTATS}, 2011.

\bibitem[Schulman et~al.(2015)Schulman, Levine, Abbeel, Jordan, and
  Moritz]{schulman2015trust}
Schulman, J., Levine, S., Abbeel, P., Jordan, M.~I., and Moritz, P.
\newblock Trust region policy optimization.
\newblock In \emph{ICML}, pp.\  1889--1897, 2015.

\bibitem[Shalev-Shwartz et~al.(2012)]{shalev2012online}
Shalev-Shwartz, S. et~al.
\newblock Online learning and online convex optimization.
\newblock \emph{Foundations and Trends{\textregistered} in Machine Learning},
  2012.

\bibitem[Silver et~al.(2016)]{silver2016mastering}
Silver, D. et~al.
\newblock Mastering the game of go with deep neural networks and tree search.
\newblock \emph{Nature}, 2016.

\bibitem[Sriperumbudur et~al.(2012)Sriperumbudur, Fukumizu, Gretton,
  Sch{\"o}lkopf, Lanckriet, et~al.]{sriperumbudur2012empirical}
Sriperumbudur, B.~K., Fukumizu, K., Gretton, A., Sch{\"o}lkopf, B., Lanckriet,
  G.~R., et~al.
\newblock On the empirical estimation of integral probability metrics.
\newblock \emph{Electronic Journal of Statistics}, 6:\penalty0 1550--1599,
  2012.

\bibitem[Sun et~al.(2016)Sun, Venkatraman, Boots, and Bagnell]{sun2016learning}
Sun, W., Venkatraman, A., Boots, B., and Bagnell, J.~A.
\newblock Learning to filter with predictive state inference machines.
\newblock In \emph{ICML}, 2016.

\bibitem[Sun et~al.(2017)Sun, Venkatraman, Gordon, Boots, and
  Bagnell]{sun2017deeply}
Sun, W., Venkatraman, A., Gordon, G.~J., Boots, B., and Bagnell, J.~A.
\newblock Deeply aggrevated: Differentiable imitation learning for sequential
  prediction.
\newblock \emph{arXiv preprint arXiv:1703.01030}, 2017.

\bibitem[Sun et~al.(2018{\natexlab{a}})Sun, Bagnell, and
  Boots]{sun2018truncated}
Sun, W., Bagnell, J.~A., and Boots, B.
\newblock Truncated horizon policy search: Combining reinforcement learning \&
  imitation learning.
\newblock \emph{arXiv preprint arXiv:1805.11240}, 2018{\natexlab{a}}.

\bibitem[Sun et~al.(2018{\natexlab{b}})Sun, Jiang, Krishnamurthy, Agarwal, and
  Langford]{sun2018model}
Sun, W., Jiang, N., Krishnamurthy, A., Agarwal, A., and Langford, J.
\newblock Model-based reinforcement learning in contextual decision processes.
\newblock \emph{arXiv preprint arXiv:1811.08540}, 2018{\natexlab{b}}.

\bibitem[Sutton(1988)]{Sutton1998}
Sutton, R.
\newblock Learning to predict by the methods of temporal differences.
\newblock \emph{Machine Learning}, 3:\penalty0 9--44, 1988.

\bibitem[Syed \& Schapire(2010)Syed and Schapire]{syed2010reduction}
Syed, U. and Schapire, R.~E.
\newblock A reduction from apprenticeship learning to classification.
\newblock In \emph{Advances in neural information processing systems}, pp.\
  2253--2261, 2010.

\bibitem[Torabi et~al.(2018)Torabi, Warnell, and Stone]{torabi2018behavioral}
Torabi, F., Warnell, G., and Stone, P.
\newblock Behavioral cloning from observation.
\newblock \emph{arXiv preprint arXiv:1805.01954}, 2018.

\bibitem[Venkatraman et~al.(2015)Venkatraman, Hebert, and
  Bagnell]{venkatraman2015improving}
Venkatraman, A., Hebert, M., and Bagnell, J.~A.
\newblock Improving multi-step prediction of learned time series models.
\newblock \emph{AAAI}, 2015.

\bibitem[Zinkevich(2003)]{Zinkevich2003_ICML}
Zinkevich, M.
\newblock {Online Convex Programming and Generalized Infinitesimal Gradient
  Ascent}.
\newblock In \emph{ICML}, 2003.

\end{thebibliography}
\bibliographystyle{icml2019}

\vfill
\onecolumn
\newpage

\appendix

\newpage

\section{Proof of~\pref{thm:L2M_result}}
\label{app:proof_of_L2M_result}
Before proving the theorem, we introduce some notations and useful lemmas. 

Given a pair of $\pi$ and $f$, denote random variable $v_i = K\pi(a_h^i|x^i_h) f(x_{h+1}^i) - f(\tilde{x}_{h+1}^i)$, 
recall the definition of the utility function $u(\pi, f)$:
\begin{align*}
     u(\pi, f) = \frac{1}{N} \sum_{i=1}^N K\pi(a^i_h|x^i_h)f(x^i_{h+1}) - \frac{1}{N}\sum_{n=1}^N f(\tilde{x}_{h+1}^i) = \frac{1}{N}\sum_{n=1}^N v_i.
\end{align*} 
Denote $v = \mathbb{E}_{x\sim \nu_h, a\sim \pi(\cdot|x), x'\sim P_{x,a}}[f(x')] - \mathbb{E}_{x\sim \mu^\star_{h+1}}[f(x)]$. It is easy to verify that $\mathbb{E}_{i} v_i = v$. We also have $|v_i - v| \leq 4K$. We can further bound the variance of $v_i -v $ as:
\begin{align*}
    \text{Var}_i(v_i - v ) &= \mathbb{E}_{i}(v_i-v)^2 = \mathbb{E}_{i} v_i^2 - v^2 \leq \mathbb{E}_i v_i^2 \\
    & = \mathbb{E}_i (K\pi(a^i_h|x^i_h)f(x^i_{h+1}) - f(\tilde{x}_{h+1}^i))^2 \\
    & \leq \mathbb{E}_i K^2 \pi(a^i_h|x^i_h) f(x^i_{h+1}) - \mathbb{E}_i K\pi(a^i_h|x^i_h)f(x^i_{h+1})f(\tilde{x}^i_{h+1}) + \mathbb{E}_{i} f(\tilde{x}^i_{h+1})^2\\
    & \leq K + 1 + 1 \leq 2K,
\end{align*} where we used the fact that $|f(x)|\leq 1, \forall x$, $\pi(a|x) \leq 1, \forall x,a$, and the last inequality uses the fact that $a_h^i$ is sampled from a uniform distribution over $\Acal$. With that, we can apply Bernstein's inequality to $\{v_i\}$ together with a union bound over $\Pi$ and $\Fcal$, we will have the following lemma:
\begin{lemma}
\label{lem:concentration_pi_f}
Given dataset $\Dcal=\{x_h^i,a_h^i,x_{h+1}^i\}$ with $x_h^i\sim \nu_h, a_h^i\sim U(\Acal), x_{h+1}^i \sim P_{x^i_h,a^i_h}$, and $\Dcal^e = \{\tilde{x}_{h+1}^i\}$ with $\tilde{x}_{h+1}^i\sim \mu^\star_{h+1}$, for any pair $\pi\in\Pi, f\in\Fcal$, with probability at least $1-\delta$, 
\begin{align}
    &\left\lvert \left(\frac{1}{N}\sum_{i=1}^N K\pi(a^i_h|x^i_h)f(x^i_{h+1}) - \frac{1}{N}\sum_{i=1}^N f(\tilde{x}^i_{h+1})\right)  - \left(\mathbb{E}_{(x,a,x')\sim \nu_h\pi P^{\star}}[f(x')] - \mathbb{E}_{x\sim \mu^{\star}_{h+1}}[f(x)]\right) \right\rvert \nonumber\\
    & \leq 4\sqrt{\frac{2K\log(2|\Pi||\Fcal|/\delta)}{N}} + \frac{8K\log(2|\Pi||\Fcal|/\delta)}{N}.
\end{align}
\end{lemma}

Let us define two loss functions for $\pi$ and $f$:
\begin{align*}
    &\ell_t(\pi) = (1/N)\sum_{i=1}^N K\pi(a^i_h|x^i_h)f^t(x^i_{h+1}) 
    \\
    & c_t(f) = (1/N) \sum_{i=1}^N K\pi^t(a^i_h|x^i_h)f(x^i_{h+1}) - (1/N)\sum_{n=1}^N f(\tilde{x}_{h+1}^i).
\end{align*}

For any $f,g:\mathcal{X}\times\mathcal{A}\times\Xcal \to\mathbb{R}$, define $\langle f,g \rangle = \mathbb{E}_{(x,a)\sim \Dcal_{x,a}} f(x,a)g(x,a)$, where we overload the notation and denote $\Dcal$ as the empirical distribution over the dataset $\Dcal$ (i.e., put probability $1/|\Dcal|$ over each data point in $\Dcal$), and $\Dcal_{x,a}$ as the marginal distribution over $x,a$. With this notation, we can see that $\ell_t(\pi)$ can be written as a linear functional with respect to $\pi$:
\begin{align}
    \ell_t(\pi) = \langle \pi, K f^t \rangle,
\end{align} where $K f_t$ is defined such that $K f^t(x,a) = K\sum_{i=1}^N \bm{1}[x=x^i_h,a=a^i_h] f^t(x^i_{h+1})$. Under this definition of inner product, we have:
\begin{align*}
    \max_{\pi} \|\pi\| \leq 1, \;\;\; \max \|K f^t\| \leq K.
\end{align*}

It is easy to verify that Algorithm~\ref{alg:l2m} is running Best Response on loss $\{c_t(f)\}_t$ and running FTRL on loss $\{\ell_t(\pi)\}_t$. Using the no-regret guarantee from FTRL, for $\{\pi^t\}$, we have:
\begin{align}
    \frac{1}{T}\sum_{t=1}^T \ell_t(\pi^t) - \min_{\pi\in\Pi}\frac{1}{T}\sum_{t=1}^T \ell_t(\pi) \leq   \frac{K}{\sqrt{T}}.
\end{align}

Denote $\hat{\pi}^\star$ and $\hat{f}^{\star}$ as the minimizer and maximizer of Eqn~\ref{eq:emprical_IPM}, i.e., 
\begin{align}
\label{eq:min_max_optimal_hat}
    (\hat{\pi}^\star, \hat{f}^\star) = \arg\min_{\pi\in\Pi}\arg\max_{f\in\Fcal}  \left( \frac{1}{N}\sum_{i=1}^N {K\pi(a^i_h|x^i_h)}f(x^i_{h+1}) - \frac{1}{N}\sum_{i=1}^N f(\tilde{x}^i_{h+1})   \right).
\end{align}

The following lemma quantifies the performance of $\bar{\pi} = \sum_{t=1}^T \pi^t / T$ and $\bar{f} = \sum_{t=1}^T f^t/T$:
\begin{lemma}
\label{lem:no_regret_results}
Denote $\bar{\pi} = \sum_{t=1}^T \pi^t/T$ and $\bar{f}^{\star} = \max_{f\in\Fcal}\left({\EE}_{(x,a,x')\sim \nu_n \bar{\pi}P^\star}[f(x')] - \EE_{x\sim \mu^{\star}_{h+1}}[f(x)]\right)$. We have:
\begin{align*}
    \frac{1}{N} \sum_{i=1}^N K\bar{\pi}(a^i_h|x^i_h)\bar{f}^\star(x^i_{h+1}) - \frac{1}{N}\sum_{i=1}^T \bar{f}^\star(\tilde{x}^i_{h+1}) \leq \frac{1}{N} \sum_{i=1}^N K\hat{\pi}^{\star}(a^i_h|x^i_h)\hat{f}^\star(x^i_{h+1}) - \frac{1}{N}\sum_{i=1}^N \hat{f}^\star(\tilde{x}^i_{h+1}) + \frac{K}{\sqrt{T}},
\end{align*} where $\hat{\pi}^\star, \hat{f}^\star$ is defined in~\pref{eq:min_max_optimal_hat}.
\end{lemma}
\begin{proof}
Using the definition of $\ell_t$ and the no-regret property on $\{\pi_t\}$, we have:
\begin{align*}
    &\frac{1}{T}\sum_{t=1}^T \left( \frac{1}{N}\sum_{i=1}^N {K\pi^t(a^i_h|x^i_h)}f^t(x^i_{h+1}) - \frac{1}{N}\sum_{i=1}^N f^t(\tilde{x}^i_{h+1})   \right)\\
    &\leq \min_{\pi\in\Pi} \frac{1}{T}\sum_{t=1}^T \left( \frac{1}{N}\sum_{i=1}^N {K\pi(a^i_h|x^i_h)}f^t(x^i_{h+1}) - \frac{1}{N}\sum_{i=1}^N f^t(\tilde{x}^i_{h+1})   \right)  + \frac{K}{\sqrt{T}}.
\end{align*}Since $f^t = \arg\max_{f\in\Fcal} c_t(f)$, we have:
\begin{align*}
    &\frac{1}{N}\sum_{i=1}^N K\bar{\pi}(a^i_h|x^i_h)\bar{f}^\star(x^i_{h+1}) - \frac{1}{N}\sum_{i=1}^N \bar{f}^\star(\tilde{x}^i_{h+1})
    = \frac{1}{T}\sum_{t=1}^T \left( \frac{1}{N}\sum_{i=1}^N {K\pi^t(a_h^i|x_h^i)}\bar{f}^\star(x_{h+1}^i) - \frac{1}{N}\sum_{i=1}^N \bar{f}^\star(\tilde{x}_{h+1}^i)   \right) \\
    & \leq \frac{1}{T}\sum_{t=1}^T  \left( \frac{1}{N}\sum_{i=1}^N {K\pi^t(a_h^i|x_h^i)}f^t(x_{h+1}^i) - \frac{1}{N}\sum_{i=1}^N f^t(\tilde{x}_{h+1}^i)  \right)
\end{align*} We also have:
\begin{align*}
   &\min_{\pi\in\Pi} \frac{1}{T}\sum_{t=1}^T \left( \frac{1}{N}\sum_{i=1}^N {K\pi(a^i_h|x^i_h)}f^t(x^i_{h+1}) - \frac{1}{N}\sum_{i=1}^N f^t(\tilde{x}^i_{h+1})   \right)\\ 
   &\leq \frac{1}{T}\sum_{t=1}^T  \left( \frac{1}{N}\sum_{i=1}^N {K\hat{\pi}^\star(a_h^i|x_h^i)}f^t(x_{h+1}^i) - \frac{1}{N}\sum_{i=1}^N f^t(\tilde{x}^i_{h+1})   \right) \\
   & \leq \max_{f\in \{f^1, \dots, f^T\}} \left( \frac{1}{N}\sum_{i=1}^N {K\hat{\pi}^\star(a_h^i|x_h^i)}f(x_{h+1}^i) - \frac{1}{N}\sum_{i=1}^N f(\tilde{x}^i_{h+1})   \right) \\
   & \leq \left( \frac{1}{N}\sum_{i=1}^N {K\hat{\pi}^\star(a_h^i|x_h^i)}\hat{f}^\star(x_{h+1}^i) - \frac{1}{N}\sum_{i=1}^N \hat{f}^{\star}(\tilde{x}^i_{h+1})   \right),
\end{align*} where the first inequality uses the definition of $\min_{\pi\in\Pi}$, the second inequality uses the fact that the maximum is larger than the average, and the last inequality uses the fact that $\hat{f}^{\star}$ is the maximizer with respect to $\hat{\pi}^{\star}$.

Combining the above results, we have:
\begin{align*}
\frac{1}{N}\sum_{i=1}^N K\bar{\pi}(a^i_h|x^i_h)\bar{f}^\star(x^i_{h+1}) - \frac{1}{N}\sum_{i=1}^N \bar{f}^\star(\tilde{x}^i_{h+1}) \leq \frac{1}{N}\sum_{i=1}^N K\hat{\pi}^\star(a_h^i|x_h^i)\hat{f}^\star(x_{h+1}^i) - \frac{1}{N}\sum_{i=1}^N \hat{f}^\star(\tilde{x}^i_{h+1}) + \frac{K}{\sqrt{T}}.
\end{align*}

\end{proof}

Denote $\pi^\star$ and $f^\star$ as
\begin{align}
\label{eq:one_step_min_max}
    \pi^\star, f^{\star} = \arg\min_{\pi\in\Pi}\arg\max_{f\in \Fcal} \left(\mathbb{E}_{x\sim v, a\sim \pi, x'\sim P_{x,a}}[f(x')] - \mathbb{E}_{x\sim \mu^{\pi^\star}_{h+1}}[f(x)]\right)
\end{align}
Now we are ready to prove~\pref{thm:L2M_result}
\begin{proof}
[Proof of~\pref{thm:L2M_result}]
Denote $C_N = 4\sqrt{\frac{2K\log(2|\Pi||\Fcal|/\delta)}{N}} + \frac{8K\log(2|\Pi||\Fcal|/\delta)}{N} $.
First, using the concentration result from~\pref{lem:concentration_pi_f}, we have:
\begin{align*}
    &\left\lvert \left(\frac{1}{N} \sum_{i=1}^N K\bar{\pi}(a_h^i|x_h^i)\bar{f}^\star(x_{h+1}^i) - \frac{1}{N}\sum_{i=1}^T \bar{f}^\star(\tilde{x}^i_{h+1}) \right) - \left(\mathbb{E}_{(x,a,x')\sim \nu_h\bar{\pi}P^{\star}}[\bar{f}^\star(x')]  - \mathbb{E}_{x\sim \mu^{\star}_{h+1}}[\bar{f}^\star(x)] \right) \right\rvert \\
    & \leq \frac{1}{T}\sum_{t=1}^T \left\lvert \left( \frac{1}{N} \sum_{i=1}^N K{\pi}^t(a_h^i|x_h^i)\bar{f}^\star(x_{h+1}^i) - \frac{1}{N}\sum_{i=1}^T \bar{f}^\star(\tilde{x}^i_{h+1})\right) - \left({\EE}_{(x,a,x')\sim \nu_h{\pi^t}P^{\star}}[\bar{f}^\star(x')]  - {\EE}_{x\sim \mu^{\star}_{h+1}}[\bar{f}^\star(x)] \right) \right\rvert \\
    &\leq \frac{1}{T}\sum_{t=1}^T \left( C_N \right) = C_N.
\end{align*}
On the other hand, for ${\pi}^{\star}, f^\star$, we have:
\begin{align}
   &\left\lvert \left(\frac{1}{N} \sum_{i=1}^N K{\pi}^{\star}(a^i_h|x^i_h){f}^\star(x^i_{h+1}) - \frac{1}{N}\sum_{i=1}^N {f}^\star(\tilde{x}^i_{h+1})\right)  -
   \left( \mathbb{E}_{(x,a,x')\sim \nu_h{\pi}^\star P^{\star}}[{f}^\star(x')]  - {\EE}_{x\sim \mu^{\star}_{h+1}}[{f}^\star(x)]\right) 
   \right\rvert \\
   &\leq C_N.
\end{align} 

Define $\hat{f}' = \max_{f\in \Fcal}(\frac{1}{N} \sum_{i=1}^N K{\pi}^{\star}(a_h^i|x_h^i){f}(x_{h+1}^i) - \frac{1}{N}\sum_{i=1}^N {f}(\tilde{x}_{h+1}^i))$. Combine the above inequalities together, we have:
\begin{align*}
   & \max_{f\in\Fcal} \mathbb{E}_{(x,a,x')\sim \nu_h\bar{\pi}P^{\star}}[{f}(x')]  - {\EE}_{x\sim \mu^{\star}_{h+1}}[{f}(x)]
   =\mathbb{E}_{(x,a,x')\sim \nu_h\bar{\pi}P^{\star}}[\bar{f}^\star(x')]  - {\EE}_{x\sim \mu^{\star}_{h+1}}[\bar{f}^\star(x)] \\
   & \leq \frac{1}{N} \sum_{i=1}^N K\bar{\pi}(a_h^i|x_h^i)\bar{f}^\star(x_{h+1}^i) - \frac{1}{N}\sum_{i=1}^T \bar{f}^\star(\tilde{x}^i_{h+1}) + C_N \\
   & \leq \frac{1}{N} \sum_{i=1}^N K\hat{\pi}^{\star}(a_h^i|x_h^i)\hat{f}^\star(x_{h+1}^i) - \frac{1}{N}\sum_{i=1}^N \hat{f}^\star(\tilde{x}^i_{h+1}) + \frac{K}{\sqrt{T}} + C_N\\
   & \leq \frac{1}{N} \sum_{i=1}^N K{\pi}^{\star}(a_h^i|x_h^i)\hat{f}'(x_{h+1}^i) - \frac{1}{N}\sum_{n=1}^N \hat{f}'(\tilde{x}_{h+1}^i) + \frac{K}{\sqrt{T}} + C_N \\
   & \leq \mathbb{E}_{(x,a,x')\sim \nu_h {\pi}^\star P^{\star}}[\hat{f}'(x')]  - {\EE}_{x\sim \mu^{\star}_{h+1}}[\hat{f}'(x)] + 2C_N + \frac{K}{\sqrt{T}}\\
   & \leq \mathbb{E}_{(x,a,x')\sim \nu_h{\pi}^\star P^{\star}}[{f}^\star(x')]  - {\EE}_{x\sim \mu^{\star}_{h+1}}[{f}^\star (x)] + 2C_N+ \frac{K}{\sqrt{T}},
\end{align*} where the first equality uses the definition of $\bar{f}^\star$, the second inequality uses~\pref{lem:no_regret_results}, the third inequality uses the fact that $\hat{\pi}^\star$ and $\hat{f}^{\star}$ are the min-max solution of~\pref{eq:min_max_optimal_hat}, and the fifth inequality uses the fact that $f^{\star}$ is the maximizer of~\pref{eq:one_step_min_max} given $\pi^{\star}$. Hence, we prove the theorem. 

\end{proof}

\section{Proof of~\pref{thm:non_generalize}}

\begin{lemma}
\label{lem:non_generalize_dist}
There exists a distribution $D\in \Delta(\Xcal)$, such that for any two datasets $S_1 = \{x_1,\dots,x_M\}$ and $S_2 = \{x'_1,\dots,x'_M\}$ where $x_i$ and $x'_i$ are drawn i.i.d from $D$, as long as $M = O(\log(|\Xcal|))$, then:
\begin{align*}
    \lim_{|\Xcal|\to \infty} \mathrm{Pr}\left([S_1\cap S_2 = \emptyset]\right) = 1.
\end{align*}
\end{lemma} 
\begin{proof}
We simply set $D$ to be a uniform distribution of $\Xcal$. Denote $|\Xcal| = N$, and $M = O(\log (N))$. The probability of $S_1$ and $S_2$ does not have any overlap samples can be easily computed as:
\begin{align*}
    \mathrm{P}(S_1 \cap S_2 = \emptyset) \geq \mathrm{P}(S_1 \cap S_2 = \emptyset \text{ and } S_1 \text{ does not have repeated samples}).
\end{align*} Note that the probability that $S_1$ does not have repeated samples can be computed as:
\begin{align*}
    \mathrm{P}(S_1 \text{ does not have repeated samples}) = (1- 1/N)(1-2/N)(1 - (M-1)/N).
\end{align*} When $N\to \infty$ and $M = O(\log N)$, we have:
\begin{align*}
    \lim_{N\to\infty} \mathrm{P}(S_1 \text{ does not have repeated samples}) = 1.
\end{align*}
Now, conditioned on the event that $S_1$ does not contain repeated samples, we have:
\begin{align*}
    \mathrm{P}(S_1\cap S_2 = \emptyset) = (1-M/N)^M = (1 - M/N)^{(N/M)*(M^2/N)}
\end{align*}
Take $N\to\infty$, we know that $\lim_{x\to\infty}(1-1/x)^x = 1/e$ and $\lim_{N\to\infty} M^2/N = 0$, hence we have:
\begin{align*}
    \lim_{N\to\infty} (1-M/N)^K = \lim_{N\to\infty}(1-M/N)^{(N/M)(M^2/N)} = \lim_{N\to\infty} (1/e)^{M^2/N} = 1.
\end{align*}Hence we prove the lemma by coming two results above. 
\end{proof} 

We construct the MDP using the above lemma. The MDP has $H=2$, two actions $\{a, a^{\star}\}$, and the initial distribution $\rho$ assigns probability one to a unique state $\hat{x} \in \Xcal$. The expert policy $\pi^\star$ is designed to be $\pi^\star(a^\star|\hat{x}) = 1$, i.e., the expert's action at time step $h=1$ is $a^{\star}$. We split the state space $\Xcal$ into half and half, denoted as $\Xcal_1$ and $\Xcal_2$, such that $\Xcal_1 \cap \Xcal_2 = \emptyset$ and $|\Xcal_1| = |\Xcal_2| = N/2$.  We design the MDP's dynamics such that $P(\cdot|\hat{x}, a)$ assigns probability $2/N$ to each state in $\Xcal_1$ and assigns probability $0$ to any other state in $\Xcal_2$. We design $P(\cdot|\hat{x},a^\star)$ such that it assigns probability $2/N$ to each state in $\Xcal_2$ and zero to each state in $\Xcal_1$. 

Denote $\Dcal^\star = \{\tilde{x}^{(i)}_2\}_{i=1}^M$ as the states generated from the expert by executing $a^\star$ at $\hat{x}$. For any two policies $\pi$ and $\pi'$, such that $\pi(a^\star|\hat{x}) =1$ and $\pi'(a|\hat{x}) = 1$, denote $\Dcal = \{x^{i}_2\}_{i=1}^M$ as the dataset sampled from executing $\pi$ at $\hat{x}$ $M$ many times, and $\Dcal' = \{x'^{(i)}_2\}_{i=1}^M$ as the dataset sampled from executing $\pi'$ at $\hat{x}$ $M$ many times. From~\pref{lem:non_generalize_dist}, we know that 
\begin{align*}
    \lim_{N\to\infty} \mathrm{P}(\Dcal \cap \Dcal^{\star} = \emptyset) = 1, \;\;\;\; \lim_{N\to\infty}\mathrm{P}(\Dcal^* \cap \Dcal' = \emptyset) = 1.
\end{align*} Hence asymptotically either $\Dcal$ nor $\Dcal'$ will overlap with $\Dcal^\star$, unless $M = \Omega(\mathrm{poly}(N)$) = $\Omega(\mathrm{poly}(|\Xcal|))$.

\section{Proof of~\pref{thm:ft_main_theorem}}
\label{app:proof_main_theorem}

We first present some extra notations and useful lemmas below. 
\begin{lemma}
[Performance Difference Lemma \citep{kakade2002approximately}]
Consider a policy $\bm\pi = \{\pi_1,\dots,\pi_H\}$ and $\bm\pi^\star = \{\pi^\star_1, \dots, \pi^\star_H\}$. We have:
\begin{align*}
    J(\bm\pi) - J(\bm\pi^\star) = \sum_{h=1}^H \mathbb{E}_{x\sim \mu^{\bm\pi}_h}\left[\mathbb{E}_{a\sim \pi_h(\cdot|x)}Q^{\star}_h(x,a) - V^{\star}_h(x)\right].
\end{align*}
\label{lem:PDL}
\end{lemma}

Note that under our setting, i.e., the cost function does not depend on actions, the above equation can be simplified to:
\begin{align}
\label{eq:PDL}
    J(\bm\pi) - J(\bm\pi^\star) & = \sum_{h=1}^H \mathbb{E}_{x\sim \mu^{\bm\pi}_h}\left[\mathbb{E}_{a\sim \pi_h(\cdot|x)}Q^{\star}_h(x,a) - V^{\star}_h(x)\right] \\
    & = \sum_{h=1}^H \mathbb{E}_{x\sim\mu^{\bm\pi}_h}\left[\mathbb{E}_{a\sim \pi_h, x'\sim P_{x,a}}\left[V^{\star}_{h+1}(x')\right] - \mathbb{E}_{a\sim \pi^\star_h, x'\sim P_{x,a}}\left[V^{\star}_{h+1}(x')\right]\right],
\end{align} where we use Bellman equations, i.e., $Q^{\star}_h(x,a) = c(x) + \mathbb{E}_{x'\sim P_{x,a}} V^{\star}_{h+1}(x')$ and $V^{\star}_{h}(x) = c(x) + \mathbb{E}_{a\sim \pi^\star_h, x'\sim P_{x,a}}V^{\star}_{h+1}(x') $.\footnote{Note that here we actually prove the theorem under a more general setting where we could have cost functions at any time step $h$.} 

Note that for any $h$, we have:
\begin{align}
    &\left\lvert\mathbb{E}_{x\sim\mu^{\bm\pi}_h}\left[\mathbb{E}_{a\sim \pi_h, x'\sim P_{x,a}}\left[V^{\star}_{h+1}(x')\right] - \mathbb{E}_{a\sim \pi^\star_h, x'\sim P_{x,a}}\left[V^{\star}_{h+1}(x')\right]\right]\right\rvert \nonumber\\
    & \leq \left\lvert \mathbb{E}_{x\sim\mu^{\bm\pi}_h}\left[\mathbb{E}_{a\sim \pi_h, x'\sim P_{x,a}}\left[V^{\star}_{h+1}(x')\right]\right]  - {\EE}_{x\sim \mu_h^{\star}} \mathbb{E}_{a\sim \pi^\star_h, x'\sim P_{x,a}}\left[V^{\star}_{h+1}(x')\right]  \right\rvert\nonumber \\
    &\;\;\;\; + \left\lvert
    \mathbb{E}_{x\sim \mu_h^{\star}} \mathbb{E}_{a\sim \pi^\star_h, x'\sim P_{x,a}}\left[V^{\star}_{h+1}(x')\right]
    - \mathbb{E}_{x\sim \mu_h^{\bm\pi}} \mathbb{E}_{a\sim \pi^\star_h, x'\sim P_{x,a}}\left[V^{\star}_{h+1}(x')\right] 
    \right\rvert \nonumber\\
    & \leq d_{\Fcal_{h+1}}(\pi_h \vert \mu_h^{\bm\pi}, \mu_{h+1}^\star) +  \left\lvert
    \mathbb{E}_{x\sim \mu_h^{\star}} \mathbb{E}_{a\sim \pi^\star_h, x'\sim P_{x,a}}\left[V^{\star}_{h+1}(x')\right] - \mathbb{E}_{x\sim \mu_h^{\bm\pi}} \mathbb{E}_{a\sim \pi^\star_h, x'\sim P_{x,a}}\left[V^{\star}_{h+1}(x')\right] \right\rvert \nonumber\\
    & \leq  d_{\Fcal_{h+1}}(\pi_h \vert \mu_h^{\bm\pi}, \mu_{h+1}^\star) + \Delta_h + 2\epsilon_{\mathrm{be}}, 
    \label{eq:link_max_to_j}
\end{align}where the first inequality comes from the triangle inequality, the second inequality comes from the fact that $V_{h+1}^\star\in \Fcal_{h+1}$, and in the third inequality,  we denote $\Delta_h = \max_{f\in\Fcal}\left\lvert  \mathbb{E}_{x\sim \mu_h^{\star}}[f(x)] - \mathbb{E}_{x\sim \mu^{\bm\pi}_h}[f(x)]\right\rvert$, and $\epsilon_{\mathrm{be}}$ is introduced because $\Gamma_h V_{h+1}^\star$ might not in $\Fcal_{h}$. 

Now we are ready to prove the main theorem. 
\begin{proof}
[Proof of~\pref{thm:ft_main_theorem}]
We consider the $h$'th iteration. Let us denote $\bm\pi = \{\pi_1,\dots, \pi_{h-1}\}$ and $\mu^{\bm\pi}_{h}$ as the observation distribution at time step $h$ of following policies $\bm\pi$ starting from the initial distribution $\rho$. Denote $\mu^{\star}_{h+1}$ as the observation distribution of the expert policy $\bm\pi^\star$ at time step $h+1$ starting from the initial distribution $\rho$. Note that the dataset $\{\tilde{x}^{(i)}_{h+1}\}_{i=1}^n$ is generated from distribution $\mu^{\star}_{h+1}$. The data at $\Dcal$ is generated i.i.d by first drawing sample $x^{(i)}_{h}$ from $\mu^{\bm\pi}_h$ (i.e., executing $\pi_1,\dots \pi_{h-1}$), and then sample action $a_{h}^{(i)}\sim U(\Acal)$, and then sample $x_{h+1}^{(i)}$ from the real system $P_{x^{(i)}_{h},a^{(i)}_h}$. 

Mapping the above setup to the setup in~\pref{thm:L2M_result}, i.e., set 
\begin{align*}
\nu_h = \mu^{\bm\pi}_h, \;\; T = \Theta(4K^2/\epsilon^2), \;\;  N = \Theta(K\log(|\Pi||\Fcal|/\delta)/\epsilon^2),
\end{align*}
\pref{alg:l2m} will output a policy $\pi_{h}$ such that with probability at least $1-\delta$:
\begin{align*}
    d_{\Fcal_{h+1}}(\pi_h \vert \mu_h^{\bm{\pi}}, \mu_{h+1}^\star) \leq \min_{\pi\in\Pi_h}  d_{\Fcal_{h+1}}(\pi \vert \mu_h^{\bm{\pi}}, \mu_{h+1}^\star) + \epsilon. 
\end{align*} 

 Recall the definition of refined inherent Bellman Error $\epsilon'_{\mathrm{be}}$ with respect to $\Fcal_h$ and $\bm\pi^\star$:
\begin{align*}
    \epsilon'_{\mathrm{be}, h} =  \max_{g\in \Fcal_{h+1}}\min_{f\in\Fcal_h} \| f - \Gamma_h g  \|_{(\mu^{\bm\pi}_h + \mu_h^\star) /2}.
\end{align*}


Denote $\hat{f}$ as:
\begin{align*}
    \hat{f} = \arg\max_{f\in\Fcal_{h+1}} \left(
        \mathbb{E}_{x\sim \mu_h^{\bm\pi}}\left[ \mathbb{E}_{a\sim \pi^{\star}_h(\cdot|x),x'\sim P_{x,a}}[f(x')]  \right] - \mathbb{E}_{x\sim \mu_h^{{\star}}}\left[\mathbb{E}_{a\sim\pi^\star_{h}(\cdot|x),x'\sim P_{x,a}}[f(x')]\right]
    \right),
\end{align*} and $\hat{g}$ as:
\begin{align*}
   \hat{g} = \arg\min_{g\in\Fcal_h} \| g - \Gamma_h \hat{f} \|_{(\mu^{\bm\pi}_h + \mu_h^\star)/2}
\end{align*}


Now we upper bound $\min_{\pi\in\Pi_h}  d_{\Fcal_{h+1}}(\pi \vert \mu_h^{\bm{\pi}}, \mu_{h+1}^\star) $ as follows.
\begin{align*}
    &\min_{\pi\in\Pi_h}  d_{\Fcal_{h+1}}(\pi \vert \mu_h^{\bm{\pi}}, \mu_{h+1}^\star)  \leq d_{\Fcal_{h+1}}(\pi_h^\star \vert \mu_h^{\bm{\pi}}, \mu_{h+1}^\star) \\
    & = \max_{f\in\Fcal_{h+1}} \left\lvert \mathbb{E}_{x\sim \mu^{\bm\pi}_h,a\sim \pi^\star_h,x'\sim P_{x,a}}f(x') - \mathbb{E}_{x\sim \mu_h^{\star},a\sim \pi^\star_h, x'\sim P_{x,a}}f(x')     \right\rvert \\
    & = \max_{f\in\Fcal_{h+1}} \left\lvert
        \mathbb{E}_{x\sim \mu_h^{\bm\pi}}\left[ \mathbb{E}_{a\sim \pi^{\star}_h(\cdot|x),x'\sim P_{x,a}}[f(x')]  \right] - \mathbb{E}_{x\sim \mu_h^{{\star}}}\left[\mathbb{E}_{a\sim\pi^*_{h}(\cdot|x),x'\sim P_{x,a}}[f(x')]\right]
    \right\rvert\\
    & = \left\lvert
        \mathbb{E}_{x\sim \mu_h^{\bm\pi}}\left[ \mathbb{E}_{a\sim \pi^{\star}_h(\cdot|x),x'\sim P_{x,a}}[\hat{f}(x')]  \right] - \mathbb{E}_{x\sim \mu_h^{{\star}}}\left[\mathbb{E}_{a\sim\pi^\star_{h}(\cdot|x),x'\sim P_{x,a}}[\hat{f}(x')]\right]
    \right\rvert\\
    & \leq \left\lvert \mathbb{E}_{x\sim\mu^{\bm\pi}_h}[\hat{g}(x)] -  \mathbb{E}_{x\sim\mu^{\star}_h}[\hat{g}(x)]   \right\rvert + \left\lvert \mathbb{E}_{x\sim \mu^{\bm\pi}_h}[\hat{g}(x) - \mathbb{E}_{a\sim \pi^\star_h,x'\sim P_{x,a}}\hat{f}(x')]\right\rvert + \left\lvert \mathbb{E}_{x\sim \mu^{\star}_h}[\hat{g}(x) - \mathbb{E}_{a\sim \pi^\star_h,x'\sim P_{x,a}}\hat{f}(x')]\right\rvert
    \\
    & \leq \max_{f\in\Fcal_h} \left\lvert \mathbb{E}_{x\sim \mu^{\bm\pi}_h}[f(x)] - \mathbb{E}_{x\sim\mu^{\star}_h}[f(x)]    \right\rvert + 2\mathbb{E}_{x\sim (\mu_h^{\bm{\pi}} + \mu_h^\star)/2}\left[\lvert \hat{g}(x) - \EE_{a\sim\pi^\star_h, x'\sim P_{x,a}}\hat{f}(x') \rvert\right]\\
    & \leq \max_{f\in\Fcal_h} \left\lvert \mathbb{E}_{x\sim \mu^{\bm\pi}_h}[f(x)] - \mathbb{E}_{x\sim\mu^{\star}_h}[f(x)]    \right\rvert + 2\epsilon'_{\mathrm{be}}\\
    & = \Delta_h + 2\epsilon'_{\mathrm{be}},
\end{align*} where the first inequality comes from the realizable assumption that $\pi^\star_h\in\Pi$, the second inequality comes from an application of triangle inequality, and the third inequality comes from the definition of $\epsilon_{\mathrm{be}}$ and the fact that $\hat{g}\in\Fcal_h$. 

After learn $\pi_h$, $\bm\pi$ is updated to $\bm\pi = \{\pi_1,\dots, \pi_h\}$.  For $\Delta_{h+1}$, we have:
\begin{align*}
    &\Delta_{h+1} = \max_{f\in\Fcal_{h+1}}\left\lvert {\EE}_{x\sim \mu^{\bm\pi}_{h+1}}[f(x)] - {\EE}_{x\sim \mu^{\star}_{h+1}}[f(x)]  \right\rvert \\
    & = \max_{f\in\Fcal_{h+1}} \left\lvert  \mathbb{E}_{x\sim \mu^{\bm\pi}_h,a\sim\pi_h, x'\sim P_{x,a}}[f(x')] - \mathbb{E}_{x\sim \mu^{\star}_h,a\sim \pi^\star_h, x'\sim P_{x,a}}[f(x')]     \right\lvert \\
    & = d_{\Fcal_{h+1}}(\pi_h \vert \mu_h^{\bm\pi}, \mu_{h+1}^\star)\leq \min_{\pi\in\Pi_h}  d_{\Fcal_{h+1}}(\pi \vert \mu_h^{\bm\pi}, \mu_{h+1}^\star) + O(\epsilon) \leq 
    \Delta_h + 2\epsilon'_{\mathrm{be}} + O(\epsilon).
\end{align*}
Define $\Delta_ 0 = \max_{f}\left\lvert \mathbb{E}_{x\sim \rho}[f(x)] - \mathbb{E}_{x\sim \rho}[f(x)]  \right\rvert = 0$, we have for any $h$, 
\begin{align*}
    \Delta_h \leq 2h\epsilon'      _{\mathrm{be}} + O(h\epsilon).
\end{align*}

Now we link $\Delta_h$ to the performance of the policy $J(\bm\pi)$. From~\pref{eq:link_max_to_j}, we know that:
\begin{align*}
    &\left\lvert\mathbb{E}_{x\sim\mu^{\bm\pi}_h}\left[\mathbb{E}_{a\sim \pi_h, x'\sim P_{x,a}}\left[V^{\star}_{h+1}(x')\right] - \mathbb{E}_{a\sim \pi^\star_h, x'\sim P_{x,a}}\left[V^{\star}_{h+1}(x')\right]\right]\right\rvert \\
    & \leq d_{\Fcal_{h+1}}(\pi_h \vert \mu_h^{\bm\pi}, \mu_{h+1}^\star) + \Delta_h  + 2\epsilon_{\mathrm{be}}\\
    & \leq \Delta_h + O(\epsilon) + 2\epsilon'_{\mathrm{be}} + \Delta_h + 2\epsilon'_{\mathrm{be}} = 2\Delta_h + 4\epsilon'_{\mathrm{be}} + O(\epsilon)\\
    & \leq 4h\epsilon'_{\mathrm{be}} + O(2h\epsilon)   + 4\epsilon'_{\mathrm{be}} + O(\epsilon) = O(h\epsilon'_{\mathrm{be}}) + O(h\epsilon).
\end{align*}
Using Performance Difference Lemma (\pref{lem:PDL}), we know that:
\begin{align*}
    J(\bm\pi) - J(\bm\pi^\star) &\leq \sum_{h=1}^H \left\lvert  
    \mathbb{E}_{x\sim\mu^{\bm\pi}_h}\left[\mathbb{E}_{a\sim \pi_h, x'\sim P_{x,a}}\left[V^{\star}_{h+1}(x')\right] - \mathbb{E}_{a\sim \pi^\star_h, x'\sim P_{x,a}}\left[V^{\star}_{h+1}(x')\right]\right]\right\rvert\\
    & \leq \sum_{h=1}^H 4h\epsilon'_{\mathrm{be}} + 4\epsilon'_{\mathrm{be}} + O(2h\epsilon) + O(\epsilon) \\
    & \leq 4H^2\epsilon'_{\mathrm{be}} + 2H\epsilon'_{\mathrm{be}} + O(2H^2\epsilon) + O(H\epsilon) = O(H^2\epsilon'_{\mathrm{be}}) + O(H^2\epsilon)
\end{align*}

\end{proof}

\section{Proof of~\pref{prop:seperate_IL_and_RL}}

We first show the construction of the MDP below. The MDP has horizon $H$, $2^H-1$ many states, and two actions $\{l, r\}$ standing for \emph{go left} and \emph{go right} respectively. All states are organized in a perfect balanced binary tree, with $2^{H-1}$ many leafs at level $h = H$, and the first level $h=1$ contains only a root. The transition is deterministic such that at any internal state, taking action $l$ leads to the state's left child, and taking action $r$ leads to the state's right child. Each internal node has cost zero, and all leafs will have nonzero cost which we will specify later. Note that in such MDP, any sequence of actions $\{a_1,\dots, a_{H-1}\}$ with $a_i\in \{l,r\}$ deterministically leads to one and only one leaf, and the total cost of the sequence of actions is only revealed once the leaf is reached.   

The first part of the proposition is proved by reducing the problem to Best-arm identification in multi-armed bandit (MAB) setting. We use the following lower bound of best-arm identification in MAB from \cite{krishnamurthy2016pac}:
\begin{lemma}
[Lower bound for best arm identification in stochastic bandits from \cite{krishnamurthy2016pac}]
\label{lem:best_arm_mab}
For any $K\geq 2$ and $\epsilon \in (0, \sqrt{1/8}]$, and any best-arm identification algorithm, there exists a multi-armed bandit problem for which the best arm $i^\star$ is $\epsilon$ better than all others, but for which the estimate $\hat{i}$ of the best arm must have $\mathrm{P}(\hat{i}\neq i^\star) \geq 1/3$ unless the number of samples collected $T$ is at least $K/(72\epsilon^2)$.
\end{lemma}

Given any MAB problem with $K$ arms, without loss of generality, let us assume $K = 2^H - 1$ for some $H\in\mathbb{N}^+$. Any such MAB problem can be reduced to the above constructed binary tree MDP with horizon $H$, and $2^{H}-1$ leafs. Each arm in the original MAB will be encoded by a unique sequence of actions $\{a_h\}_{h=1}^{H-1}$ with $a_h\in\{l,r\}$, and its corresponding leaf. We assign each leaf the cost distribution of the corresponding arm. The optimal policy in the MDP, i.e., the sequence of actions leading to the leaf that has the smallest expected cost, is one-to-one corresponding to the best arm, i.e., the arm that has the smallest expected cost in the MAB. Hence, without any further information about the MDP, any RL algorithm that aims to find the near-optimal policy must suffer the lower bound presented in~\pref{lem:best_arm_mab}, as otherwise one can solve the original MAB  by first converting the MAB to the MDP and then running an RL algorithm. Hence, we prove the first part of the proposition. 

For the second part, let us denote the sequence of the observations from the expert policy as $\{\tilde{x}_{h}\}_{h=1}^H$, i.e., the sequence of states corresponding to the optimal sequence of actions where the last state $\tilde{x}_{H}$ has the smallest expected cost. We design an IL algorithm as follows. 

We initialize a sequence of actions $\bm{a} = \emptyset$. At every level $h$, staring at $h = 1$, we try any sequence of actions with prefix $\bm{a}\circ l$ ($\bm a\circ a$ means we append action $a$ to end of the sequence $\bm{a}$), record the observed observation $x^l_{h+1}$; we then reset and try any sequence of actions with prefix $\bm{a}\circ r$, and record the observed observation $x^l_{h+1}$. If $x^l_{h+1} = \tilde{x}_{h+1}$, then we append $l$ to $\bm{a}$, i.e., $\bm{a} = \bm{a}\circ l$, otherwise, we append $r$, i.e., $\bm{a} = \bm{a}\circ r$. We continue the above procedure until $h = H-1$, and we output the final action sequence $\bm{a}$.

Due to the deterministic transition, by induction, it is easy to verify that the outputted sequence of actions $\bm{a}$ is exactly equal to the optimal sequence of actions executed by the expert policy. Note that in each level $h$, we only generate two trajectories from the MDP. Hence the total number trajectories before finding the optimal sequence of actions is at most $2(H-1)$. Hence we prove the proposition.

\section{Reduction to LP}
\label{app:proof_of_claim}

Let us denote a set $\{y_i\}_{i=1}^{2N}$ such that $\{y_1, \dots, y_N\} = \{x_1,\dots,x_N\}$, and $\{y_{N+1}\dots, y_{2N}\} = \{x'_1,\dots, x'_N\}$. Denote $d_{i,j} = \Dcal(y_i, y_j)$ for $i\neq j$, and $c_i = 1/N$ for $i\in [N]$ and $c_i = -1/N$ for $i \in [N+1, 2N]$. We formulate the following LP with $2N$ variables and $O(N^2)$ many constraints:
\begin{align}
\label{eq:LP_formulation}
&\max_{\alpha_1,\dots,\alpha_{2N}} \sum_{i=1}^{2N} c_i\alpha_{i}, \;\;\;\;  s.t.,  \forall i\neq j,  -L d_{i,j} \leq   \alpha_i - \alpha_j \leq L d_{i,j},  \;\;\; \forall i, -1 \leq \alpha_i \leq 1. 
\end{align} Denote the solution of the above LP as $\alpha_{i}^\star$. We will have the following claim:
\begin{claim}[LP Oracle]
Given $\Fcal$ in~\pref{eq:special_F},  $\{x_i\}_{i=1}^N$, and $\{x_i'\}_{i=1}^N$, denote $\{\alpha_{i}^\star\}_{i=1}^{2N}$  as the solution of the LP from~\pref{eq:LP_formulation}, we have:
$\sup_{f\in\Fcal} \left(\sum_{i=1}^N f(x_i)/N - \sum_{i=1}^N f(x_i')/N\right) = \sum_{i=1}^{2N}c_i\alpha_{i}^\star$. 
\label{claim:equl_lp}
\end{claim}

\begin{proof}
[Proof of~\pref{claim:equl_lp}]
Given the solutions $\{\alpha_{i}^{\star}\}_{i=1}^{2N}$, we first are going to construct a function $\hat{f}:\Xcal\to\mathbb{R}$, such that for any $y_i$, we have $\hat{f}(y_i) = \alpha_i^\star$, and $\hat{f}\in\Fcal$. 
Denote $L^\star = \max_{i\neq j} |\alpha_i^\star - \alpha_j^\star|/d_{i,j}$. Note that $L^\star\leq L$.
The function $\hat{f}$ is constructed as:
\begin{align*}
    \hat{f}(x) = \max\left(-1, \min\left( 1, \min_{i\in[2N]}L^\star \Dcal(y_i, x) + \alpha_i^\star \right)\right)
\end{align*} First of all, we show that for any $y_i$, we have $\hat{f}(y_i) = \alpha_i^\star$. For any $j\neq i$, we have:
\begin{align*}
    L^{\star}\Dcal(y_j, y_i) + \alpha_j^\star \geq |\alpha_j^\star - \alpha_i^\star| + \alpha_j^\star \geq \alpha_i^\star,
\end{align*} where the first inequality uses the definition of $L^\star$. Also we know that $-1\leq \alpha_i^\star\leq 1$. Hence we have that for $y_i$, $\max(-1, \min(1, \min_{j\in[2N]} L^\star \Dcal(y_j, y_i) + \alpha_j^\star)) = \alpha_i^\star$.  

Now we need to prove that $\hat{f}$ is $L$-Lipschitz continuous. Note that we just need to prove that $\min_{i} L^\star\Dcal(y_i, x) + \alpha_i^\star$ is $L$-Lipschitz continuous, since for any $L$-Lipschitz continuous function $f(x)$, we have $\max(1, f(x))$ and $\min(-1, f(x))$ to be $L$-Lipschitz continuous as well. 

Consider any two points $x$ and $x'$ such that $x\neq x'$. Denote $\hat{i}$ as $\arg\min_{i} L^\star\Dcal(y_i, x) + \alpha_i^\star$ and $\hat{i}' = \arg\min_{i} L^\star\Dcal(y_i, x') + \alpha_i^\star$.
We have:
\begin{align*}
    &\hat{f}(x) - \hat{f}(x') =  L^\star\Dcal(y_{\hat{i}}, x) + \alpha_{\hat{i}}^\star - ( L^\star\Dcal(y_{\hat{i}'} , x') + \alpha_{\hat{i}'}^\star) \\
    & \leq L^\star\Dcal(y_{\hat{i}'}, x) + \alpha_{\hat{i}'}^\star - ( L^\star\Dcal(y_{\hat{i}'} , x') + \alpha_{\hat{i}'}^\star) \\
    & \leq L^{\star}\Dcal(x, x'),
\end{align*} where for the first inequality we used the definition of $\hat{i}$, and the second inequality uses the triangle inequality. Similarly, one can show that
\begin{align*}
    \hat{f}(x) - \hat{f}(x')\geq -L^\star\Dcal(x, x').
\end{align*} Combine the above two inequalities and the fact that $L^\star \leq L$, we conclude that $\hat{f}$ is $L$-Lipschitiz continuous.

Now we have constructed $\hat{f}$ such that $\hat{f}(y_i) = \alpha_i^\star$ for all $i\in[2N]$ and $\hat{f}\in\Fcal$. Now suppose that there exists a function $f'\in\Fcal$, such that $|\sum_{i=1}^{N}f'(x_i)/N - \sum_{i=1}^N f'(x_i')/N| > |\sum_{i=1}^{N}\hat{f}(x_i)/N - \sum_{i=1}^N \hat{f}'(x_i')/N|$, then we must have for some $i\in [2N]$, $f'(y_i) \neq \alpha_i^\star$. However, since $f'\in\Fcal$, we must have that $\{f'(y_i)\}_{i=1}^{2N}$ satisfies all constrains in the LP in~\pref{eq:LP_formulation}. Hence the assumption that $\sum_{i=1}^{2N} c_i f'(y_i) > \sum_{i=1}^{2N} c_i \alpha_i^\star$ contradicts to the fact that $\{\alpha_i\}_{i=1}^{2N}$ is the maximum solution of the LP formulation in~\pref{eq:LP_formulation}. Hence we prove the claim. 

\end{proof}

\section{Proof of~\pref{corr:metric_MDP}}
\label{app:proof_of_metric_MDP}

Since in this setting, $\Fcal_h$ for all $h\in[H]$ contains infinitely many functions, we need to discretize $\Fcal_h$ before we can apply the proof techniques from the proof of~\pref{thm:ft_main_theorem}. We use covering number. 

Denote $\Ncal(\Xcal, \epsilon, \Dcal)$ as the $\epsilon$-cover of the metric space $(\Xcal, \Dcal)$. Namely, for any $x\in\Xcal$, there exists a $x'\in \Ncal(\Xcal,\epsilon,\Dcal)$ such that $D(x',x) \leq \epsilon$. Consider any function class $\Fcal = \{f: \Xcal\to \RR, \|f\|_L \leq L, \|f\|_{\infty} \leq 1\}$ with $L \in \RR^+ $. Below we construct the $\epsilon$-cover over $\Fcal$.

For any $f\in\Fcal$, denote $\bar{f} \in \RR^{\abr{\Ncal(\Xcal, \epsilon, \Dcal)}}$ with the i-th element $\bar{f}_i$ being the function value $f(x_i)$ measured at the $i$-th element $x_i$ from $\Ncal(\Xcal,\epsilon,\Dcal)$. Hence $\bar{\Fcal} \defeq \{\bar{f}: f\in\Fcal \} \in \RR^{\abr{\Ncal(\Xcal, \epsilon, \Dcal)}}$, and $\|\bar{f}\|_{\infty} \leq C$ for any $\bar{f}\in\bar{\Fcal}$. Denote $\bar{\Ncal}(\bar{\Fcal}, \alpha, \|\cdot\|_{\infty})$ as the $\alpha$-cover  of $\bar{\Fcal}$. Let us denote the set $\Ncal \defeq \{f \in \Fcal: \bar{f}\in\bar{\Ncal}(\tilde{\Fcal}, \alpha, \|\cdot\|_{\infty})\}$.
 
 \begin{claim}
With the above set up,  for $\Fcal$'s $(\alpha+2L\epsilon)$-cover, i.e., $\Ncal(\Fcal, \alpha+2L\epsilon, \|f\|_{\infty})$, we have 
\begin{align*}
\abr{\Ncal(\Fcal, \alpha+2L\epsilon, \|\cdot\|_{\infty})} \leq \abr{\bar{\Ncal}(\bar{\Fcal}, \alpha, \|\cdot\|
_{\infty})} \leq \left(\frac{1}{\alpha}\right)^{\abr{\Ncal(\Xcal,\epsilon, \Dcal)}}.
\end{align*}
 \end{claim}
\begin{proof}
By definition, we know that for any $\bar{f}\in\bar{\Fcal}$, we have that there exists a $\bar{f^\star}\in\bar{\Fcal}$ such that $\|\bar{f} - \bar{f}^\star\|_{\infty} \leq \alpha$.  Now consider $\|f - f^\star\|_{\infty}$. Denote $x^\star = \arg\max_{x} \abr{f(x) - f^\star(x)}$ and ${x^\star}'$ is its closest point in $\Ncal(\Xcal, \epsilon, \Dcal)$. We have:
\begin{align*}
&\sup_{x} \abr{f(x) - f^\star(x)}  = \abr{f(x^\star) - f^\star(x^\star)} \leq \abr{f(x^\star) - f({x^\star}')} + \abr{f({x^\star}') - f^\star(x^\star)} \\
& \leq \abr{f(x^\star) - f({x^\star}')} + \abr{f({x^\star}') - f^\star({x^\star}')} + \abr{f^\star({x^\star}') - f^\star(x^\star)} \\
& \leq L\epsilon + \alpha + L\epsilon = 2L\epsilon + \alpha,
\end{align*} where the last inequality comes from the fact that $f, f^\star$ are $L$-Lipschitz continuous, $\Dcal(x^\star, {x^\star}') \leq \epsilon$, $\|\bar{f} - \bar{f^\star}\|_{\infty} \leq \alpha$ and ${x^\star}' \in \Ncal(\Xcal,\epsilon, \Dcal)$.
Hence, we just identify a subset of $\Fcal$ such that it forms a $\alpha+2L\epsilon$ cover for $\Fcal$ under norm $\|\cdot\|_{\infty}$. 

Note that $\bar{\Ncal}$ is a $\alpha$-cover with $\|\cdot\|_{\infty}$ for $\bar{\Fcal}$ which is a subset of $\abr{\Ncal(\Xcal,\epsilon,\Dcal)}$-dimension space. By standard discretization along each dimension, we prove the claim. 
\end{proof}

From the above claim, setting up $\alpha$ and $\epsilon$ properly, we have:
\begin{align*}
\abr{\Ncal(\Fcal, \epsilon/K, \|\cdot\|_{\infty})} \leq \left(\frac{K}{\epsilon}\right)^{\abr{\Ncal(\Xcal, \epsilon/(3KL), \Dcal)}}.
\end{align*}

Extending the analysis of~\pref{thm:ft_main_theorem} simply results extending the concentration result in~\pref{lem:concentration_pi_f}. Specifically via Bernstein's inequality and a union bound over $\Pi\times \Ncal(\Fcal, \epsilon/K, \|\cdot\|_{\infty})$, we have that for any $\pi\in\Pi$, $\tilde{f}\in \Ncal(\Fcal, \epsilon/K, \|\cdot\|_{\infty})$, with probability at least $1-\delta$,
\begin{align*}
&\left\lvert \left(\frac{1}{N}\sum_{i=1}^N K\pi(a_h^i|x_h^i)\tilde{f}(x_{h+1}^i) - \frac{1}{N}\sum_{i=1}^N \tilde{f}(\tilde{x}^i_{h+1})\right)  - \left({\EE}_{(x,a,x')\sim \nu_h \pi P^{\star}}[\tilde{f}(x')] - {\EE}_{x\sim \mu^{\star}_{h+1}}[\tilde{f}(x)]\right) \right\rvert \nonumber\\
    & \leq 4\sqrt{\frac{2K \abr{\Ncal(\Xcal, \epsilon/(3KL), \Dcal)} \log(2|\Pi|K/\epsilon(\delta))}{N}} + \frac{8K \abr{\Ncal(\Xcal, \epsilon/(3KL), \Dcal)} \log(2|\Pi| K/(\epsilon\delta))}{N}.
\end{align*}
Now using the fact that $ \Ncal(\Fcal, \epsilon/K, \|\cdot\|_{\infty})$ is an $\epsilon$-cover under norm $\|\cdot\|_{\infty}$, we have that for any $\pi\in\Pi$, $f\in\Fcal$, with probability at least $1-\delta$,
\begin{align*}
\left\lvert \left(\frac{1}{N}\sum_{i=1}^N K\pi(a_h^i|x^i_h)f(x^i_{h+1}) - \frac{1}{N}\sum_{i=1}^N f(\tilde{x}_{h+1}^i)\right)  - \left({\EE}_{(x,a,x')\sim \nu_h\pi P^{\star}}[f(x')] - {\EE}_{x\sim \mu^{\star}_{h+1}}[f(x)]\right) \right\rvert \nonumber\\ 
\leq 4\sqrt{\frac{2K \abr{\Ncal(\Xcal, \epsilon/(3KL), \Dcal)} \log(2|\Pi|3C/\epsilon(\delta))}{N}} + \frac{8K \abr{\Ncal(\Xcal, \epsilon/(3KL), \Dcal)} \log(2|\Pi| 3C/(\epsilon\delta))}{N} + 2\epsilon.
\end{align*}

The rest of the proof is the same as the proof of~\pref{thm:ft_main_theorem}. 


\section{\fail{} in Interactive Setting}
\label{app:ifail}

\begin{algorithm}
\begin{algorithmic}[1]
\STATE Set $\bm\pi = \emptyset$
\FOR{$h = 1$ to $H-1$}
    \STATE $\Dcal = \emptyset, \tilde{\Dcal} = \emptyset$
    \FOR{$i = 1$ to $n$}
        \STATE Reset $x^{(i)}_1\sim \rho$ and from $x^{(1)}_i$ execute $\bm\pi=\{\pi_1,\dots,\pi_{h-1}\}$ to generate  $x_h^{(i)}$
        \STATE Execute $a\sim U(\Acal)$ to generate $x_{h+1}^{(i)}$ and add it to $\Dcal$ \label{line:collect_pi_interactive}
        \STATE Reset again $\tilde{x}_1^{(i)}\sim \rho$ and from $\tilde{x}_1^{(i)}$ execute $\bm{\pi} = \{\pi_1,\dots,\pi_{h-1}\}$ to generate $\tilde{x}_{h}^{(i)}$
        \STATE Ask expert to execute its policy at $\tilde{x}_h^{(i)}$ for one step, observe $\tilde{x}_{h+1}^{(i)}$, and add it to $\tilde{\Dcal}$ \label{line:collect_expert_interactive}
    \ENDFOR
    \STATE Set $\pi_h$ to be the return of~\pref{alg:l2m} with inputs $\left(\tilde{\Dcal}, \Dcal, \Pi, \Fcal, T\right)$ \label{line:call_l2m}
    \STATE Append $\pi_h$ to $\bm\pi$
\ENDFOR
\end{algorithmic}
\caption{ \ifail($\Pi$, $\Fcal$, $\epsilon$, n, T)}
\label{alg:main_alg_2}
\end{algorithm}

Recall that with $\{\pi_1, \dots, \pi_{h-1}\}$ being fixed, we denote $\nu_h$ as resulting observation distribution resulting at time step $h$. The interactiveness comes from the ability we can query expert to generate next observation conditioned on states sampled from $\nu_h$---the states that would be visited by learner at time step $h$. Let us define $d(\pi \vert \nu_h, \pi_h^\star)$ as:
\begin{align*}
    d_{\Fcal_{h+1}}(\pi \vert \nu_h, \pi_h^\star) \defeq \max_{f\in\Fcal_{h+1}} 
   \left( \mathbb{E}_{x\sim \nu_h}\mathbb{E}_{a\sim \pi, x'\sim P_{x,a}}[f(x')] - \mathbb{E}_{x\sim \nu_h}\mathbb{E}_{a\sim \pi^\star, x'\sim P_{x,a}}[f(x')]\right).
\end{align*} Note that different from $d_{\Fcal_{h+1}}(\pi\vert \nu_h, \mu^\star_{h+1})$, in $d_{\Fcal_{h+1}}(\pi \vert \nu_h, \pi_h^\star)$, the marginal distributions on $x$ are the same for both $\pi$ and $\pi_h^\star$ and we directly access $\pi_h^\star$ to generate epxert observations at $h+1$ rather than thorough the expert observation distribution $\mu^\star_{h+1}$. In other words, we use IPM to compare the observation distribution at time step $h+1$ after applying  $\pi$ and the observation distribution at time step $h+1$ after applying  $\pi^\star$, \emph{conditioned on the distribution $\nu_h$ generated by the previously learned policies $\{\pi_1,\dots, \pi_{h-1}\}$}. In~\pref{alg:main_alg_2}, at every time step $h$,   to find a policy $\pi_h$ that approximately minimizes $d_{\Fcal_{h+1}}(\pi \vert \nu_h, \pi_h^\star)$, we replace expectations in $d_{\Fcal_{h+1}}(\pi \vert \nu_h, \pi^\star_h)$ by proper samples (\pref{line:collect_pi_interactive} and~\pref{line:collect_expert_interactive}), and then call~\pref{alg:l2m} (\pref{line:call_l2m}). 


\begin{proof}

Recall the definition of $d(\pi \vert \nu_h, \pi_h^\star)$,
\begin{align*}
    d_{\Fcal_{h+1}}(\pi \vert \nu_h, \pi_h^\star) = \max_{f\in\Fcal_{h+1}}\left( 
    \mathbb{E}_{x\sim \nu_h}\mathbb{E}_{a\sim \pi,x'\sim P_{x,a}}[f(x')] - \mathbb{E}_{x\sim \nu_h}\mathbb{E}_{a\sim \pi^\star,x'\sim P_{x,a}}[f(x')]
    \right),
\end{align*} with $\nu_h$ being the distribution over $\Xcal_h$ resulting from executing policies $\{\pi_1,\dots,\pi_{h-1}\}$. 

We will use~\pref{lem:concentration_pi_f},~\pref{lem:no_regret_results}, and~\pref{lem:PDL} below. 

The Performance Difference Lemma (\pref{lem:PDL}) tells us that:
\begin{align}
 J(\bm\pi) - J(\bm\pi^\star) 
    & = \sum_{h=1}^H \mathbb{E}_{x\sim\mu^{\bm\pi}_h}\left[\mathbb{E}_{a\sim \pi_h, x'\sim P_{x,a}}\left[V^{\star}_{h+1}(x')\right] - \mathbb{E}_{a\sim \pi^\star_h, x'\sim P_{x,a}}\left[V^{\star}_{h+1}(x')\right]\right]\nonumber\\
    & \leq \sum_{h=1}^H \left\lvert  \mathbb{E}_{x\sim\mu^{\bm\pi}_h}\left[\mathbb{E}_{a\sim \pi_h, x'\sim P_{x,a}}\left[V^{\star}_{h+1}(x')\right] - \mathbb{E}_{a\sim \pi^\star_h, x'\sim P_{x,a}}\left[V^{\star}_{h+1}(x')\right]\right]   \right\rvert \nonumber\\
    & \leq \sum_{h=1}^H d_{\Fcal_{h+1}}(\pi_h \vert \mu_h^{\bm\pi}, \pi_h^\star),
    \label{eq:perf_ell}
\end{align} where the last inequality comes from the realizable assumption that $V_{h}^\star \in \Fcal_h$.

At every time step $h$, mapping to~\pref{thm:L2M_result} with $\nu_h= \mu^{\bm\pi}_h$, $T = 4K^2/\epsilon^2$, $n = K\log(|\Pi||\Fcal|/\delta)/\epsilon^2$, we have that with probability at least $1-\delta$:
\begin{align*}
    d_{\Fcal_{h+1}}(\pi_h \vert \mu_h^{\bm\pi}, \pi_h^\star) \leq \min_{\pi\in\Pi_h}d_{\Fcal_{h+1}}(\pi \vert \mu_h^{\bm\pi} \pi_h^\star), + \epsilon. 
\end{align*}
Note that $\min_{\pi\in\Pi_h} d_{\Fcal_{h+1}}(\pi \vert \mu_h^{\bm\pi}, \pi_h^\star) \leq d_{\Fcal_{h+1}}(\pi_h^\star \vert \mu_h^{\bm\pi}, \pi_h^\star) = 0$, since $\pi_h^\star\in\Pi_h$ by the realizable assumption. Hence, we have that:
\begin{align*}
    d_{\Fcal_{h+1}}(\pi_h \vert \mu_h^{\bm\pi}, \pi_h^\star) \leq \epsilon.
\end{align*}
Hence, using~\pref{eq:perf_ell}, and a union bound over all time steps $h\in [H]$, we have that with probability at least $1-\delta$,
\begin{align*}
    J(\bm\pi) - J(\bm\pi^\star) \leq H\epsilon, 
\end{align*} with $T = 4K^2/\epsilon^2$, and $N =K\log(H|\Pi||\Fcal|/\delta)/\epsilon^2$. Since in every round $h$, we need to draw $N$ many trajectories, hence, the total number of trajectories we need is at most $HK\log(H|\Pi||\Fcal|/\delta)/\epsilon^2$.

\end{proof}

\section{Relaxation of~\pref{ass:realizable}}
\label{app:relax}

Our theoretical results presented so far rely on the realizable assumption (\pref{ass:realizable}). While equipped with recent advances in powerful non-linear function approximators (e.g., deep neural networks), readability can be ensured,  in this section, we relax the realizable assumption and show that our algorithms' performance only degenerates mildly. We relax~\pref{ass:realizable} as follows:
\begin{assum}
[Approximate Realizability]
We assume $\Pi$ and $\Fcal$ is approximate realizable in a sense that for any $h\in[H]$, we have $\min_{\pi\in\Pi_h} \max_{x,a}\|\pi(a|x) - \pi^\star_h(a|x)\|\leq \epsilon_{\Pi}$ and $\min_{f\in\Fcal_h} \|f - V_{h}^\star\|_{\infty} \leq \epsilon_{\Fcal}$.
\label{ass:relaxed}
\end{assum}
The above assumption does not require $\Fcal$ and $\Pi$ to contain the exact $V_h^\star$ and $\pi^\star_h$, but assumes $\Fcal$ and $\Pi$ are rich enough to contain functions that can  approximate $V_h^\star$ and $\pi^\star_h$ uniformly well (i.e., $\epsilon_{\Fcal}$ and $\epsilon_{\Pi}$ are small). Without any further modification of~\pref{alg:main_alg} and~\pref{alg:main_alg_2} for non-interactive and interactive setting, we have the following corollary. 
\begin{corollary}
 Under~\pref{ass:relaxed}, for $\epsilon \in (0,1)$ and $\delta\in (0,1)$, with $T = \Theta(K/\epsilon^2)$, $n = \Theta(K\log(|\Pi||\Fcal|H/\delta)/\epsilon^2)$, with probability at least $1-\delta$,   (1) for non-interactive setting, \fail{} (\pref{alg:main_alg}) outputs a policy $\bm\pi$ with $J(\bm\pi) - J(\bm\pi^\star) \leq O\left( H^2(\epsilon_{\mathrm{be}} + \epsilon) + H(\epsilon_{\Fcal}+\epsilon_{\Pi})\right)$, and (2) for interactive setting, \ifail{} (\pref{alg:main_alg_2}) outputs a policy $\bm\pi$ with $J(\bm\pi) - J(\bm\pi^\star) \leq O\left(H\epsilon + H\epsilon_{\Fcal} + H\epsilon_{\Pi}\right)$, by using at most $\tilde{O}((HK/\epsilon^2)\log(|\Pi||\Fcal|/\delta))$ many trajectories under both settings.
 \label{corr:robust_extension}
\end{corollary}
 
The proof is deferred to~\pref{app:relaxation_proof}. 

\subsection{Proof of~\pref{corr:robust_extension}}
\label{app:relaxation_proof}
\begin{proof}[Proof of Corollary~\ref{corr:robust_extension}]

For any $h$, denote $g_h$ as
\begin{align*}
    g_h = \arg\min_{g\in \Fcal} \|g - V_{h}^\star\|_{\infty}
\end{align*}

Below we prove the first bullet in~\pref{corr:robust_extension}, i.e., the results for non-interactive setting.

\paragraph{Non-Interactive Setting}  
Using PDL (\pref{lem:PDL}), we have
\begin{align*}
    &J(\bm\pi) - J(\bm\pi^\star) \\ & \leq \sum_{h=1}^H \left\lvert\mathbb{E}_{x\sim\mu^{\bm\pi}_h}\left[\mathbb{E}_{a\sim \pi_h, x'\sim P_{x,a}}\left[V^{\star}_{h+1}(x')\right] - \mathbb{E}_{a\sim \pi^\star_h, x'\sim P_{x,a}}\left[V^{\star}_{h+1}(x')\right]\right]\right\rvert \\
    & \leq \sum_{h=1}^H \left\lvert \mathbb{E}_{x\sim\mu^{\bm\pi}_h}\left[\mathbb{E}_{a\sim \pi_h, x'\sim P_{x,a}}\left[V^{\star}_{h+1}(x')\right]\right]  - \mathbb{E}_{x\sim \mu_h^{\star}} \mathbb{E}_{a\sim \pi^\star_h, x'\sim P_{x,a}}\left[V^{\star}_{h+1}(x')\right]  \right\rvert \\
    &\;\;\;\; + \left\lvert
    \mathbb{E}_{x\sim \mu_h^{\star}} \mathbb{E}_{a\sim \pi^\star_h, x'\sim P_{x,a}}\left[V^{\star}_{h+1}(x')\right]
    - \mathbb{E}_{x\sim \mu_h^{\bm\pi}} \mathbb{E}_{a\sim \pi^\star_h, x'\sim P_{x,a}}\left[V^{\star}_{h+1}(x')\right] 
    \right\rvert \\
    & \leq \sum_{h=1}^H \left\lvert \mathbb{E}_{x\sim\mu^{\bm\pi}_h}\left[\mathbb{E}_{a\sim \pi_h, x'\sim P_{x,a}}\left[g_{h+1}(x')\right]\right]  - \mathbb{E}_{x\sim \mu_h^{\star}} \mathbb{E}_{a\sim \pi^\star_h, x'\sim P_{x,a}}\left[g_{h+1}(x')\right]  \right\rvert \\
    & \;\;\;\; + \left\lvert
    \mathbb{E}_{x\sim \mu_h^{\star}} \mathbb{E}_{a\sim \pi^\star_h, x'\sim P_{x,a}}\left[g_{h+1}(x')\right]
    - \mathbb{E}_{x\sim \mu_h^{\bm\pi}} \mathbb{E}_{a\sim \pi^\star_h, x'\sim P_{x,a}}\left[g_{h+1}(x')\right] 
    \right\rvert  + 4\epsilon_{\Fcal} \\
    & \leq \sum_{h=1}^H \left( d_{\Fcal_{h+1}}(\pi_h \vert \mu_h^{\bm\pi}, \mu_h^\star)+ \Delta_h + 4\epsilon_{\Fcal}\right).
\end{align*}
Now repeat the same recursive analysis for $d_{\Fcal_{h+1}}(\pi_h \vert \mu_h^{\bm\pi}, \mu_h^\star)$ as we did in proof of~\pref{thm:ft_main_theorem} in~\pref{app:proof_main_theorem}, we can prove the first bullet in the corollary.

Now we prove the second bullet in~\pref{corr:robust_extension}, i.e., the results for interactive setting.

\paragraph{Interactive Setting}
Again, using Performance Difference Lemma (\pref{lem:PDL}), we have
\begin{align*}
 J(\bm\pi) - J(\bm\pi^\star) 
    & = \sum_{h=1}^H \mathbb{E}_{x\sim\mu^{\bm\pi}_h}\left[\mathbb{E}_{a\sim \pi_h, x'\sim P_{x,a}}\left[V^{\star}_{h+1}(x')\right] - \mathbb{E}_{a\sim \pi^\star_h, x'\sim P_{x,a}}\left[V^{\star}_{h+1}(x')\right]\right]\nonumber\\
    & \leq \sum_{h=1}^H \left\lvert  \mathbb{E}_{x\sim\mu^{\bm\pi}_h}\left[\mathbb{E}_{a\sim \pi_h, x'\sim P_{x,a}}\left[V^{\star}_{h+1}(x')\right] - \mathbb{E}_{a\sim \pi^\star_h, x'\sim P_{x,a}}\left[V^{\star}_{h+1}(x')\right]\right]   \right\rvert \nonumber\\
    & \leq \sum_{h=1}^H \left\lvert  \mathbb{E}_{x\sim\mu^{\bm\pi}_h}\left[\mathbb{E}_{a\sim \pi_h, x'\sim P_{x,a}}\left[g_{h+1}(x')\right] - \mathbb{E}_{a\sim \pi^\star_h, x'\sim P_{x,a}}\left[g_{h+1}(x')\right]\right]   \right\rvert \\
    &+ \left\lvert \mathbb{E}_{x\sim\mu^{\bm\pi}_h}\mathbb{E}_{a\sim \pi_h, x'\sim P_{x,a}}[V_{h+1}^\star(x') - g_{h+1}(x')]  \right\vert + \left\lvert \mathbb{E}_{x\sim\mu^{\bm\pi}_h}\mathbb{E}_{a\sim \pi_h^\star, x'\sim P_{x,a}}[V_{h+1}^\star(x') - g_{h+1}(x')]  \right\vert \\
    & \leq \sum_{h=1}^H \max_{f\in\Fcal_{h+1}} \left\lvert  \mathbb{E}_{x\sim\mu^{\bm\pi}_h}\left[\mathbb{E}_{a\sim \pi_h, x'\sim P_{x,a}}\left[f(x')\right] - \mathbb{E}_{a\sim \pi^\star_h, x'\sim P_{x,a}}\left[f(x')\right]\right]   \right\rvert + 2\epsilon_{\Fcal}\\
    & = \sum_{h=1}^H\left( d_{\Fcal_{h+1}}(\pi_h \vert \mu_h^{\bm\pi}, \pi_h^\star) + 2\epsilon_{\Fcal}\right)
\end{align*}
Now  repeat the same steps from the proof of~\pref{thm:sample_complexity_interactive} after~\pref{eq:perf_ell} in proof of~\pref{thm:sample_complexity_interactive}, we can prove the second bullet in the corollary.

\end{proof}

\section{Missing Details on \ilo {} with State Abstraction}
\label{app:abstraction}

We consider the bisimulation model from~\pref{eq:bisimulation}. The following proposition summarizes the conclusion in this section. 
\begin{proposition}
Assume Bisimulation holds (Eq.~\ref{eq:bisimulation}) and set $\Fcal_{h} = \left\{f: \|f\|_{\infty}\leq 1, f(x) = f(x'), \forall x,x' \text{ s.t. } \phi(x) = \phi(x')\right\}, \forall h\in [H]$ to be piece-wise constant functions over the partitions induced from $\phi$. We have:
\begin{enumerate} 
    \item $V_h^\star$ is a piece-wise constant function for all $h\in [H]$, 
    \item $\epsilon_{\mathrm{be}} = 0$, 
    \item $\sup_{f\in\Fcal_h}(\sum_{i=1}^N f(x_i)/N - \sum_{i=1}^N f(x_i')/N)$ can be solved by LP, for all $h\in [H]$,  
    \item given any $\{x_i\}_{i=1}^N$, the Rademacher complexity of $\Fcal_h$ is in the order of $O(\sqrt{|\Scal|/N})$, i.e.,  $(1/N)\EE_{\sigma} [\sup_{f\in\Fcal_h}\sum_{i=1}^N \sigma_i f(x_i)] = O(\sqrt{|\Scal|}/N)$, with $\sigma_i$ being a Rademacher number. 
\end{enumerate}
\end{proposition}
The above proposition states that by leveraging the abstraction, we can design discriminators to be piece-wise constant functions over the partitions induced by $\phi$, such that inherent Bellman error is zero, and the discriminator class has bounded statistical complexity.
Below we prove the above proposition. The first two points in the above proposition were studied in \cite{chen2019information}. For completeness, we simply prove all four points below. 

\paragraph{Piece-wise constant $V^\star$} First, we show that $V_h^\star(x)$ is piece-wise constant over the partitions induced from $\phi$. Starting from $H$, via~\pref{eq:bisimulation}, we know that $V_H^\star(x) = c(x)$, which is piece-wise constant over the partitions induced from $\phi$. Then let us assume that for $h+1$, we have $V_{h+1}^\star(x) = V_{h+1}^\star(x')$ for any $x,x'$ s.t. $\phi(x) = \phi(x')$. At time step $h$, via Bellman equation, we know:
\begin{align*}
    V_h^\star(x) =\EE_{a\sim \pi^\star(\cdot|x)}\EE_{x'\sim P_{x,a}} V_{h+1}^\star(x').
\end{align*} Hence for any two $x,x'$ with $\phi(x) = \phi(x')$, we have:
\begin{align*}
    V_h^\star(x) - V_h^\star(x')& =  \EE_{a\sim \pi^\star(\cdot|x),x''\sim P_{x,a}}V_{h+1}^\star(x'') - \EE_{a\sim \pi^\star(\cdot|x'),x''\sim P_{x',a}}V_{h+1}^\star(x'') \\
    & = \sum_{a} \pi^\star(a|x) \left( \sum_{s\in \Scal}\sum_{x''\in\phi^{-1}(s)}\left(P(x''|x,a) - P(x''|x',a)\right) V_{h+1}^\star(x'') \right)\\
    & = \sum_{a} \pi^\star(a|x) \left( \sum_{s\in \Scal} V_{h+1}^\star(s)\sum_{x''\in\phi^{-1}(s)}\left(P(x''|x,a) - P(x''|x',a)\right) \right) = 0, 
\end{align*} where the second and the last equality use~\pref{eq:bisimulation}. In the third equality above, we abuse the notation $V_{h+1}^\star(s)$ for $s\in\Scal$ to denote the value of $V_{h+1}^\star(x)$ for any $x$ such that $\phi(x) = s$.

\paragraph{Inherent Bellman Error} With $\Fcal_{h} = \left\{f: \|f\|_{\infty}\leq 1, f(x) = f(x'), \forall x,x' \text{ s.t. } \phi(x) = \phi(x')\right\}, \forall h\in [H]$, we can show $\epsilon_{\mathrm{be}} = 0$ as follows. For any $x,x'\in\Xcal$ with $\phi(x) = \phi(x'), f\in\Fcal_{h+1}$, we have:
\begin{align*}
    &\EE_{a\sim \pi^\star(\cdot|x)}\EE_{x''\sim P_{x,a}}f(x'') - \EE_{a\sim \pi^\star(\cdot|x')}\EE_{x''\sim P_{x',a}}f(x'')\\
    & = \sum_{a} \pi^\star(a|x) \left( \sum_{s\in\Scal} f(s)\sum_{x''\in\phi^{-1}(s)} (P(x''|x,a) - P(x''|x',a)) \right) = 0,
\end{align*} where again we abuse the notation $f(s)$ to denote that value $f(x)$ for any $x$ such that $\phi(x) = s$. Namely, $\Gamma_h f$ is also a piece-wise constant over the partitions induced from $\phi$. Since $\|f\|_{\infty}\leq 1$, it is also easy to see that $\|\Gamma_h f\|_{\infty}\leq 1$. Hence we have $\Gamma_h f \in \Fcal_h$.

\paragraph{Reduction to LP} Regarding evaluating $\sup_{f\in\Fcal_h} \left(\sum_{i=1}^N f(x_i)/N - \sum_{i=1}^N f(x_i')/N\right)$, we can again reduce it an LP. Denote $\alpha\in [-1,1]^{|\Scal|}$, where the i-th entry in $\alpha$ corresponds to the i-th element in $\Scal$. We denote $\alpha_s$ as the entry in $\alpha$ that corresponds to the state $s$ in $\Scal$.  Take $\{x_i\}_{i=1}^N$, and compute $c_s = \sum_{i=1}^N \one(\phi(x_i) = s) $ for every $s\in \Scal$ (i.e., $c_s$ is the number of points mapped to $s$). Take $\{x_i'\}_{i=1}^N$ and compute $c'_s = \sum_{i=1}^N \one(\phi(x'_i) = s)$. We solve the following LP:
\begin{align*}
    & \max_{\alpha\in\RR^{|\Scal|}} \sum_{s\in \Scal} \left( c_s\alpha_s/N - c'_s\alpha_s/N \right), \\
    & s.t.,  \alpha_s \in [-1,1], \forall s\in\Scal. 
\end{align*} Denote the solution of the above LP as $\alpha^\star$. Then $f^\star(x) = \alpha^\star_{\phi(x)}$. 

\paragraph{Complexity of Discriminators $\Fcal_h$} Regarding the complexity of $\Fcal_h$, note that $\Fcal_h$ essentially corresponds to a $\abr{\Scal}$-dim box: $[-1,1]^{|\Scal|}$. Again, consider a dataset $\{x_i\}_{i=1}^N$ and the counts $\{c_s\}_{s\in\Scal}$. For any $f$, and Rademacher numbers $\sigma\in \{-1,1\}^{N}$, we have \begin{align*}
\sum_{i=1}^N \sigma_i f(x_i) = \sum_{s\in\Scal} f_s \sum_{i\in \phi^{-1}(s)} \sigma_i \leq \sum_{s\in \Scal} \abr{\sum_{i\in\phi^{-1}(s)}\sigma_i}. \end{align*} Note that $(\EE_{\sigma} \abr{\sum_{i=1}^N \sigma_i})^2 \leq \EE_{\sigma} (\sum_{i=1}^N \sigma_i)^2  = N$, which implies that $\EE_{\sigma}|\sum_{i=1}^N \sigma_i| \leq \sqrt{N}$. Hence, \begin{align*}
\EE_{\sigma} \sum_{i=1}^N \sigma_i f(x_i) \leq \sum_{s\in\Scal} \EE_{\sigma}| \sum_{i\in\phi^{-1}(s)}\sigma_i| \leq \sum_{s\in\Scal} \sqrt{c_s} \leq \sum_{s\in\Scal}\sqrt{N/|\Scal|} = \sqrt{N|\Scal|}.
\end{align*} Now, we can show that the Rademacher complexity of $\Fcal_h$ is bounded as follows:
\begin{align*}
 \frac{1}{N}\EE_{\sigma}\left[\sup_{f\in\Fcal_h}\sum_{i=1}^N \sigma_i f(x_i)\right] \leq \sqrt{N|\Scal|}/N = \sqrt{\frac{|\Scal|}{N}}.
\end{align*}

\section{Additional Experiments}
\label{app:additional}

When we design the utility in~\pref{eq:utility}, we sample actions from $U(\Acal)$ and then perform importance weighting. This ensures that in analysis the variance will be bounded by $K$. In practice, we can use any reference policy to generate actions, and then perform importance weighting accordingly. Assume that we have a dataset $\Dcal = \{x_h^i, a_h^i, p_h^i, x_{h+1}^i \}_{i=1}^N$ and the expert's dataset $\Dcal^\star = \{\tilde{x}_{h+1}^i\}_{i=1}^{N'}$, where $p_h^i$ is the probability of action $a_h^i$ being chosen at $x_h^i$. We can form the utility as follows:
\begin{align}
\label{eq:new_utility}
    u(\pi,f) \defeq\sum_{i=1}^N (\pi(a_h^i|x_h^i)/p_h^i)f(x_{h+1}^i)/N -\sum_{i=1}^{N'}f(\tilde{x}_{h+1}^{i})/N'.
\end{align} As long as the probability of choosing any action at any state is lower bounded, then the variance of the above estimator is  upper bounded. This formulation also immediately extends \fail{} to continuous action space setting. For a parameterized policy $\pi_{\theta}$, given any $f$, we can compute $\nabla_{\theta}u(\pi_{\theta}, f)$ easily. If $a_h\sim \pi_{\theta}(\cdot|x)$ (i.e., on-policy samples), then for any fixed $f$, the policy gradient $\nabla_{\theta} u(\pi_{\theta}, f)$ can be estimated using the REINFORCE trick: 
\begin{align}
\label{eq:pg}
    \nabla_{\theta} u(\pi_{\theta},f)|_{\theta=\theta_0} = (1/N)\sum_{i=1}^N \nabla_{\theta}(\ln \pi_{\theta}(a_h^i|x_h^i)|_{\theta=\theta_0}) f(x_{h+1}^i).
\end{align} 

With the form of $\nabla_{\theta}u(\pi_{\theta},f)$, we can perform the min-max optimization in Alg.~\ref{alg:l2m} by iteratively finding the maximizer $f^n = \arg\max_{f}u(\pi_{\theta_n}, f)$ using LP oracle, and then perform gradient descent update $\theta^{n+1} = \theta^n - \eta^n \nabla_{\theta}u(\pi_{\theta^n}, f^n)$. See~\pref{alg:l2m_2} below.
\begin{algorithm}[t]
\begin{algorithmic}[1]
\FOR{$n = 0$ to $T$}
    \STATE 
    $f^n = \arg\max_{f\in\Fcal} u(\pi_{\theta^n}, f)$ (LP Oracle) 
    \STATE 
    $u^n = u(\pi_{\theta^n}, f^n)$
    \STATE 
    $\theta^{n+1} = \theta^n - \nabla_{\theta}u(\pi_{\theta^n},f^n) $
    (Policy Gradient)  
\ENDFOR
\STATE \textbf{Output}:  $\pi^{n^\star}$ with $n^\star = \arg\min_{n\in [T]} u^n$  \label{line:output_l2m}
\end{algorithmic}
\caption{Min-Max Game ($\Dcal^{\star}, \Dcal, \Pi, \Fcal, T, \theta_0$)}
\label{alg:l2m_2}
\end{algorithm} Note that in~\pref{alg:l2m_2} the dataset $\Dcal = \{x_h^i, a_h^i, p_h^i, x_{h+1}^i\}$ contains $p_h^i$ which is the probability of $a_h^i$ being chosen at $x_h^i$.  We can integrate~\pref{alg:l2m_2} into the forward training framework.

\begin{algorithm}
\begin{algorithmic}[1]
\STATE Set $\bm\pi = \emptyset$
\FOR{$h = 1$ to $H-1$}
    \STATE Initialize $\pi_h$
    \STATE Extract expert's data at $h+1$: $\tilde{\Dcal}_{h+1} = \{\tilde{x}_{h+1}^i\}_{i=1}^{n'}$ 
    \STATE $\Dcal_1 = \emptyset, \dots \Dcal_h = \emptyset$
    \FOR{$i = 1$ to $n$}
        \STATE Execute $ \{\pi_1,\dots,\pi_{h-1}\}$ to generate $\tau^i = \{x_1^i, a_1^i, p_1^i, x_2^i, \dots, x_{h-1}^i, a_{h-1}^i,p_{h-1}^i, x_h^i\}$ with $p_{t}^i = \pi_t(a_t^i|x_t^i)$
        \STATE For any $t\in [h-1]$, add $(x_t^i, a_t^i, p_t^i, x_{t+1}^i)$ to $\Dcal_t$
        \STATE Execute $a_h^i\sim U(\Acal)$ to generate $x_{h+1}^{i}$ and add $(x_h^i, a_h^i,p_h^i, x_{h+1}^i)$ to $\Dcal_{h}$ with $p_h^i$ being the probability corresponding to the uniform distribution over $\Acal$
    \ENDFOR
    \STATE For all $t\in [h]$, update $\pi_t$ to be the return of~\pref{alg:l2m_2} with inputs $\left(\tilde{D}_{t+1}, \Dcal_t, \Pi_h, \Fcal_{h+1}, T, \pi_t \right)$
\ENDFOR
\end{algorithmic}
\caption{{\fail$^*$} ($\{\Pi_h\}_h$, $\{\Fcal_h\}_h$, $\epsilon, n,n', T$)}
\label{alg:main_alg_2}
\end{algorithm}

In~\pref{alg:main_alg}, at every time step $h$, we execute the current sequence of policies $\bm{\pi} = \{\pi_1, \dots, \pi_{h-1}\}$ to collect samples at time step $h$, i.e., $x_h$. We then throw away all generated samples $\{x_1, \dots, x_{h-1}\}$ except $x_h$. While this simplifies the analysis, in practice, we could leverage these samples $\{x_1, \dots, x_{h-1}\}$ as well, especially now we can form the utility with on-policy samples and compute the corresponding policy gradient (\pref{eq:pg}). This leads us to Alg.~\ref{alg:main_alg_2}. Namely, in~\pref{alg:main_alg_2}, when training $\pi_h$, we also incrementally update $\pi_1, \dots, \pi_{h-1}$ using their on-policy samples (Line 11~\pref{alg:main_alg_2}).

\begin{figure}[h!]
\centering
\begin{subfigure}{.235\textwidth}
  \centering
  \includegraphics[width=1.0\linewidth]{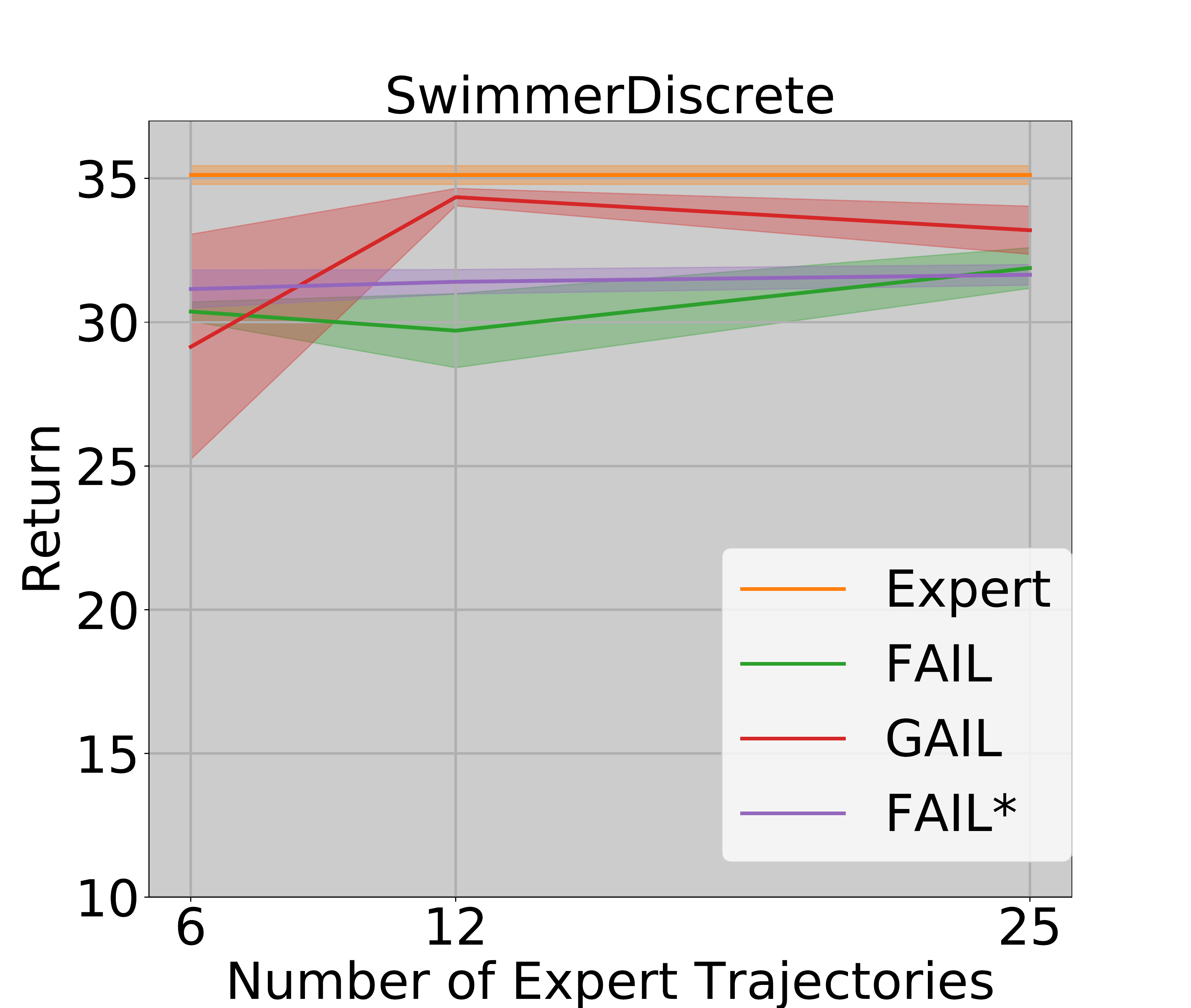}
  \caption{Swimmer (1m)}
  \label{fig:swimmer_new}
\end{subfigure}%
\begin{subfigure}{.235\textwidth}
  \centering
  \includegraphics[width=1.0\linewidth]{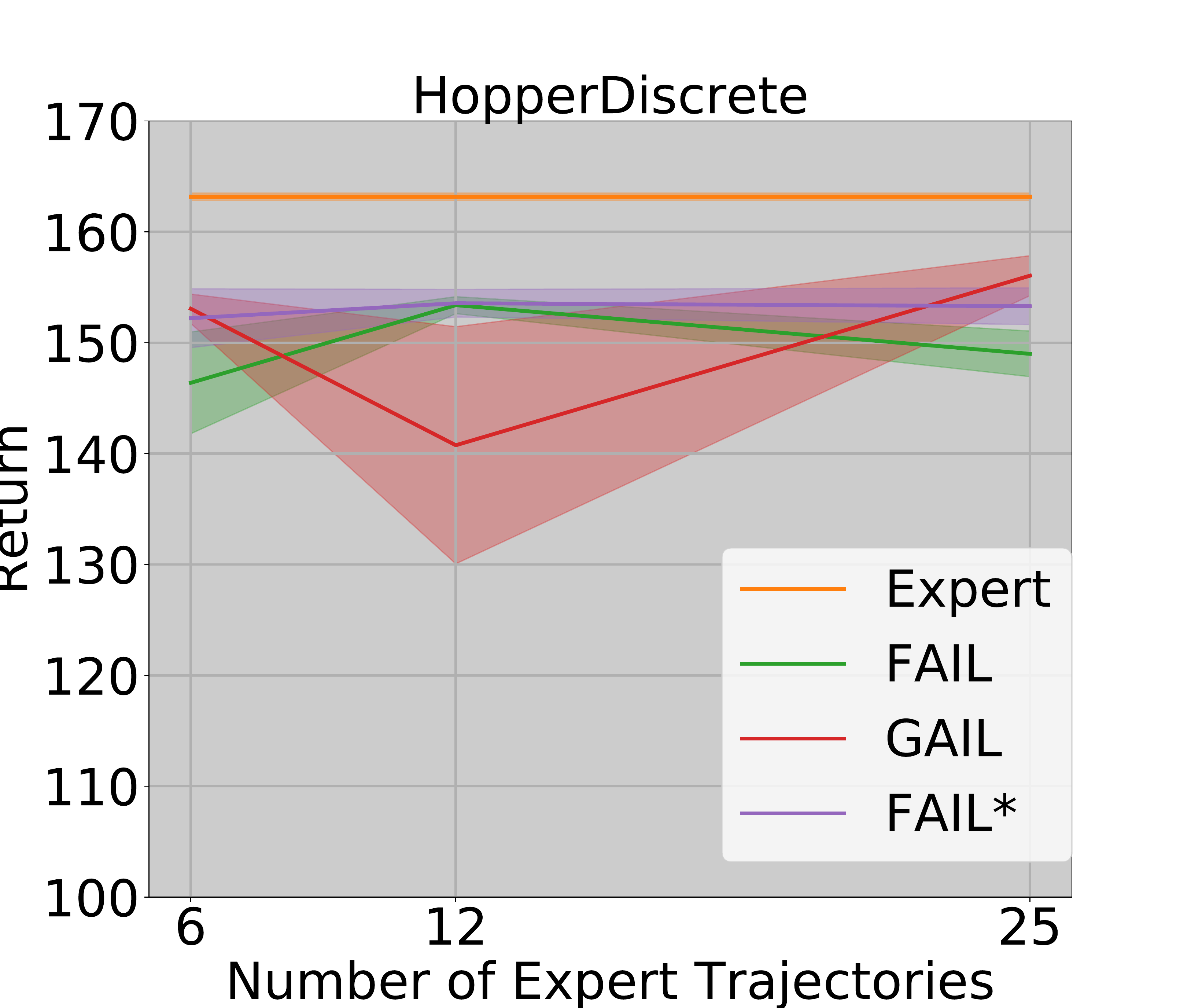}
  \caption{Hopper (1m)}
  \label{fig:hopper_new}
\end{subfigure}
\begin{subfigure}{.235\textwidth}
  \centering
  \includegraphics[width=1.0\linewidth]{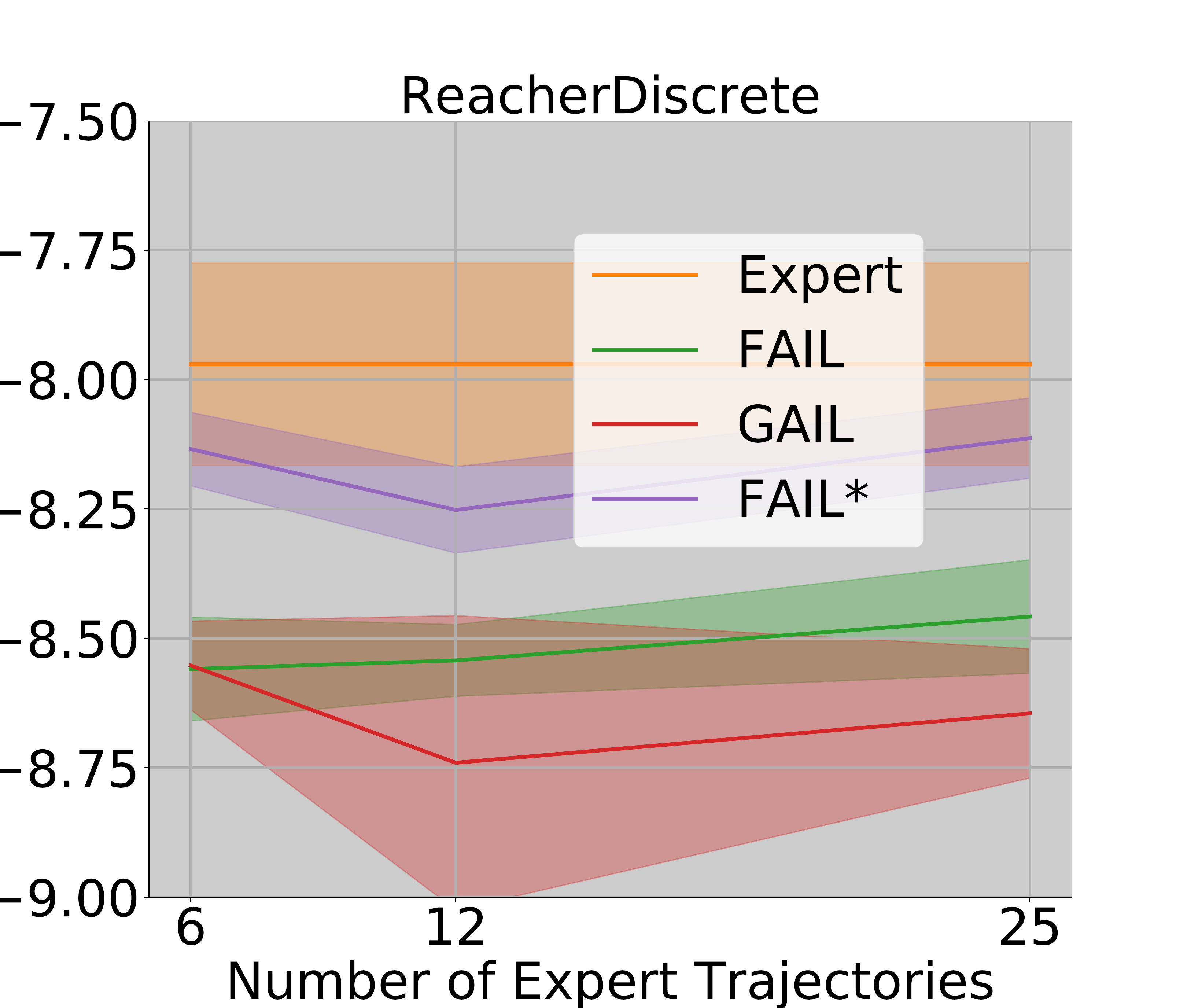}
  \caption{Reacher (1m)}
  \label{fig:reacher_new}
\end{subfigure}
\begin{subfigure}{.235\textwidth}
  \centering
  \includegraphics[width=1.0\linewidth]{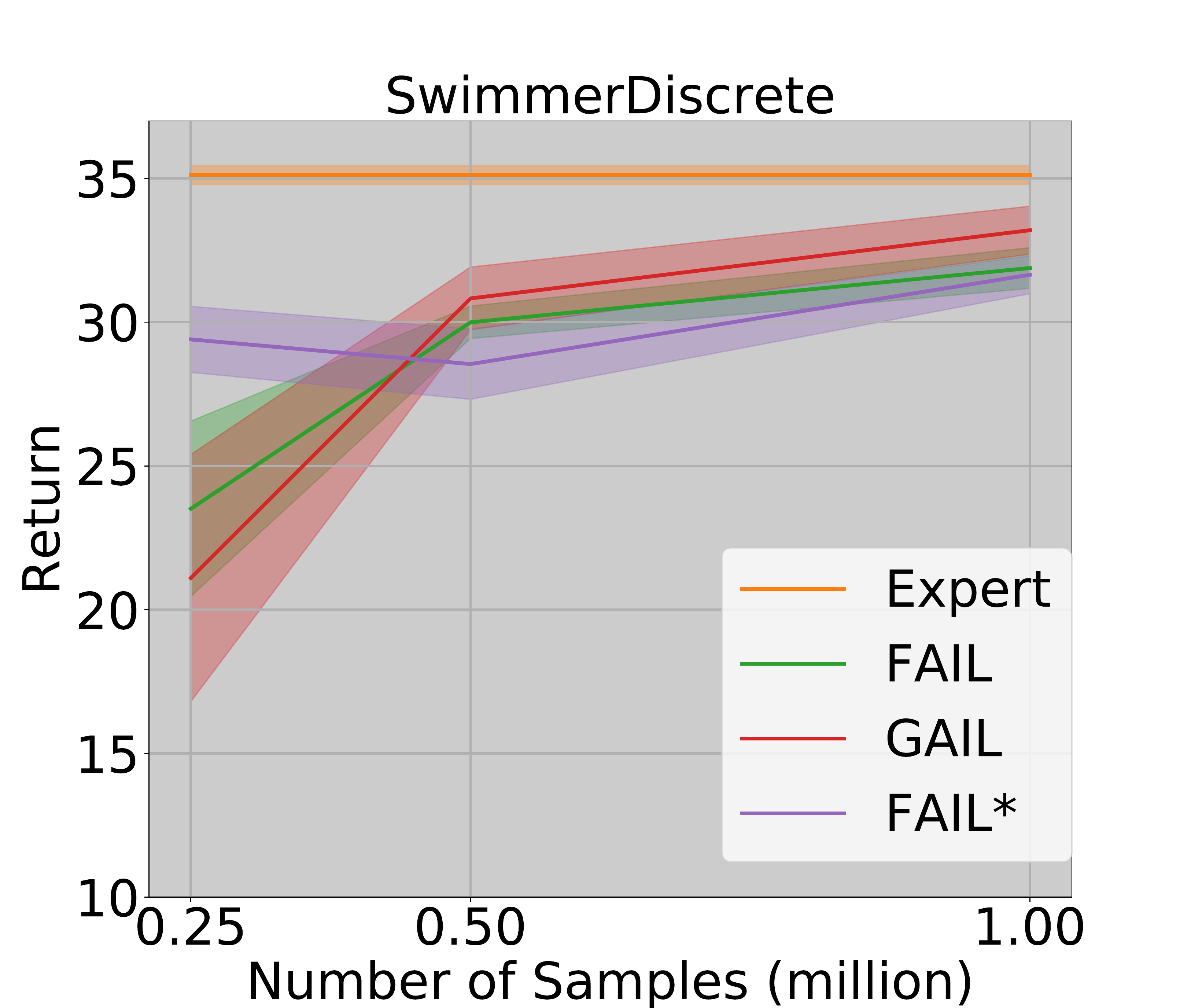}
  \caption{Swimmer (25)}
  \label{fig:swimmer_new_2}
\end{subfigure}%
\begin{subfigure}{.235\textwidth}
  \centering
  \includegraphics[width=1.0\linewidth]{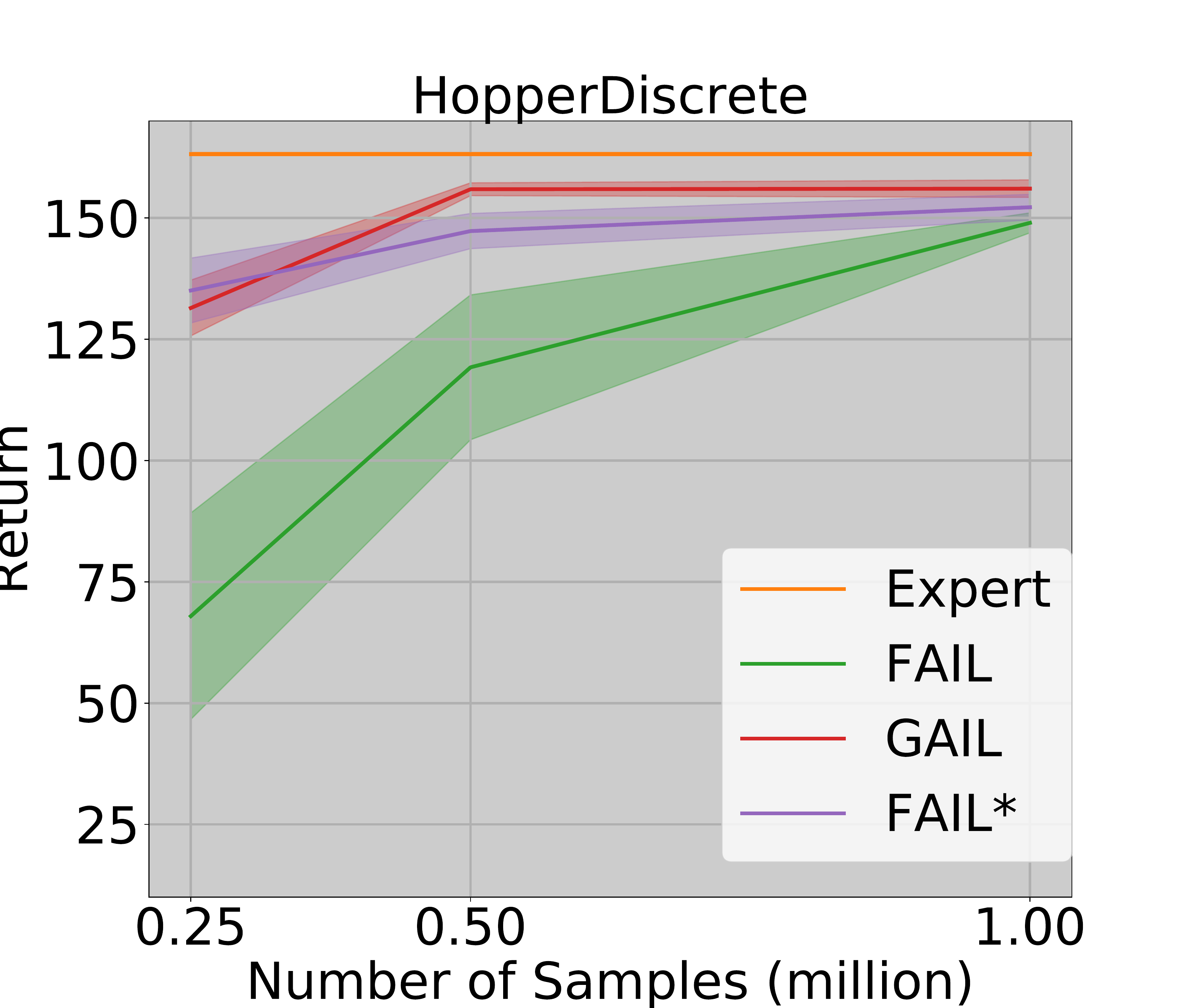}
  \caption{Hopper (25)}
  \label{fig:hopper_new_2}
\end{subfigure}
\begin{subfigure}{.235\textwidth}
  \centering
  \includegraphics[width=1.0\linewidth]{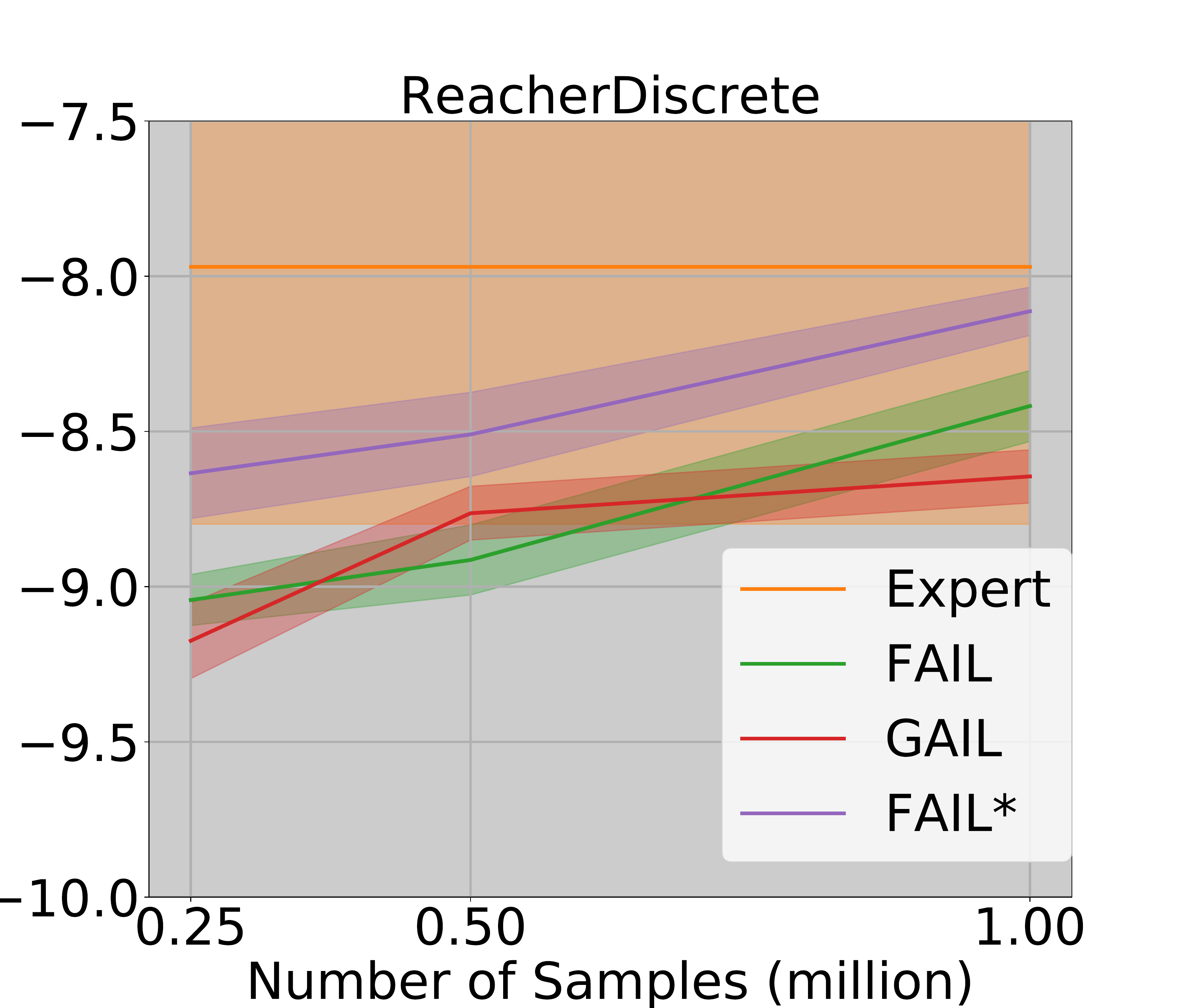}
  \caption{Reacher (25)}
  \label{fig:reacher_new}
\end{subfigure}
\caption{Performance of expert, \fail${^*}$ (\pref{alg:main_alg_2}), \fail (\pref{alg:main_alg}),  and GAIL (without actions) on three  control tasks. For the top line, we fix the number of training samples while varying the number of expert demonstrations (6, 12, 25). For the bottom line, we fix the number of expert demonstrations, while varying the number of training samples. All results are averaged over 10 random seeds.}
\label{fig:new_fail}
\vspace{-5pt}
\end{figure}

We test~\pref{alg:main_alg_2} on the same set of environments (\pref{fig:new_fail}) under 10 random rand seeds, with all default parameters. We observe that \fail${^*}$ can be more sample efficient especially in small data setting (e.g., $0.25$ million training samples).  Implementation of~\pref{alg:main_alg_2} and scripts for reproducing results can be found in \url{https://github.com/wensun/Imitation-Learning-from-Observation/tree/fail_star}.

\end{document}
